\documentclass[journal]{IEEEtran}
\usepackage{graphicx}
\usepackage{array}
\usepackage[caption=false,font=footnotesize,labelfont=footnotesize]{subfig}
\usepackage{textcomp}
\usepackage{stfloats}
\usepackage{url}
\usepackage{verbatim}
\usepackage{cite}
\usepackage{xcolor}
\usepackage{hyperref}

\hypersetup{
colorlinks=true,
linkcolor=blue,
urlcolor=blue,
citecolor=blue
}

\usepackage{multirow}
\usepackage{colortbl}
\usepackage{amsthm}
\newtheorem{theorem}{Theorem}

\usepackage{booktabs}
\usepackage{algorithm}
\usepackage{algpseudocode}
\usepackage{amsmath,amssymb,amsfonts}

\usepackage{pgfplots}
\newtheorem{definition}{Definition}
\pgfplotsset{compat=1.18}

\usepackage{xspace}
\newcommand{\ie}{\textit{i.e.}\xspace}
\newcommand{\eg}{\textit{e.g.}\xspace}

\usepackage{tabularx}
\usepackage{makecell}
\usepackage{arydshln}
\usepackage{paralist}
\usepackage{enumitem}
\usepackage{diagbox}
\renewcommand{\appendixname}{}

\newcommand{\rev}[1]{{\color{black}#1}}
\newcommand{\revminor}[1]{{\color{black}#1}}

\makeatletter
\newcommand{\manualLabel}[2]{%
  \@namedef{r@#1}{{#2}{1}}%
}
\makeatother

\begin{document}

\title{Boosting Adversarial Transferability with Low-Cost Optimization\\ via Maximin Expected Flatness}

\author{Chunlin~Qiu,
        Ang~Li,
        Yiheng~Duan, 
        Shenyi~Zhang,~\IEEEmembership{Graduate Student Member,~IEEE}, 
        Yuanjie~Zhang, \\
        Lingchen~Zhao,
        Qian~Wang,~\IEEEmembership{Fellow,~IEEE}

\thanks{This work was supported in part by the National Science Foundation of China under Grant U2441240 (``Ye Qisun'' Science Foundation), Grant 62441238 and Grant 62302344. (Corresponding author: Qian Wang.)}


\thanks{Chunlin~Qiu, Ang~Li, Yiheng~Duan, Shenyi~Zhang, Yuanjie~Zhang, Lingchen~Zhao, Qian~Wang are with Key Laboratory of Aerospace Information Security and Trusted Computing, Ministry of Education, School of Cyber Science and Engineering, Wuhan University, Wuhan 430072, China (e-mail: \{chunlinqiu, anglii, yihengduan, shenyizhang, yuanjiezhang, lczhaocs, qianwang\}@whu.edu.cn).}

\thanks{This article has supplementary downloadable material available at http://ieeexplore.ieee.org, provided by the authors. The material includes a PDF file containing mathematical proofs and additional experimental results.}

}

\markboth{IEEE TRANSACTIONS ON INFORMATION FORENSICS AND SECURITY}%
{Qiu \MakeLowercase{\textit{et al.}}:Boosting Adversarial Transferability with Low-Cost Optimization via Maximin Expected Flatness}

\maketitle
\begin{abstract}

Transfer-based attacks craft adversarial examples on white-box surrogate models and directly deploy them against black-box target models, offering \rev{practical} query-free threat scenarios. While flatness-enhanced methods have recently emerged to improve transferability by enhancing the loss surface flatness of adversarial examples, their divergent flatness definitions and heuristic attack designs suffer from unexamined optimization limitations and missing theoretical foundation, thus constraining their effectiveness and efficiency. This work exposes the severely imbalanced exploitation-exploration dynamics in flatness optimization, establishing the first theoretical foundation for flatness-based transferability and proposing a principled framework to overcome these optimization pitfalls. Specifically, we systematically unify fragmented flatness definitions across existing methods, revealing their imbalanced optimization limitations in over-exploration of sensitivity peaks or over-exploitation of local plateaus. To resolve these issues, we rigorously formalize average-case flatness and transferability gaps, proving that enhancing zeroth-order average-case flatness minimizes cross-model discrepancies. Building on this theory, we design a Maximin Expected Flatness (MEF) attack that enhances zeroth-order average-case flatness while balancing flatness exploration and exploitation. Extensive evaluations across \rev{33} models and \rev{43} current transfer-based attacks demonstrate MEF's superiority: it surpasses the state-of-the-art PGN attack by 4\% in attack success rate at half the computational cost and achieves 8\% higher success rate under the same budget. When combined with input augmentation, MEF attains 15\% additional gains against defense-equipped models, establishing new robustness benchmarks. Our code is available at \url{https://github.com/SignedQiu/MEFAttack}.

\end{abstract}

\begin{IEEEkeywords}
Adversarial attack, black-box attack, adversarial transferability, loss surface flatness.
\end{IEEEkeywords}    
\begin{figure}[ht!]
  \centering
  \footnotesize
  \vspace{-20pt}
  \subfloat[VMI~\cite{vmi-vni}]{\includegraphics[width=0.33\linewidth]{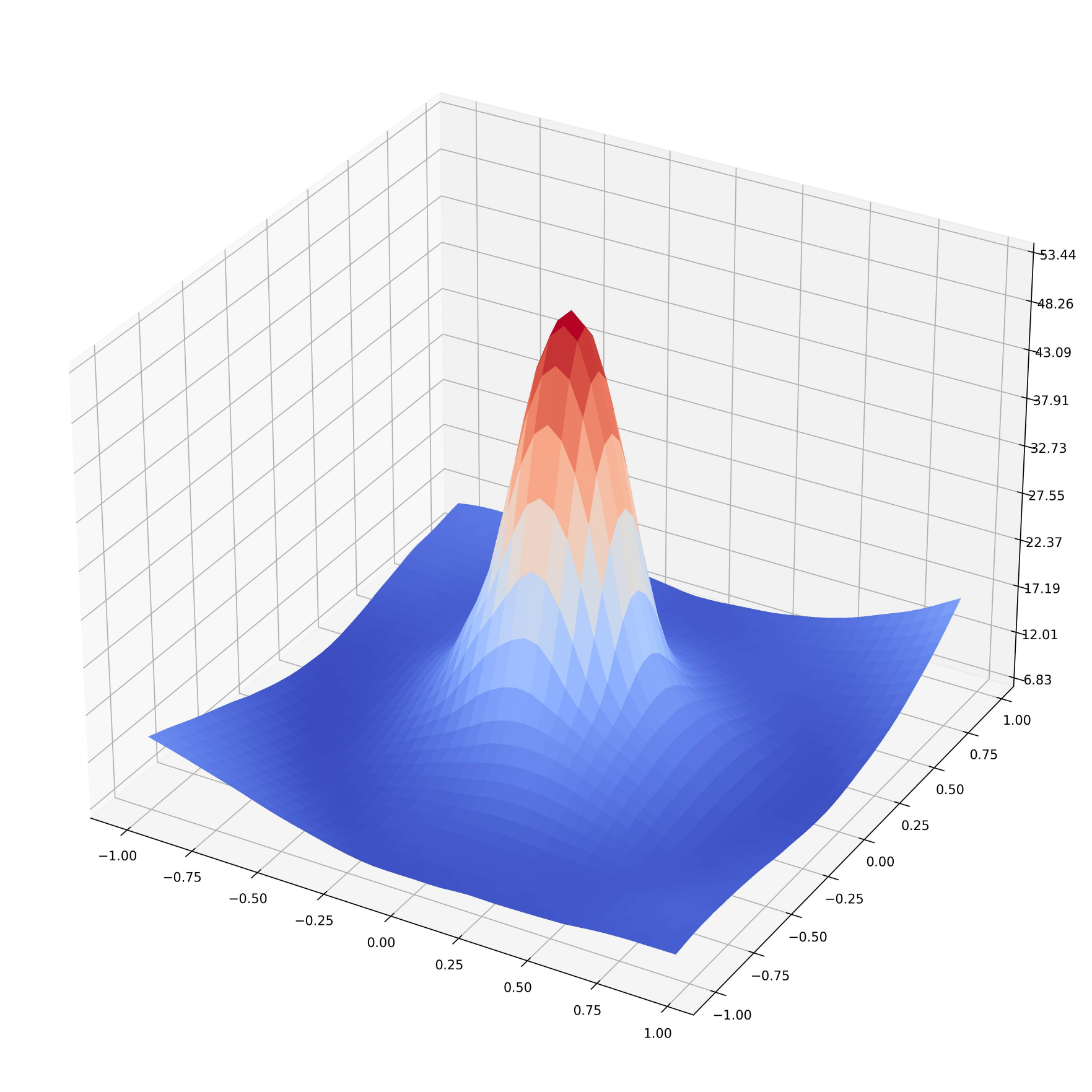}}
  \hfill
  \subfloat[TPA~\cite{tpa}]{\includegraphics[width=0.33\linewidth]{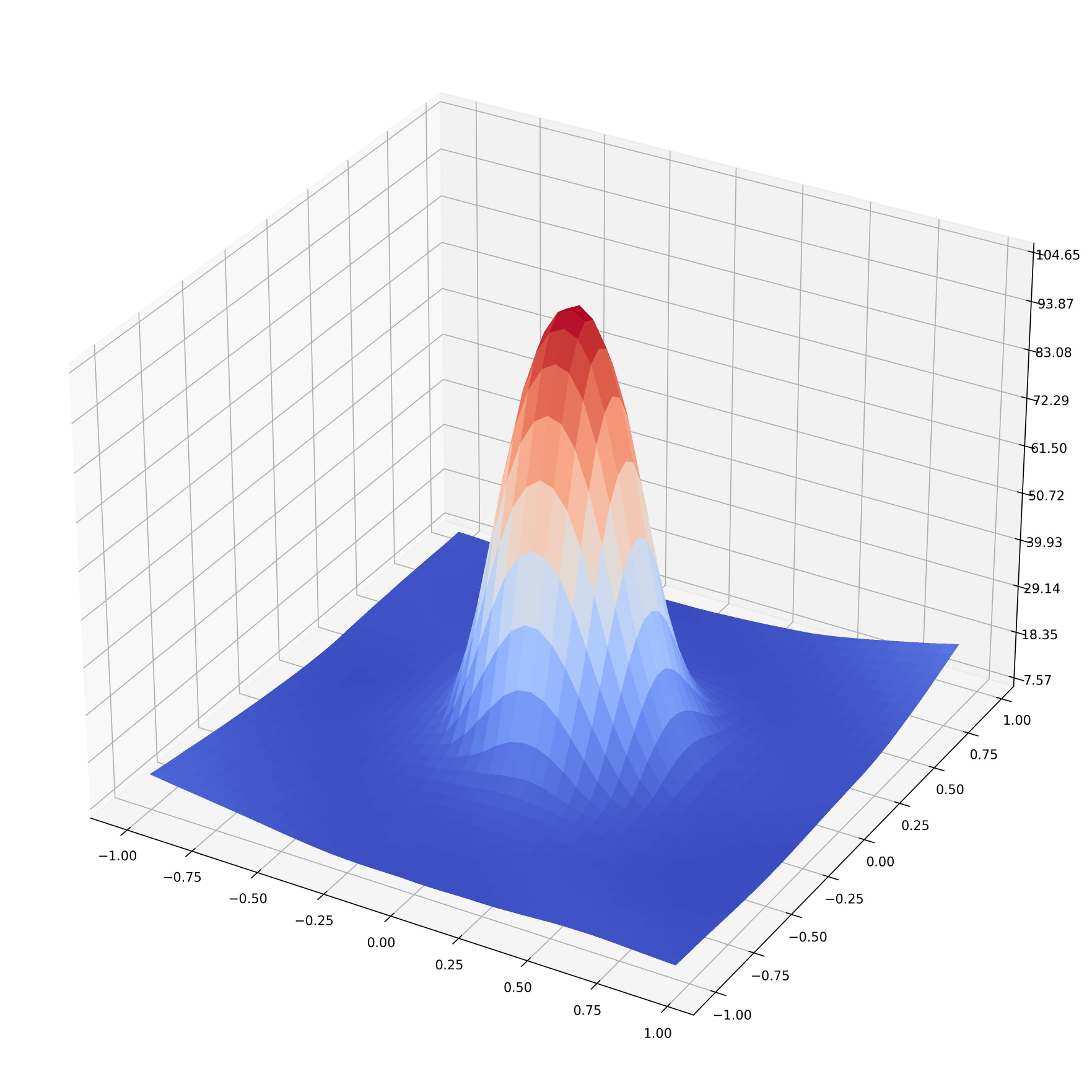}}
  \hfill
  \subfloat[FEM~\cite{femi}]{\includegraphics[width=0.33\linewidth]{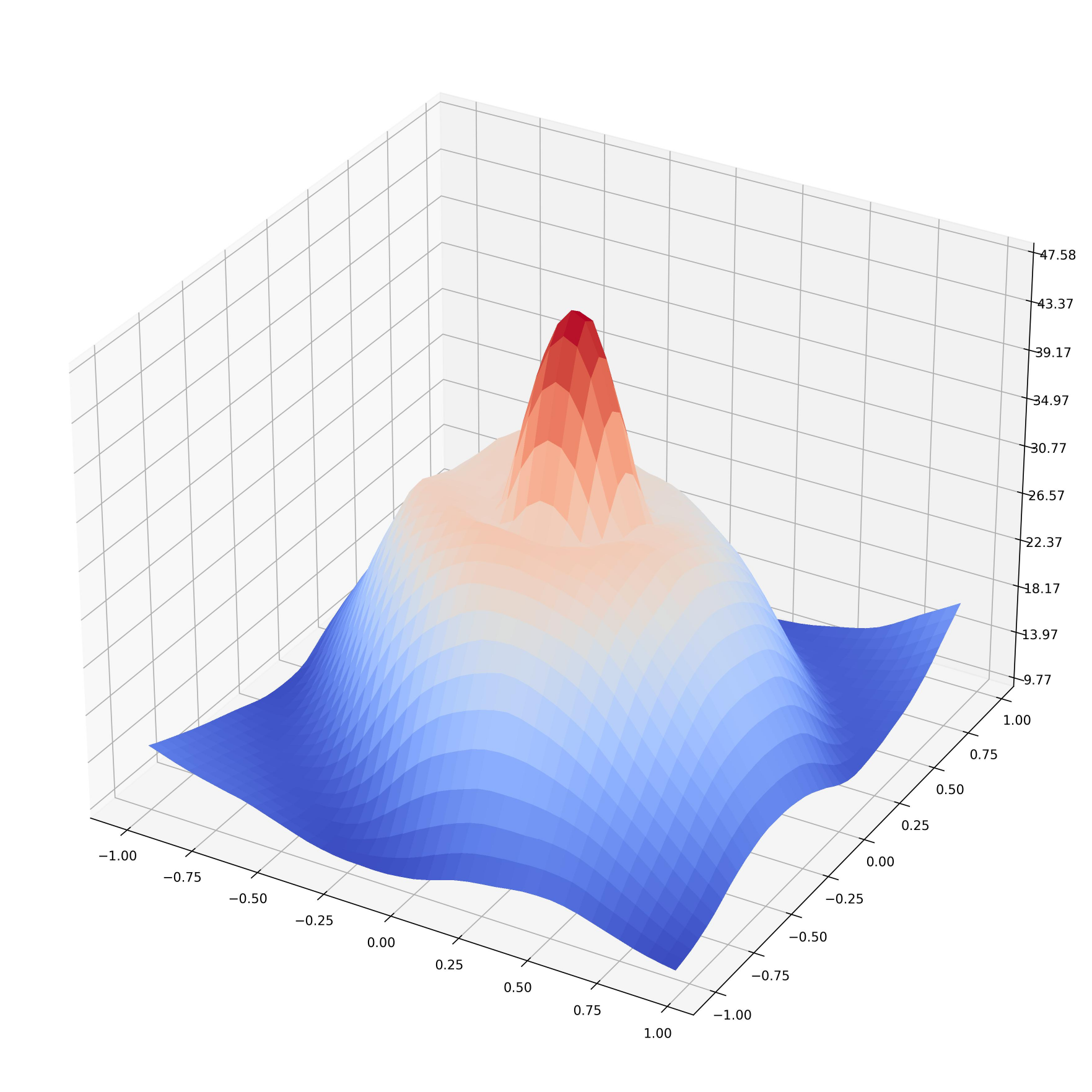}}

  \vspace{-10pt}
  
  \subfloat[EMI~\cite{pi-emi}]{\includegraphics[width=0.33\linewidth]{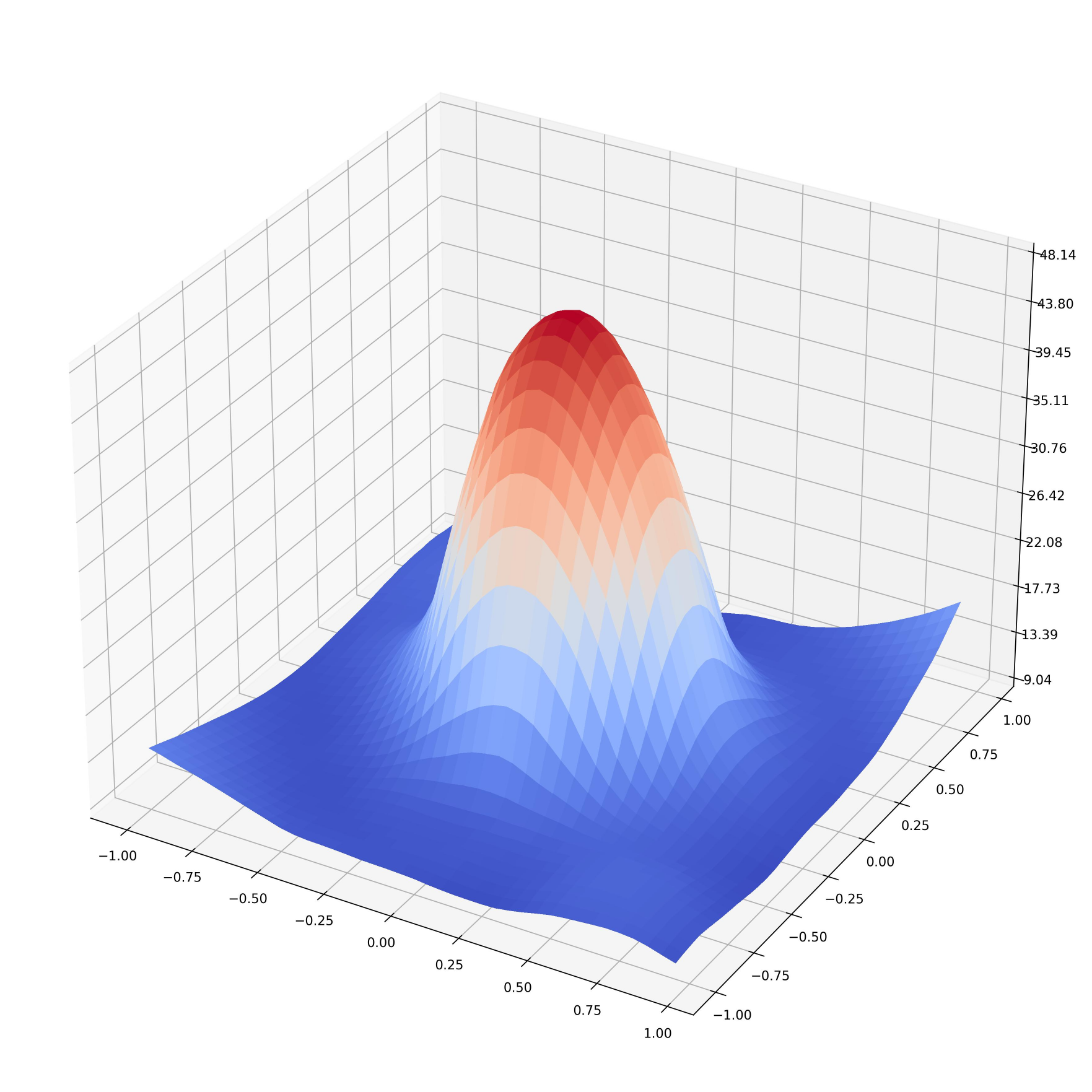}}
  \hfill
  \subfloat[GNP~\cite{gnp}]{\includegraphics[width=0.33\linewidth]{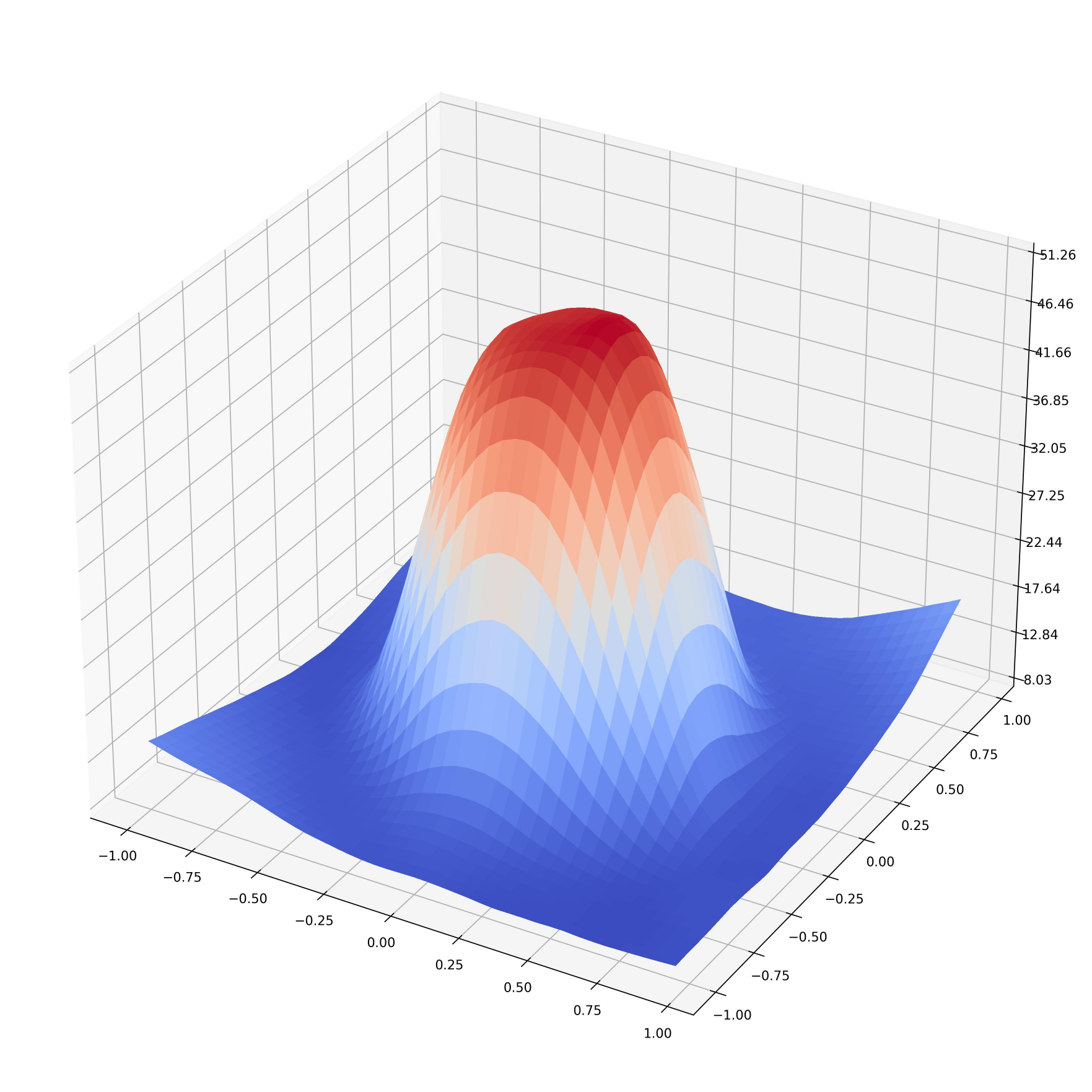}}
  \hfill
  \subfloat[APP~\cite{app}]{\includegraphics[width=0.33\linewidth]{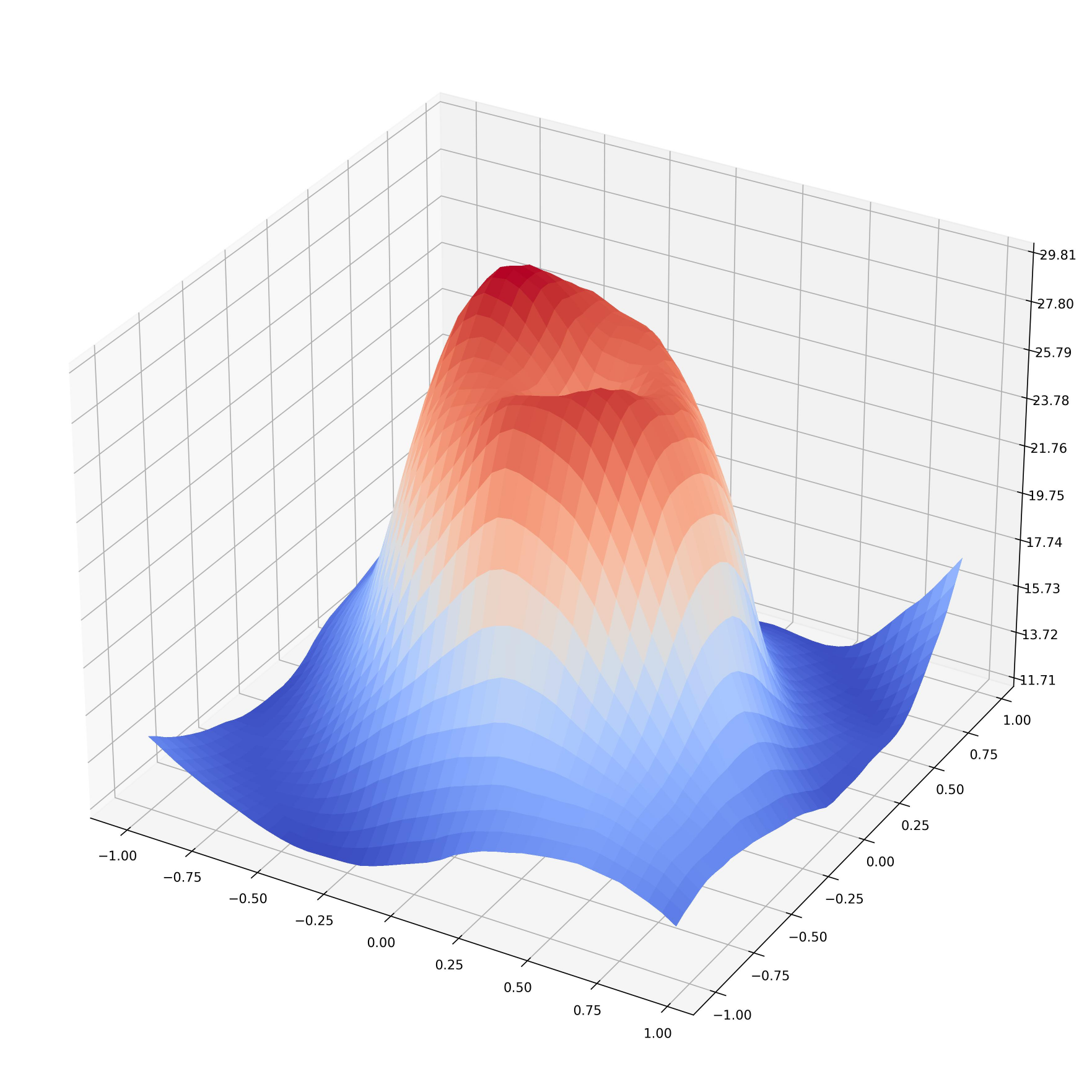}}

  \vspace{-10pt}
  
  \subfloat[RAP~\cite{rap}]{\includegraphics[width=0.33\linewidth]{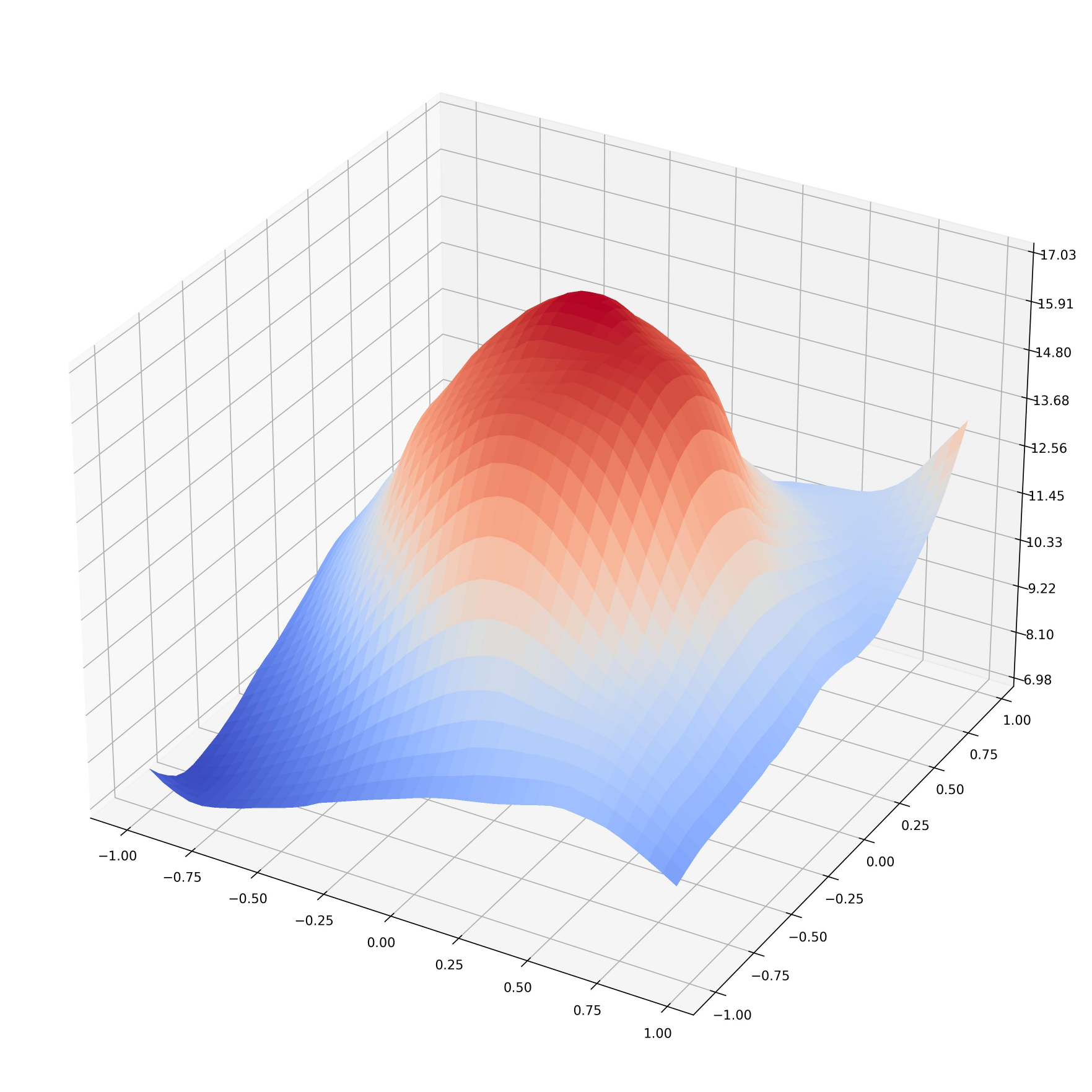}}
  \hfill
  \subfloat[PGN~\cite{pgn}]{\includegraphics[width=0.33\linewidth]{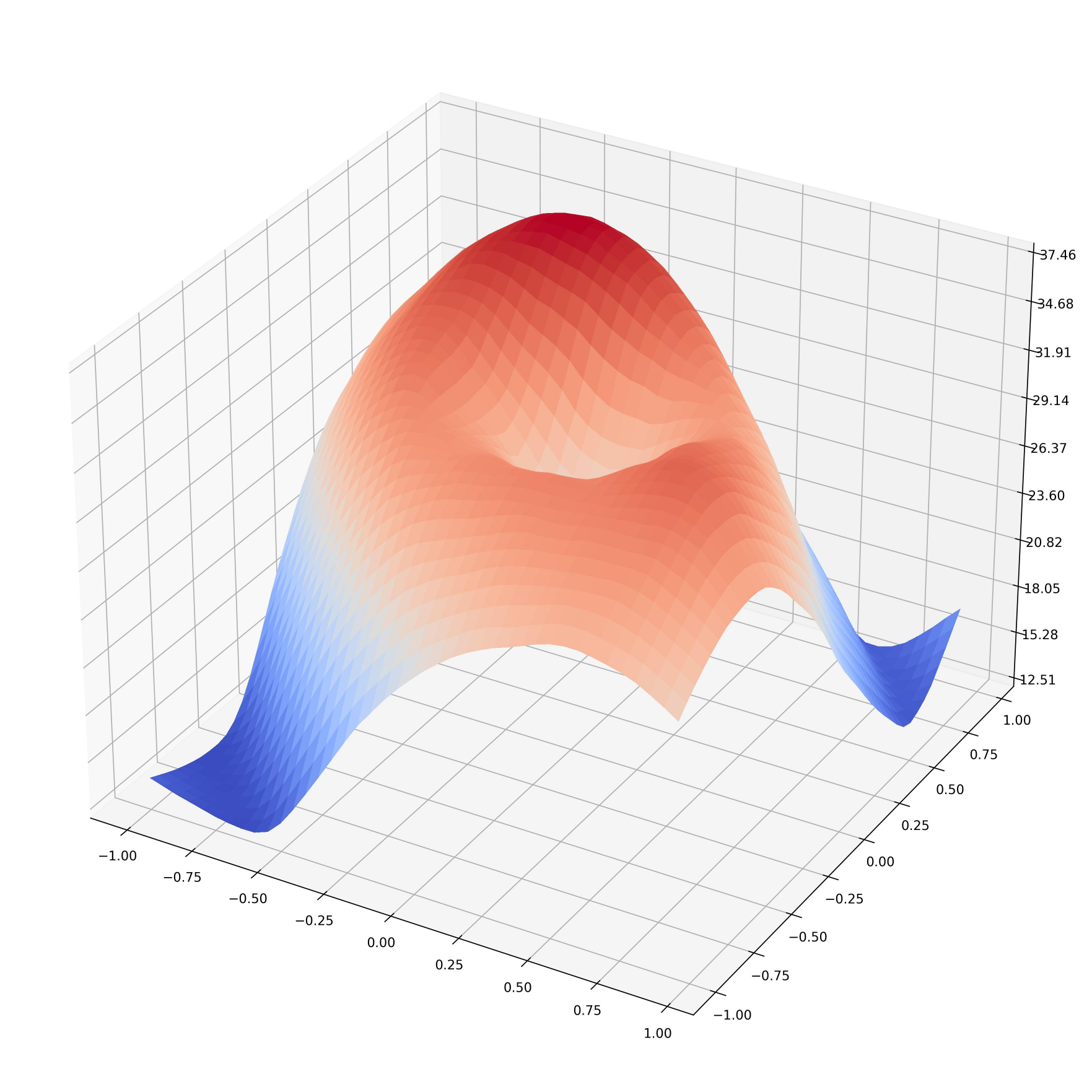}}
  \hfill
  \subfloat[MEF]{\includegraphics[width=0.33\linewidth]{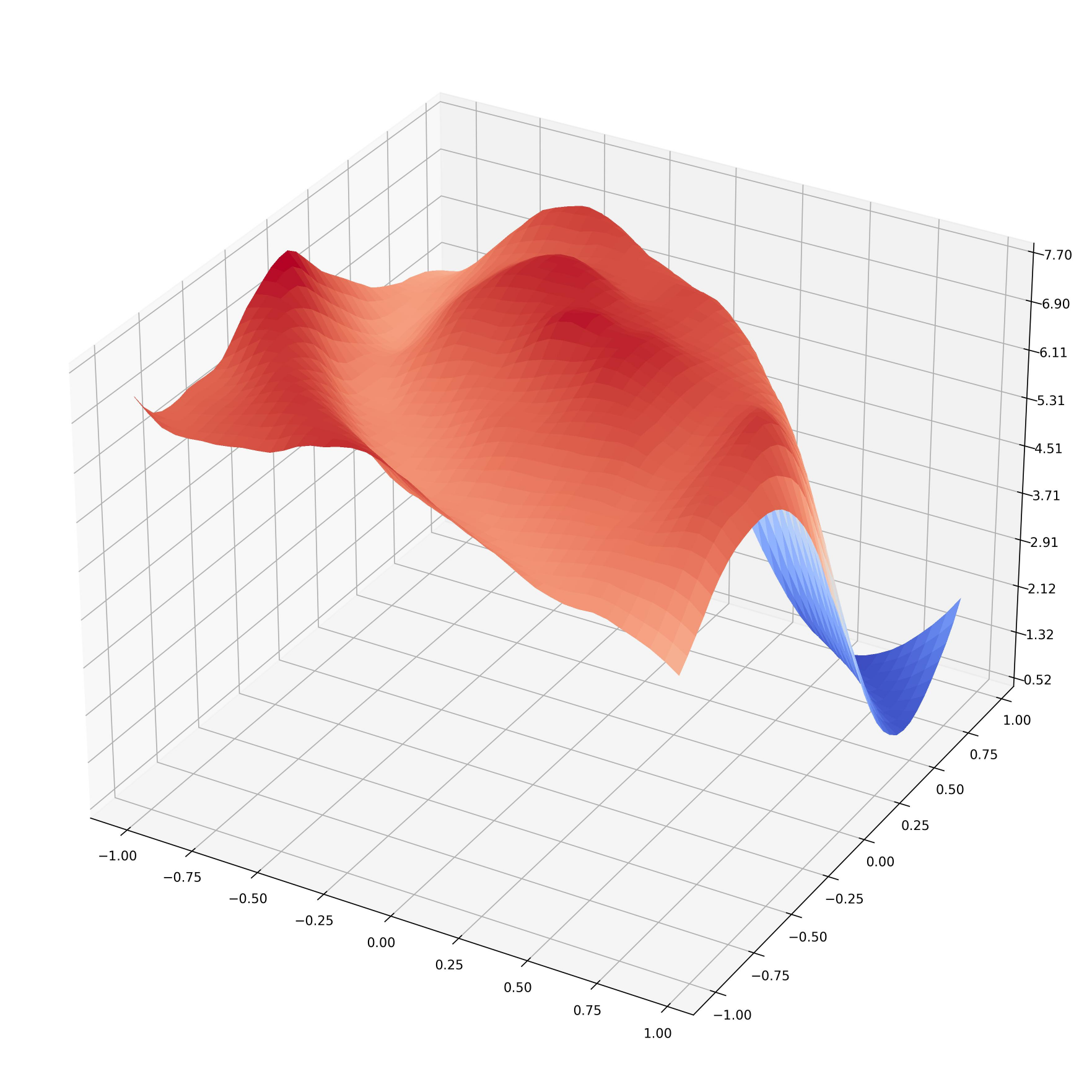}}
  \caption{Visualization of adversarial loss landscapes for nine attacks on Res-50~\cite{resnet}. The loss surfaces are constructed by perturbing adversarial examples along two random directions. Our MEF attack achieves the flattest loss landscape, indicating superior flatness-based transferability.}
  \vspace{-10pt}
  \label{fig:pdf_teaser}
\end{figure}
\section{Introduction}
\label{sec:intro}

\IEEEPARstart{A}{dversarial} examples~\cite{szegedy2013intriguing, pgd} pose a significant security threat to deep neural networks (DNNs). By introducing imperceptible perturbations to benign inputs, attackers can craft adversarial examples that mislead DNNs into producing erroneous predictions. Early adversarial attacks primarily operated under two paradigms: white-box attacks that require full access to the target model’s architecture and parameters, such as L-BFGS~\cite{szegedy2013intriguing} and PGD~\cite{pgd}, and query-based grey-box attacks that rely on iterative adversarial direction estimation through repeated model queries, as seen in Boundary Attack~\cite{boundaryattack} and SimBA~\cite{simba}. While effective, these methods face critical deployment barriers: white-box attacks assume unrealistic access to victim models, while query-based methods incur prohibitive query overhead, making them impractical in real-world scenarios.

Transfer-based attacks~\cite{di, tim, mi, vmi-vni, ni, tpa} overcome these limitations by crafting adversarial examples on surrogate models and directly transferring them to black-box targets. This paradigm eliminates both dependency on target model's internals and query interactions, making it a practical threat to real-world systems. Among these attacks, a recent class of flatness-enhanced methods~\cite{femi, rap, pgn, tpa} has emerged as a promising direction by linking the geometric flatness of the adversarial loss landscape to adversarial transferability. Flatness-enhanced methods craft adversarial examples with flatter neighborhood loss landscapes, enhancing cross-model robustness and achieving state-of-the-art transferability~\cite{pgn}.

Despite their empirical success, flatness-enhanced methods face two fundamental limitations. Firstly, they adopt fragmented flatness definitions without systematic comparison, failing to discern their inherent advantages and critical flaws. Secondly, they rely on the empirically assumed but unproven connection between flatness and transferability~\cite{rap,pgn,gnp,femi}, leading to non-rigorous and suboptimal optimization.

To address these limitations, we first unify fragmented flatness definitions within a multi-order framework that differentiates between worst-case and average-case flatness. Through experiments, we demonstrate that worst-case flatness optimization over-explores high-curvature regions and average-case methods over-exploit suboptimal plateaus, two critical flaws responsible for the performance bottleneck in existing attacks. We further establish the first theoretical proof linking flatness to transferability, formalizing that enhancing multi-order average-case flatness reduces cross-model discrepancies. Our theoretical analysis certifies zeroth-order flatness as the dominant transferability source, paving the way for high-transferability attacks with reduced computational costs.

Based on the empirical and theoretical analysis, we propose the Maximin Expected Flatness (MEF) attack to enhance adversarial transferability while reducing computational cost. Our key insight is to leverage the complementary strengths of average-case and worst-case flatness optimization, where average-case methods stabilize local exploitation while worst-case strategies preserve global exploration capability. MEF attack explicitly optimizes \textit{zeroth-order worst-neighborhood average flatness}, a unified metric that internally computes neighborhood-level average-case flatness to stabilize local exploitation and externally identifies worst-case regions to guide global exploration. MEF implements this via gradient-guided neighborhood conditional sampling (NCS) for efficient worst-region identification and gradient balancing optimization (GBO) that reuses gradients across stages, reducing computation by 50\% while enhancing transferability. By coupling these mechanisms, MEF achieves flattest loss landscape (as shown in Fig.~\ref{fig:pdf_teaser}) and enhanced adversarial transferability. Third-party evaluation on the widely-adopted \href{https://github.com/Trustworthy-AI-Group/TransferAttack}{TransferAttack} benchmark shows MEF dominates \rev{26} gradient-based methods as the new leader. Our main contributions are summarized as follows:

\begin{itemize}  
\item We unify fragmented flatness definitions across existing methods and establish the first theoretical foundation of flatness-based transferability, connecting multi-order average-case flatness with adversarial transferability.  

\item We propose Maximin Expected Flatness (MEF) attack, together with Neighborhood Conditional Sampling (NCS) and Gradient Balancing Optimization (GBO) to boost adversarial transferability with efficient computation.

\item Extensive experiments and third-party evaluation\footnote{Available in: https://github.com/Trustworthy-AI-Group/TransferAttack} validate the superior effectiveness of MEF attack, surpassing the SOTA PGN~\cite{pgn} attack by 4\% in attack success rate on average at half the computational cost and achieves 8\% higher success rate under the same budget.
\end{itemize}

\section{Related Work}
\label{sec:related_work}

This section systematically reviews adversarial machine learning literature through two lenses: (1) we survey transfer-based adversarial attacks, and (2) analyze various defense mechanisms, including both reactive and proactive defenses.

\subsection{Adversarial Attacks}
\label{sec:adversarial_attacks}
We first give the formal definition of adversarial attacks and then introduce various transfer-based attacks. Adversarial attacks craft perturbed examples $\mathbf{x}^{adv} \in B_\epsilon(\mathbf{x}) = \{\mathbf{x}' : \|\mathbf{x}' - \mathbf{x}\|_p \leq \epsilon\}$ to fool a neural network $\mathcal{F}: \mathcal{X} \rightarrow \mathcal{Y}$, where $\mathbf{x} \in \mathcal{X}$ is the clean input with label $y \in \mathcal{Y}$. The attack objective:
\revminor{
\begin{equation}
\underset{\mathbf{x}^{adv} \in B_{\epsilon}(\mathbf{x})}{\max} J(\mathbf{x}^{adv}, y; \mathcal{F}),
\label{eq:mp}
\end{equation}
}
where $J: \mathcal{X} \times \mathcal{Y} \rightarrow \mathbb{R}$ denotes the adversarial loss function. Following standard practice, we adopt the $\ell_\infty$-norm ($p=\infty$) for perturbation constraints. For notation simplicity, we may write $J(\mathbf{x}) \triangleq J(\mathbf{x}, y; \mathcal{F})$. An intriguing property of adversarial examples is transferability, where adversarial examples crafted on surrogate model $\mathcal{F}$ also deceive the unknown target model $\mathcal{F}'$, enabling black-box attacks. Transfer-based attacks can be roughly divided into \emph{\textbf{data-driven methods}}, \emph{\textbf{model-driven methods}} and \emph{\textbf{optimization-driven methods}}, as follows.

\textbf{Data-driven Methods.} Data-driven methods~\cite{di, tim, ni, admix} boost adversarial transferability by aggregating gradients over multiple augmented variants through data transformation strategies. The key distinction between these methods centers on the specific transformation operators. DI~\cite{di} optimizes over resized images, TI~\cite{tim} applies spatial shifting, and SI~\cite{ni} manipulates pixel scaling. Beyond these single-input spatial transformations, Admix~\cite{admix} enhances diversity by blending cross-category images, while SSA~\cite{ssa} perturbs frequency domains to break spatial locality. To further amplify input variety, composite strategies integrate multiple transformations~\cite{l2t,ops}. However, such strategies incur significant computational overheads and induce semantic distortions, ultimately compromising their effectiveness.

\textbf{Model-driven Methods.} Model-driven methods~\cite{sgm, linbp, svre, cwa} exploit geometric priors in surrogate models’ feature spaces to enhance adversarial transferability, primarily through two strategies: model ensemble aligns heterogeneous architectures via joint perturbation optimization to amplify spectral commonality with victims~\cite{svre,cwa}, while surrogate refinement adversarially tunes surrogates or aligns gradients to reshape decision boundaries~\cite{sgm,linbp}. However, model ensemble scales linearly with surrogates' computational costs, and refinement struggles on large datasets due to complex regularization, limiting both to small-scale applications.

\textbf{Optimization-driven Methods.} Optimization-driven methods~\cite{mi,rap,ila,naa}, enhance transferability by redesigning optimization objectives~\cite{ila,naa} or algorithms~\cite{cda,mi}. Advanced Objective strategies replace cross-entropy loss with feature discrepancy losses that align surrogate-victim model features~\cite{ila, naa}. Generative Modeling strategies leverage generative networks to synthesize perturbations without iterative updates~\cite{cda}. Gradient Stabilization strategies retain simple losses but refine gradient dynamics via momentum~\cite{mi,ni,pi-emi} or variance reduction~\cite{vmi-vni}. Compared to the former two paradigms requiring complex losses or architectures, gradient stabilization strategies achieve transferability with minimal complexity, avoiding auxiliary data and models.

Recent advances in gradient stabilization attacks have seen the emergence of flatness-enhanced methods, which propose that adversarial examples within flatter loss regions exhibit stronger transferability. These methods diverge into two strategies: those optimizing \textit{average-case flatness} by aggregating gradients over local neighborhoods (\eg, FEM~\cite{femi}, APP~\cite{app}), and those pursuing \textit{worst-case flatness} by minimizing gradient peaks (PGN~\cite{pgn}) or maximizing minimal loss (RAP~\cite{rap}). However, existing works adopt fragmented flatness definitions while relying on empirically assumed correlations between flatness and transferability. \rev{Crucially, these methods typically employ flatness as a heuristic regularizer without a rigorous theoretical foundation linking specific geometric properties to the transferability gap. This theoretical void leads to suboptimal optimization directions}. Our work bridges this gap by establishing unified flatness-transferability guarantees and adaptive optimization that harmonizes both flatness criteria.

\subsection{Adversarial Defenses}
\label{sec:adversarial_defenses}
Adversarial defenses can be categorized into two paradigms: \emph{\textbf{proactive defenses}} that strengthen model robustness through architectural or parametric adjustments, and \emph{\textbf{reactive defenses}} that deploy external mechanisms to sanitize or detect adversarial inputs without modifying model parameters.

\textbf{Proactive Defenses.} Proactive defenses aim to intrinsically harden model parameters against adversarial perturbations. The dominant approach, adversarial training (AT)~\cite{pgd,advincv3}, reformulates robust model optimization as a min-max optimization problem, where the inner maximization generates perturbations to craft challenging adversarial examples, while the outer minimization enforces the model to correctly classify these perturbed samples. Variants like ensemble adversarial training (EAT)~\cite{eat} augment robustness through multi-model perturbations during training, strengthening defense against cross-architecture transfer attacks. In contrast to AT’s empirical robustness, randomized smoothing~\cite{rs} provides certifiable robustness by injecting Gaussian noise into inputs and certifying predictions over perturbed samples. While this method theoretically guarantees robustness within a certified radius, such radii are often smaller than adversarial perturbations defended by AT, rendering it ineffective against practical adversarial threats despite its theoretical appeal.

\textbf{Reactive Defenses.} Reactive defenses mitigate adversarial attacks without modifying model parameters. Adversarial purification~\cite{hgd,nrp} projects inputs onto clean data manifolds using generative models~\cite{hgd,nrp}, while adversarial detection identifies adversarial examples via activation anomalies or confidence thresholds. These strategies offer plug-and-play deployment but suffer from two notable limitations: purification efficacy depends critically on generative model capacity~\cite{diffpure, densepure}, and detectors can be bypassed by sophisticated adaptive attacks that mimic clean data statistics.

\section{Limitations in Flatness Enhancement}
\label{sec:preliminary}

The heuristic adoption of flatness concepts in flatness-enhanced methods lacks rigorous theoretical grounding while suffering from absence of systematic comparisons between flatness metrics, which constrains their effectiveness and efficiency. In this section, we first formalize flatness as a measurable geometric property, resolving definitional ambiguities across existing methods. Then we identify two inherent limitations in current flatness-enhanced methods. Finally, we distill the critical yet understudied challenges in balancing local exploitation and global exploration for real-world deployment.

\subsection{K-th Order Flatness}
\begin{table*}[htb]
\centering
\caption{Comparative analysis of flatness-enhanced methods, summarizing their conference venues, optimization target, and flatness formulations. We systematically compare six representative methods' mathematical frameworks for achieving flatness.}
\label{tab:flatness_methods}
\footnotesize
\setlength{\tabcolsep}{0pt}
\renewcommand{\tabularxcolumn}[1]{m{#1}}
\newcolumntype{Y}{>{\centering\arraybackslash}X}
\newcolumntype{P}{>{\centering\arraybackslash}m{0.40\textwidth}}
\begin{tabularx}{\textwidth}{@{} l Y P Y @{}}
  \toprule
  Flatness-Enhanced Method & Conference (Year) & Optimization Target & Flatness Formulation \\
  \midrule
  Reverse Adversarial Perturbation (RAP~\cite{rap}) & NeurIPS 2022 & \makecell{$\underset{\mathbf{x}^{adv}\in{B}_{\epsilon}(\mathbf{x})}{max}\,\underset{\mathbf{x'}\in{B}_{\xi}(\mathbf{x}^{adv})}{min}\,J(\mathbf{x'})$} & $\widehat{R}^{(0)}_{\xi}(\mathbf{x})$ \\
  
  Gradient Norm Penalty (GNP~\cite{gnp}) & ICIP 2022 & \makecell{$\underset{\mathbf{x}^{adv}\in{B}_{\epsilon}(\mathbf{x})}{max}\,\underset{\mathbf{x'}\in{B}_{\xi}(\mathbf{x}^{adv})}{min}\,[J(\mathbf{x}^{adv}) - \lambda\cdot{\|\nabla_{\mathbf{x'}}J(\mathbf{x'})\|_{2}}]$} & $\widehat{R}^{(1)}_{\xi}(\mathbf{x})$ \\

  Penalizing Gradient Norm (PGN~\cite{pgn})  & NeurIPS 2023 & \makecell{$\underset{\mathbf{x}^{adv}\in{B}_{\epsilon}(\mathbf{x})}{max}\,\underset{\mathbf{x'}\in{B}_{\xi}(\mathbf{x}^{adv})}{min}\,[J(\mathbf{x}^{adv}) - \lambda\cdot{\|\nabla_{\mathbf{x'}}J(\mathbf{x'})\|_{2}}]$} & $\widehat{R}^{(1)}_{\xi}(\mathbf{x})$ \\

  Flatness Enhanced Method (FEM~\cite{femi}) & ICME 2023 & \makecell{$\underset{\mathbf{x}^{adv}\in{B}_{\epsilon}(\mathbf{x})}{max}\,\mathbb{E}_{\mathbf{x'}\sim \mathrm{Unif}{B}_{\xi}(\mathbf{x}^{adv})}\,\left[J(\mathbf{x'})\right]$} & $\overline{R}^{(0)}_{\xi}(\mathbf{x})$ \\

  Adversarial Pixel Perturbation (APP~\cite{app})  & ICIP 2023 & \makecell{$\underset{\mathbf{x}^{adv}\in{B}_{\epsilon}(\mathbf{x})}{max}\,\mathbb{E}_{\mathbf{x'}\sim \mathrm{Unif}{B}_{\xi}(\mathbf{x}^{adv})}\,\left[J(\mathbf{x'})\right]$} & $\overline{R}^{(0)}_{\xi}(\mathbf{x})$ \\

  Theoretically Provable Attack (TPA~\cite{tpa})  & NeurIPS 2024 & \makecell{$\underset{\mathbf{x}^{adv}\in{B}_{\epsilon}(\mathbf{x})}{max}\,\mathbb{E}_{\mathbf{x'} \in {B}_{\xi}(\mathbf{x}^{adv})}\,{\left[ \|\nabla_{\mathbf{x'}}J(\mathbf{x'})\|_{2}\right]}$} & $\overline{R}^{(1)}_{\xi}(\mathbf{x})$ \\
  \bottomrule
\end{tabularx}
\vspace{-15pt}
\end{table*}
\begin{figure*}[ht]
  \centering
  
  \hspace{-7.75pt}
  \subfloat[RAP~\cite{rap} ($\widehat{R}^{(0)}_{\xi}(\mathbf{x})$)]{\includegraphics[width=0.26\textwidth]{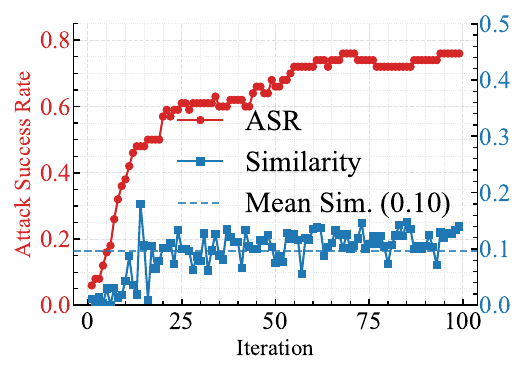}\label{fig:flat_rethink_rap}}
  \hspace{-7.75pt}
  \subfloat[FEM~\cite{femi} ($\overline{R}^{(0)}_{\xi}(\mathbf{x})$)]{\includegraphics[width=0.26\textwidth]{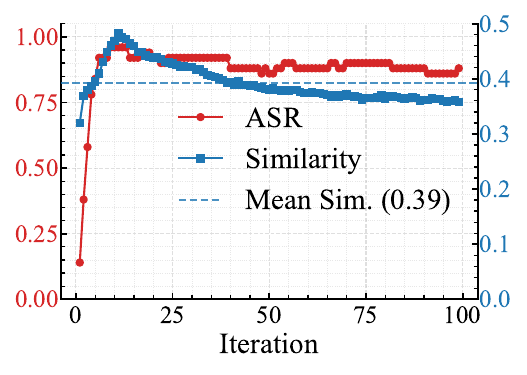}\label{fig:flat_rethink_fem}}
  \hspace{-7.75pt}
  \subfloat[PGN~\cite{pgn} ($\widehat{R}^{(1)}_{\xi}(\mathbf{x})$)]{\includegraphics[width=0.26\textwidth]{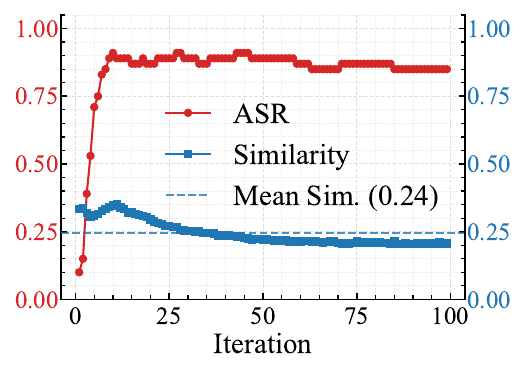}\label{fig:flat_rethink_pgn}}
  \hspace{-7.75pt}
  \subfloat[TPA~\cite{tpa} ($\overline{R}^{(1)}_{\xi}(\mathbf{x})$)]{\includegraphics[width=0.26\textwidth]{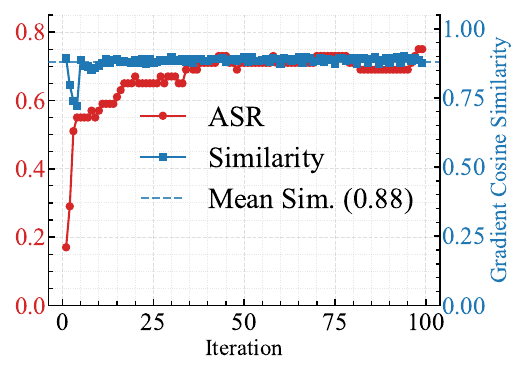}\label{fig:flat_rethink_tpa}}
  \hspace{-7.75pt}
  
  \vspace{-2pt}
  \caption{Optimization dynamics under flatness duality ($\overline{R}^{(n)}_\xi$ vs. $\widehat{R}^{(n)}_\xi$). Comparing zeroth/first-order methods on 100 ImageNet samples~\cite{imagenet}, we measure \textit{inter-update} gradient similarity and transfer attack success rates (Res-50~\cite{resnet}$\rightarrow$Inc-v3~\cite{incv3}). Post ASR convergence, worst-case variants sustain lower gradient similarity (0.10-0.25) than average-case ones (0.38-0.88), revealing exploration-exploitation trade-off governed by flatness formalism.}
  \vspace{-8pt}
  
  \label{fig:flat_rethink}
\end{figure*}
\label{sec:k_order_flatness}
In this section, we introduce the formal definitions of k-order worst- and average-case flatness metrics, establishing a unified taxonomy for flatness enhancement and enabling systematic classification of six flatness-enhanced methods.

The study of loss landscape flatness originated in model generalization theory, where flatter minima were shown to correlate with improved generalization performance~\cite{sam,largebatch,gam}. In model generalization, the most popular mathematical definition of flatness, proposed by SAM~\cite{sam}, considers the maximal loss value within a radius, termed the \textit{zeroth-order worst-case flatness}. Building upon this foundation, Gradient norm Aware Minimization (GAM)~\cite{gam} extends this framework to \textit{first-order worst-case flatness}, a stricter geometric criterion that constrains maximal gradient norms rather than loss values.

Based on SAM's framework~\cite{sam} for minimizing sharpness, RAP~\cite{rap} adapts the \textit{zeroth-order worst-case flatness} to construct transferable adversarial examples. Motivated by RAP~\cite{rap}, Penalizing Gradient Norm (PGN)~\cite{pgn} enhances adversarial transferability by enhancing the \textit{first-order worst-case flatness}. TPA~\cite{tpa} further correlates transferability with first-order average flatness (local mean gradient norms) and second-order gradient component. We give the definition of \textit{$\xi$-radius $n$-order worst-case flatness $\widehat{R}^{(n)}_{\xi}(\mathbf{x})$ and average-case flatness $\overline{R}^{(n)}_{\xi}(\mathbf{x})$} in Definition~\ref{def:k_order_flatness}. Unless otherwise specified, the zeroth‑order derivative is taken to be the loss function itself, \ie $\nabla^0 J(\mathbf{x}) = J(\mathbf{x})$.

\begin{definition}[$\xi$-radius $n$-order Flatness]
\label{def:k_order_flatness}
For any $\xi > 0$ and integer $k \in \mathbb{N}$, the worst- and average-case $\xi$-radius $n$-order flatness of the loss $J(\mathbf{x})$ at sample $\mathbf{x}$ are defined as follows:
\begin{itemize}
    \item \textbf{Worst-case flatness}: 
    \revminor{
    \small
    \setlength{\belowdisplayskip}{0pt}
    \begin{equation}
    \widehat{R}^{(n)}_{\xi}(\mathbf{x}) \triangleq \max_{\mathbf{x}' \in B_{\xi}(\mathbf{x})} \big\| \nabla^n J(\mathbf{x}') - \nabla^n J(\mathbf{x}) \big\|_2,
    \end{equation}
    }
    \item \textbf{Average-case flatness}:
    \revminor{
    \small
    \setlength{\belowdisplayskip}{0pt}
    \begin{equation}
    \overline{R}^{(n)}_{\xi}(\mathbf{x}) \triangleq \mathbb{E}_{\mathbf{x}' \sim \mathrm{Unif}(B_{\xi}(\mathbf{x}))} \big[ \big\| \nabla^n J(\mathbf{x}') - \nabla^n J(\mathbf{x}) \big\|_2 \big],
    \end{equation}
    }
\end{itemize}
\revminor{where $\mathbf{x}'$ denotes a point in $\ell_{\infty}$-ball $B_{\xi}(\mathbf{x}) = \big\{ \mathbf{x}' \in \mathbb{R}^d \,\big|\, \|\mathbf{x}' - \mathbf{x}\|_\infty \le \xi \big\}$ centered at $\mathbf{x}$ with radius $\xi$,} $\nabla^n J(\mathbf{x})$ denotes the $n$‑th order derivative of $J$ at $\mathbf{x}$ and \(\mathrm{Unif}(B_{\xi}(\mathbf{x}))\) denotes the uniform distribution on the set $B_{\xi}(\mathbf{x})$.
\end{definition}

In Table~\ref{tab:flatness_methods}, we summarize six typical flatness-enhanced methods through the lens of our formalized $n$-order flatness metrics. Current methods adopt four distinct flatness criteria: $\widehat{R}^{(0)}_{\xi}$, $\widehat{R}^{(1)}_{\xi}$, $\overline{R}^{(0)}_{\xi}$, and $\overline{R}^{(1)}_{\xi}$. While RAP~\cite{rap}, GNP~\cite{gnp}, and PGN~\cite{pgn} optimize worst-case flatness via maximal loss ($\widehat{R}^{(0)}_{\xi}$) or gradient norm ($\widehat{R}^{(1)}_{\xi}$) constraints, FEM~\cite{femi}, APP~\cite{app}, and TPA~\cite{tpa} focus on average-case objectives, either expected loss ($\overline{R}^{(0)}_{\xi}$) or gradient magnitudes ($\overline{R}^{(1)}_{\xi}$).

To our knowledge, this is the first work that unifies flatness definitions across flatness-enhanced methods and provides a taxonomy distinguishing worst/average-case flatness criteria. While prior work~\cite{robust_flat_minima} quantifies flatness in parameter space for model robustness, our focus centers on sample-wise loss landscape flatness, a distinct perspective aligned with adversarial transferability rather than robust generalization.

\subsection{Rethinking Flatness-Enhanced Methods}
\label{subsec:flatness_rethinking}
Despite the prevalent use of distinct flatness criteria, existing works neither analyze their geometric discrepancies nor characterize their inherent limitations, leading to obscured distinctions, limited performance and suboptimal efficiency. In this section, we analyze the optimization dynamics of different flatness formulation and validate the unrevealed over-exploitation and over-exploration issues in current methods.

The definition duality between average- ($\overline{R}^{(n)}_{\xi}$) and worst-case ($\widehat{R}^{(n)}_{\xi}$) flatness fundamentally governs their optimization behaviors. The expectation in $\overline{R}^{(n)}_{\xi}$ aggregates local gradient variations, prioritizing exploitation of dominant neighborhood. While encouraging smoother updates, this gradient averaging risks over-exploitation to suboptimal plateaus dominated by local gradients. In contrast, $\widehat{R}^{(n)}_{\xi}$'s maximization explicitly probes extreme sensitivity points, demanding exhaustive exploration to identify maximal $n$-th derivative discrepancies. Such over-exploration, however, often oscillates between sensitivity peaks, particularly under practical sampling constraints, due to the non-convexity of adversarial landscapes~\cite{pgd,c&w}.

To validate the over-exploitation and over-exploration issues, we generate adversarial examples on Res-50~\cite{resnet} using four flatness-enhanced methods (RAP~\cite{rap}, FEM~\cite{femi}, PGN~\cite{pgn}, TPA~\cite{tpa}) for 100 iterations. We measure inter-update gradient cosine similarity during optimization and evaluate transfer attack success rates against Inc-v3~\cite{incv3}. \revminor{The gradient similarity serves as an indicator of the optimization dynamics: high similarity signifies consistent update directions for steady exploitation, whereas low values reflect frequent shifts in search directions for broad exploration}.

The experimental results, illustrated in Fig.~\ref{fig:flat_rethink}, reveal a clear empirical divergence. Worst-case flatness optimization exhibit significantly lower inter-iteration gradient similarities (mean similarity: RAP 0.10, PGN 0.25) compared to average-case counterparts (FEM 0.38, TPA 0.88) after convergence. The persistently low gradient similarity in worst-case optimization reveals \revminor{an unstable oscillatory trajectory}, an indicator of over-exploration \revminor{driven by the hypersensitivity of the maximization objective to shifting local extrema in non-convex landscapes}. In contrast, high similarity in average-case optimization suggests over-exploitation of local gradient consensus, leading to premature convergence on suboptimal plateaus \revminor{as the expectation operator traps the search in dominant local directions}. \revminor{Ultimately, both pitfalls lead to premature stagnation of optimization, thereby strictly constraining the maximum achievable adversarial transferability.}

\subsection{Design Challenges}
\label{subsec:design_challenges}
Existing flatness-enhanced methods face two unresolved challenges: theoretical ambiguity and optimization imbalance. First, the heuristic adoption of flatness metrics lacks rigorous theoretical grounding to connect geometric flatness with adversarial transferability~\cite{rap,pgn}, with some work even questioning their inherent correlation~\cite{tpa}, leading to inconsistent metric selection and suboptimal performance. Second, as shown in Section~\ref{subsec:flatness_rethinking}, current methods suffer from a critical exploitation-exploration trade-off: average-case flatness optimization over-exploits local gradient consensus, causing premature convergence to suboptimal plateaus, while worst-case criteria induce over-exploration of non-informative sensitivity peaks, resulting in oscillatory updates. Addressing these challenges requires developing a unified framework that integrates theoretical insights into flatness-transferability relationships with adaptive mechanisms to harmonize local exploitation and global exploration, thereby maximizing adversarial transferability.

\section{Theoretical Analysis}
\label{sec:theory}

Having formalized flatness metrics (Section~\ref{sec:preliminary}), we now resolve the first key challenge: the ambiguous relationship between geometric flatness and adversarial transferability. We first formally define transferability via the adversarial loss gap, then prove that improved multi-order flatness enhances it. This theory not only addresses the heuristic limitations of prior work but also provides principled optimization criteria.

\subsection{Adversarial Transferability Formulation}
\label{subsec:formulation}

To formally characterize transferability, we propose the \emph{Adversarial Transferability Gap (ATG)}. \rev{We premise this definition on three explicit assumptions regarding the surrogate $\mathcal{F}$ and target $\mathcal{F'}$: (1) they share the same label space; (2) they employ the same loss function $J(\cdot)$; and (3) they exhibit aligned performance on clean samples (i.e., $J(\mathbf{x},y;\mathcal{F}) \approx J(\mathbf{x},y;\mathcal{F'})$).}

\begin{definition}[Adversarial Transferability Gap]
\label{def:ATG}
\rev{Given the aligned models $\mathcal{F}$ and $\mathcal{F'}$ (as assumed above), the ATG for an input–label pair \((\mathbf{x},y)\) and perturbation \(\boldsymbol{\delta}\) is defined as:
\revminor{
\small
\begin{equation}
\operatorname{ATG}\bigl((\mathbf{x},y),\boldsymbol{\delta};\mathcal{F},\mathcal{F'}\bigr)
\triangleq J(\mathbf{x}+\boldsymbol{\delta},y;\mathcal{F'}) - J(\mathbf{x}+\boldsymbol{\delta},y;\mathcal{F}).
\end{equation}
}
\noindent \textit{Remark} A detailed justification of this metric's reasonableness and a discussion of its limitations under assumption violations are provided in Section~\ref{sec:atg_justificaton} of the Supplementary Material.}
\end{definition}

In the next section, we will rigorously characterize the relationship between adversarial transferability and loss landscape flatness by leveraging the proposed \emph{adversarial transferability gap} (ATG) and \emph{average-case flatness} ($\overline{R}^{(n)}_{\xi}$).


\subsection{Flatness and Transferability}
\label{sec:flatness_transferability}

Based on the \rev{definition} established in Section~\ref{subsec:formulation}, \rev{particularly in the typical scenario where transferability degrades (i.e., $ATG<0$)}, achieving successful adversarial transfer requires simultaneously addressing two critical objectives: (1) maximizing the adversarial loss $J(\mathbf{x}+\boldsymbol{\delta},y;\mathcal{F})$ on the surrogate model $\mathcal{F}$ to strengthen the attack,, and (2) minimizing $|\operatorname{ATG}|$ to align adversarial effects between the surrogate $\mathcal{F}$ and the target $\mathcal{F'}$. While gradient-ascent attacks (\eg, PGD~\cite{madry2017towards}) readily achieve the first, systematically reducing $|\operatorname{ATG}|$ remains theoretically understudied and challenging. We therefore make this our theoretical focus.

The minimization of \(|\operatorname{ATG}|\) hinges on characterizing the geometric alignment of adversarial perturbations between divergent loss landscapes. We address this by establishing a multi-order flatness-dependent bound on \(|\operatorname{ATG}|\) through Theorem~\ref{thm:flatness-transferability}, which directly connects transferability gaps to three measurable geometric properties: surrogate model flatness \( \bar R^{(n)}_\xi(\mathbf{x};F) \), Target model flatness \( \bar R^{(n)}_\xi(\mathbf{x};F') \), and Cross-model gradient discrepancy \( C_n(\mathbf{x}) \).

\begin{theorem}[Flatness-based Bound on Transferability]
\label{thm:flatness-transferability}
Let $\mathcal{F}$ and $\mathcal{F'}$ be two models with loss $J(\mathbf x,y;\,\cdot\,)$. Assume $J$ is locally approximable by a Taylor expansion of order $N$. For $0 \le n \le N$, define
\revminor{
\small
\begin{equation}
C_n(\mathbf x)
=\mathbb{E}_{\mathbf{x'}\sim \mathrm{Unif}(B_{\xi}(\mathbf{x}))}
  \|\nabla^nJ(\mathbf{x'},y;F')-\nabla^nJ(\mathbf{x'},y;F)\|,
\end{equation}
}
and let $\bar R^{(n)}_{\xi}(\mathbf x;F)$ and $\bar R^{(n)}_{\xi}(\mathbf x;F')$ be the $n$-order flatness for $F$ and $F'$. Then for perturbation $\|\boldsymbol\delta\|\le\xi$, the adversarial transferability gap satisfies
\revminor{
\small
\begin{equation}
\begin{split}
\bigl|\mathrm{ATG}((\mathbf{x},y),\boldsymbol\delta;\mathcal{F},F')\bigr|
&\le \sum_{n=0}^N\frac{\|\boldsymbol\delta\|^n}{n!}
\Bigl[ 
  \bar R^{(n)}_\xi(\mathbf{x};F) \\
  &\quad+ C_n(\mathbf{x}) + \bar R^{(n)}_\xi(\mathbf{x};F') \Bigr] \\
  &\quad+ \Delta R_N(\mathbf{x}, \boldsymbol\delta),
\end{split}
\end{equation}
}
where $\Delta R_N$ bounds the difference in higher-order residuals.
\end{theorem}
The proof of Theorem~\ref{thm:flatness-transferability} \rev{(see Section~\ref{proof:flat_transfer} in Supplementary Material)} employs finite-order multivariate Taylor expansions to decompose \(\operatorname{ATG}\) into curvature alignment terms plus a residual bound. This framework reveals that suppressing multi-order flatness variations and gradient discrepancies (\( C_n\)) minimizes the transferability gap. \rev{Crucially, $C_n(\mathbf{x})$ quantifies intrinsic \textit{surrogate-target dissimilarity}; a larger $C_n$ loosens the bound, acting as a natural barrier to transfer attack success.}

In transfer-based attacks where only surrogate $\mathcal{F}$ is accessible, optimizing source flatness $\bar R^{(n)}_\xi(\mathbf{x};F)$ becomes the sole effective strategy, as target flatness and $C_n(\mathbf{x})$ are uncontrollable. While TPA~\cite{tpa} linked flatness to transferability, it was limited to first-order analysis ($n=1$). Our theory incorporates higher-order interactions ($n \ge 1$), formalizing why heuristics~\cite{rap,gnp,pgn,femi,app} succeed. Theorem~\ref{thm:flatness-transferability} also shows that high-order terms diminish rapidly due to the $\frac{\|\boldsymbol\delta\|^n}{n!}$ coefficient. For ReLU networks where high-order derivatives vanish ($N \approx 1$), the zeroth-order flatness $\bar R^{(0)}_\xi(\mathbf{x};F)$ becomes the dominant factor, establishing it as an efficient optimization target.

Theorem~\ref{thm:flatness-transferability} also demonstrates that the impact of higher-order terms ($n\geq1$) diminishes exponentially due to the $\frac{\|\boldsymbol\delta\|^n}{n!}$ coefficient under $\ell_p$-bounded perturbations ($\|\boldsymbol\delta\|\ll1$), rendering zeroth-order flatness $\bar R^{(0)}_\xi(\mathbf{x};F)$, which governs expected loss variations in $B_{\xi}(\mathbf{x})$, the dominant contributor to transferability. This establishes zeroth-order flatness optimization as an efficient mechanism for maximizing adversarial transferability with theoretical guarantees.

\section{Method}
\label{sec:method}
\begin{figure*}[t]
    \centering
    
    \begin{minipage}{0.35\textwidth}
        \centering
        \includegraphics[width=\linewidth]{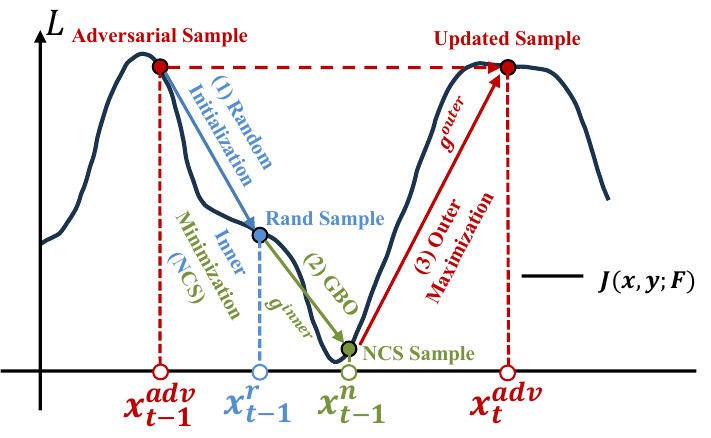}
        \centerline{\rev{(a) The Workflow of MEF Optimization}}
    \end{minipage}
    \hspace{30pt}
    \begin{minipage}{0.35\textwidth}
        \centering
        \includegraphics[width=\linewidth]{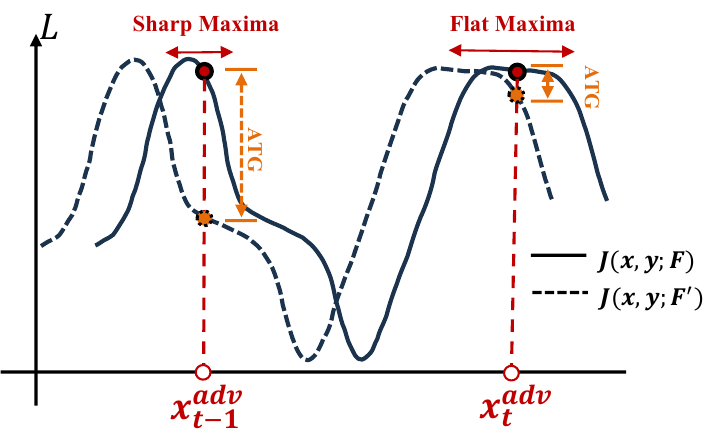}
        \centerline{\rev{(b) Geometric Interpretation}}
    \end{minipage}
    
    \caption{\rev{Overview of the proposed MEF attack. (a) Illustration of the iterative optimization process. At step $t-1$, Neighborhood Conditional Sampling (NCS) transforms random samples ($x^r$) into worst-case neighborhood samples ($x^n$) via inner minimization. Outer maximization then aggregates gradients from these regions to update the adversarial example $x_t^{adv}$. (b) Geometric interpretation showing how MEF seeks flat maxima to minimize the Adversarial Transferability Gap (ATG) between the source model $F$ (solid line) and the target model $F'$ (dashed line), as opposed to sharp maxima which lead to larger gaps.}}
    \label{fig:framework}
    \vspace{-5pt}
\end{figure*}


Building on the theoretical linkage between multi-order flatness and transferability established in Section~\ref{sec:theory}, we now address the second challenge outlined in Section~\ref{subsec:design_challenges}: designing a principled mechanism to balance local exploitation and global exploration. \rev{Instead of heuristically adopting existing metrics, our framework is strictly guided by Theorem~\ref{thm:flatness-transferability}: since the impact of higher-order terms ($n \ge 1$) diminishes exponentially due to the $\frac{\|\delta\|^n}{n!}$ coefficient, we prioritize the theoretically dominant zeroth-order flatness. Crucially, to resolve the exploitation-exploration dilemma inherent in standard definitions, we define a novel \textit{Zeroth-order Worst-Neighborhood Average Flatness} and optimize it
}through maximin expected flatness, neighborhood conditional sampling, and gradient balancing optimization.

\subsection{Maximin Expected Flatness}
\label{subsec:maximin}
Despite the computational efficiency and desirable theoretical properties of zeroth-order flatness metrics, existing zeroth-order flatness metrics exhibit polarized limitations: worst-case flatness $\widehat{R}^{(0)}_{\xi}(\mathbf{x})$ risks over-exploration by emphasizing local peaks, while average-case flatness $\overline{R}^{(0)}_{\xi}(\mathbf{x})$ tends toward over-exploitation by averaging out critical geometric features. To resolve the conflicting demands of local exploitation and global exploration in flatness optimization, we propose a novel zeroth-order worst-neighborhood average flatness metric that strategically integrates the stability of $\overline{R}^{(0)}_{\xi}(\mathbf{x})$ with the exploratory power of $\widehat{R}^{(0)}_{\xi}(\mathbf{x})$.

Our key insight lies in a hierarchical decomposition of flatness measurement that first identifies high-uneven sub-regions within an expanded neighborhood, then applies average-case smoothing to stabilize gradient estimation. Formally, we define the \emph{zeroth-order worst-neighborhood average flatness} as:
\revminor{
\small
\begin{equation}
\label{eq:wnac}
\widehat{\overline{R}}^{(0)}_{\gamma,\xi}(\mathbf{x}) 
\triangleq \max_{\mathbf{x}^{\rm c} \in B_{\gamma}(\mathbf{x})} \overline{R}^{(0)}_\xi(\mathbf{x}^{\rm c}),
\end{equation}
}
where $\gamma$ controls the exploration radius for identifying critical subregion centers $\mathbf{x}^{\rm c}$, and $\xi$ governs the neighborhood radius for zeroth-order average-case flatness computation. This hierarchical structure explicitly separates exploration scope ($\gamma$) from local smoothness ($\xi$), where larger $\gamma$ values permit discovery of distant high-loss regions while smaller $\xi$ values ensure stable gradient estimation within identified sub-regions.

With the proposed $\widehat{\overline{R}}^{(0)}_{\gamma,\xi}$, we establish the adversarial example generation framework guided by Theorem~\ref{thm:flatness-transferability}. As analyzed in Section~\ref{subsec:formulation}, successful transfer attacks require: 
\begin{inparaenum}[(i)]
\item minimizing $|\operatorname{ATG}|$ to align adversarial effects across models, and 
\item maximizing $J(\mathbf{x}+\boldsymbol{\delta},y;\mathcal{F})$ to elevate the upper bound of target model loss $J(\mathbf{x}+\boldsymbol{\delta},y;\mathcal{F'})$.
\end{inparaenum}
To unify these objectives, we formulate the attack generation as a multi-objective optimization problem:
\revminor{
\small
\begin{align}
\label{eq:original_opt}
\min_{\mathbf{x}^{\rm adv}\in B_{\epsilon}(\mathbf{x})}\Bigl[
  \underbrace{%
    \max_{\mathbf{x}^{\rm c}\in B_{\gamma}(\mathbf{x}^{\rm adv})}
      \widehat{\overline{R}}^{(0)}_{\gamma,\xi}(\mathbf{x}^{\rm c})
  }_{\text{Flatness Enhancement}}
  -
  \underbrace{%
    J(\mathbf{x}^{\rm adv}, y; \mathcal{F})
    \vphantom{%
      \max_{\mathbf{x}^{\rm c}\in B_{\gamma}(\mathbf{x}^{\rm adv})}
        \widehat{\overline{R}}^{(0)}_{\gamma,\xi}(\mathbf{x}^{\rm c})
    }
  }_{\text{Loss Maximization}}
\Bigr].
\end{align}
}
Substituting Definition~\ref{eq:wnac} into \eqref{eq:original_opt} yields:
\revminor{
\small
\begin{align}
\label{eq:expanded_opt}
\min_{\mathbf{x}^{\rm adv}\in B_{\epsilon}(\mathbf{x})}\Bigl\{ 
  &\max_{\mathbf{x}^{\rm c}\in B_{\gamma}(\mathbf{x}^{\rm adv})}
    \mathbb{E}_{\mathbf{x}'\sim \mathrm{Unif}(B_{\xi}(\mathbf{x}^{\rm c}))}\bigl[
      \bigl|J(\mathbf{x}') - J(\mathbf{x}^{\rm c})\bigr|
    \bigr] \nonumber \\
  &- J(\mathbf{x}^{\rm adv})
\Bigr\}.
\end{align}
}
Under the $\ell_\infty$-norm perturbation constraint $\|\mathbf{x}' - \mathbf{x}^{\rm c}\|_\infty \leq \xi$, $\|\mathbf{x}^{\rm adv} - \mathbf{x}^{\rm c}\|_\infty \leq \gamma$, we observe:
\revminor{
\small
\begin{equation}
\label{eq:approx_relation}
J(\mathbf{x}^{\rm adv}) - J(\mathbf{x}^{\rm c}) \approx 0, J(\mathbf{x}') - J(\mathbf{x}^{\rm c}) \leq 0.
\end{equation}
}
Substituting \eqref{eq:approx_relation} into \eqref{eq:expanded_opt}, we derive the final optimization objective:
\revminor{
\small
\begin{equation}
\label{eq:final_opt}
\max_{\mathbf{x}^{\rm adv}\in B_{\epsilon}(\mathbf{x})}
  \min_{\mathbf{x}^{\rm c}\in B_{\gamma}(\mathbf{x}^{\rm adv})}
    \mathbb{E}_{\mathbf{x}'\sim \mathrm{Unif}(B_{\xi}(\mathbf{x}^{\rm c}))} J(\mathbf{x}'),
\end{equation}
}
where the expectation term $\mathbb{E}\bigl[J(\mathbf{x}')\bigr]$ represents the \emph{expected flatness} over local neighborhoods. This maximin structure motivates our method's name: \textbf{Maximin Expected Flatness (MEF)}, which strategically explores worst-case neighborhoods while exploiting average loss landscapes.

\subsection{Neighborhood Conditional Sampling}
\label{subsec:sampling}
Directly solving the inner minimization in \eqref{eq:final_opt} is computationally infeasible due to the continuous loss landscape and infinite candidates in $B_{\gamma}(\mathbf{x}^{\rm adv})$. To efficiently approximate this intractable problem and obtain representative sampling points capturing worst-case flatness, we propose \textbf{Neighborhood Conditional Sampling (NCS)} method that leverages gradients for guided exploration.

The NCS uniformly samples $N$ points $\{\mathbf{x}_{0,i}\}_{i=1}^N$ within $B_{\xi}(\mathbf{x})$. At each iteration $t$, gradients $\nabla_{\mathbf{x}} J(\mathbf{x}_{t,i}, y; \mathcal{F})$ are computed for all points, followed by a sign-based update:
\revminor{
\small
\begin{equation}
\mathbf{x}_{t+1,i} = \mathbf{x}_{t,i} - \xi \cdot \mathrm{sign}\left(\nabla_{\mathbf{x}} J(\mathbf{x}_{t,i}, y; \mathcal{F})\right), 
\end{equation} 
}
which prioritizes directional consistency over gradient magnitudes. Updated points are then projected back to $B_{\xi}(\mathbf{x}_{0,i})$ and $B_{\gamma}(\mathbf{x})$ to enforce spatial constraints. After $T$ iterations, the final set $\{\mathbf{x}_{T,i}\}_{i=1}^N$ captures regions with minimal expected loss. This approach eliminates sampling overhead by leveraging gradients, reducing computation while maintaining efficiency.


\subsection{Gradient Balancing Optimization}
\label{subsec:grad_balance}
\begin{figure}[t]
    \centering
    \includegraphics[width=0.5\textwidth]{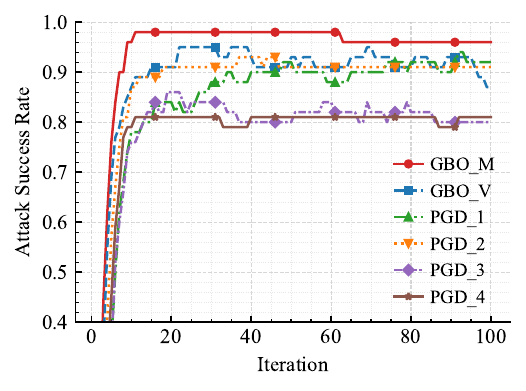}
    \vspace{-15pt}
    \caption{GBO's optimization superiority in transferability (Res-50~\cite{resnet}$\rightarrow$Inc-v3~\cite{incv3}). Achieving 90-98\% transfer attack success rates within 10 iterations, GBO outperforms PGD's 84-90\% at 100 steps with 6-14\% absolute improvement, while attaining 10× faster convergence on 100 ImageNet samples.}
    \vspace{-5pt}
    \label{fig:gbo}
\end{figure}
The nested maximin structure in \eqref{eq:final_opt} presents an optimization imbalance: the inner minimization requires fine-grained gradient descent across $N$ sampling points, while the outer maximization must update a single adversarial example $\mathbf{x}^{\rm adv}$ to elevate losses globally. This asymmetry arises because precise inner updates demand multiple iterations, whereas the outer optimization struggles to coordinate perturbations that simultaneously affect all minimized sub-regions. Experiments with gradient descent (1–4 inner steps) on 100 ImageNet images (Res-50→Inc-v3) confirm the issue: stricter inner optimization increases required convergence steps while reducing transfer attack success rates (Figure~\ref{fig:gbo}).

To address the imbalanced maximin optimization problem, we propose \textbf{Gradient Balancing Optimization (GBO)}, which strategically weakens the inner minimization's dominance by reusing outer maximization gradients. Let $\{g_{t,i}\}_{i=1}^N = \nabla_{\mathbf{x}}J({\{\mathbf{x'}_{t,i}\}}_{i=1}^N,y;\mathcal{F})$ denote the gradients from the outer maximization at iteration $t$. Instead of computing separate gradients for inner minimization, GBO approximates them as:
\revminor{
\small
\begin{equation}
\nabla_{\mathbf{x}_{t,i}^{\rm '}} J(\mathbf{x}_{t,i}^{\rm '}, y; \mathcal{F}) \approx g_{t-1,i},  
\end{equation}
}
leveraging the geometric proximity between consecutive adversarial examples. This shared gradient scheme entirely eliminates the computational overhead of inner minimization by reusing outer gradients for inner updates, accelerates convergence through aligned optimization directions, and improves transfer attack success rates.

To stabilize gradient directions, we further introduce momentum:
\revminor{
\small
\begin{equation}
\{g^{inner}_{t,i}\}_{i=1}^N = \frac{\{g_{t,i}\}_{i=1}^N}{\lvert\lvert\{g_{t,i}\}_{i=1}^N\rvert\rvert}_{1} - \mu_{inner}\cdot\{g^{inner}_{t,i}\}_{i=1}^N,
\end{equation}
}
where $\mu_{inner}$ controls historical gradient contributions. We compare GBO\_M (momentum-enhanced) and GBO\_V (vanilla) against PGD-based inner minimization with different iteration steps (1–4) on 100 ImageNet images. As shown in Fig.~\ref{fig:gbo}, GBO achieves faster convergence and higher transfer success rates than gradient-intensive PGD baselines, with GBO\_M (momentum-enhanced) outperforming GBO\_V (vanilla) in attack success rates. This demonstrates that gradient reuse with momentum stabilization not only resolves the optimization imbalance but also maximizes transferability, making GBO\_M our default solution.

\begin{algorithm}[t]
    \algnewcommand\algorithmicinput{\textbf{Input:}}
    \algnewcommand\Input{\item[\algorithmicinput]}
    \algnewcommand\algorithmicoutput{\textbf{Output:}}
    \algnewcommand\Output{\item[\algorithmicoutput]}

    \caption{Maximin Expected Flatness~(MEF) Attack}
    \label{alg:MEF}
    \begin{algorithmic}[1]
        \Input Surrogate model $\mathcal{F}$ and the loss function $J$; Raw example $\textbf{x}$ with ground-truth label $y$; The perturbation magnitude $\epsilon$; the maximum iterations $T$; the outer/inner decay factor $\mu_{outer}$/$\mu_{inner}$; the number of sampled examples, $N$; neighborhood radius $\xi$, exploration radius $\gamma$ and \rev{$U_{\gamma}(\cdot)$ is the uniform distribution over the $\ell_\infty$-ball $B_\gamma(\cdot)$}.
        \Output An adversarial example $\mathbf{x}^{adv}$
        
        \State $g^{outer}_{0} = 0$; $\{g^{inner}_{0,i}\}_{i=1}^N = 0$; $\mathbf{x}^{adv}_{0}=\mathbf{x}$; $\alpha={\epsilon/}{T}$
        \For {$t = 0, 1, \ldots, T-1$}
            \State \textcolor{orange}{\# Neighborhood Conditional Sampling}
            
            \State ${\{\mathbf{x}_{t,i}\}}_{i=1}^N \sim U_{\gamma}(\mathbf{x}^{adv}_{t})$
            
            \State ${\{\mathbf{x'}_{t,i}\}}_{i=1}^N = {\{\mathbf{x}_{t,i}\}}_{i=1}^N + \xi \cdot sign(\{g^{inner}_{t,i}\}_{i=1}^N)$
            \State \textcolor{orange}{\# Gradient Calculation}
            \State $\{g_{t,i}\}_{i=1}^N = \nabla_{\mathbf{x}}J({\{\mathbf{x'}_{t,i}\}}_{i=1}^N,y;\mathcal{F})$
            \State \textcolor{orange}{\# Gradient Balancing Optimization}
            \State $\{g^{inner}_{t,i}\}_{i=1}^N = \frac{\{g_{t,i}\}_{i=1}^N}{\lvert\lvert\{g_{t,i}\}_{i=1}^N\rvert\rvert}_{1} - \mu_{inner}\cdot\{g^{inner}_{t,i}\}_{i=1}^N$
            \State \textcolor{orange}{\# Outer Gradient Update}
            \State $g^{outer}_{t} = \mu_{outer}\cdot\ g^{outer}_{t} + \frac{1}{N}\sum_{i=1}^{N}{\frac{\{g_{t,i}\}_{i=1}^N}{\lvert\lvert\{g_{t,i}\}_{i=1}^N\rvert\rvert}_{1}}$
            \State \textcolor{orange}{\# Adversarial Example Update}
            \State $\mathbf{x}^{adv}_{t+1} = \Pi_{{B}_{\epsilon}(\mathbf{x})} [\mathbf{x}^{adv}_{t} + \alpha\cdot sign(g_{t}^{outer})]$;
        \EndFor
        \State\Return $\mathbf{x}^{adv}_{T}$.
    \end{algorithmic}
\end{algorithm}
\subsection{Algorithm}
\label{subsec:algorithm}
The proposed framework integrates the objective of \textbf{Maximin Expected Flatness}, \textbf{Neighborhood Conditional Sampling}, and \textbf{Gradient Balancing Optimization} through a unified max-min bi-level algorithm summarized in Algorithm~\ref{alg:MEF}. \rev{To clarify the interaction between these components, the overall workflow is visualized in Fig.~\ref{fig:framework}.} By strategically sharing normalized gradients between inner minimization and outer maximization, the algorithm achieves an optimal balance between computational efficiency and attack effectiveness.

\section{Experiments}
\label{sec:exp}

\begin{table*}[tb]
\caption{Transfer attack success rate ($\%$) comparison of gradient-stabilized attacks on normally trained models. The best results are \textbf{bold} and the second best are \underline{underlined}.}
\vspace{-10pt}
\label{tab:eval_normal_model}
\begin{center}
\begin{small}
\setlength{\tabcolsep}{4pt}
\setlength{\extrarowheight}{0.2pt}
\scalebox{0.75}{
\begin{tabular}{c|cccccccc|cccccccc}
\hline
\multirow{2}{*}{Attack} & \multicolumn{8}{c|}{\textbf{Res-50} $\Longrightarrow$} & \multicolumn{8}{c}{\textbf{Res-101} $\Longrightarrow$}  \\
 & Res-101 & Inc-v3 & Inc-v4 & IncRes-v2 & VGG-19 & Dense-121 & Xcept & \rev{AVG} & Res-50 & Inc-v3 & Inc-v4 & IncRes-v2 & VGG-19 & Dense-121& Xcept  & \rev{AVG} \\
\hline
EMI~\cite{pi-emi}
& 99.3 & 75.0 & 71.2 & 54.3 & 94.8 & 97.2 & 74.3 & \rev{80.9} & 99.8 & 76.8 & 72.6 & 58.5 & 92.9 & 97.2 & 74.0 & \rev{81.7} \\

APP~\cite{app}
& 99.3 & 83.1 & 79.6 & 73.1 & 95.3 & 97.8 & 83.3 & \rev{87.4} & 99.4 & 83.2 & 80.9 & 72.8 & 93.4 & 96.6 & 82.0 & \rev{86.9} \\

FEM~\cite{femi}
& \underline{99.8} & 89.4 & 86.3 & 78.8 & 97.4 & 98.7 & 87.7 & \rev{91.2} & \underline{99.8} & 89.1 & 86.4 & 77.6 & 95.4 & 98.0 & 87.3 & \rev{90.5} \\

PGN~\cite{pgn}
& 99.2 & 86.3 & 83.9 & 79.0 & 95.6 & 97.8 & 86.9 & \rev{89.8} & 99.5 & 86.8 & 84.5 & 80.0 & 94.5 & 97.1 & 88.1 & \rev{90.1} \\

\rev{ANDA~\cite{anda}} & \rev{99.4} & \rev{81.5} & \rev{80.1} & \rev{74.1} & \rev{90.7} & \rev{94.7} & \rev{82.2} & \rev{86.1} & \rev{99.5} & \rev{82.6} & \rev{81.2} & \rev{76.4} & \rev{91.0} & \rev{92.8} & \rev{83.7} & \rev{86.7} \\

\rev{FGSRA~\cite{fgsra}} & \rev{98.4} & \rev{85.3} & \rev{83.2} & \rev{78.0} & \rev{94.6} & \rev{97.3} & \rev{86.0} & \rev{89.0} & \rev{98.6} & \rev{86.0} & \rev{83.9} & \rev{79.4} & \rev{93.9} & \rev{96.3} & \rev{87.3} & \rev{89.3} \\

\rev{GI~\cite{gifgsm}} & \rev{99.5} & \rev{82.4} & \rev{80.5} & \rev{75.1} & \rev{91.6} & \rev{94.8} & \rev{83.1} & \rev{86.7} & \rev{99.4} & \rev{83.2} & \rev{81.4} & \rev{76.7} & \rev{91.2} & \rev{93.5} & \rev{84.4} & \rev{87.1} \\

\rev{MUMODIG~\cite{mumodig}} & \rev{99.3} & \rev{84.4} & \rev{82.5} & \rev{77.1} & \rev{93.6} & \rev{96.8} & \rev{85.1} & \rev{88.4} & \rev{99.4} & \rev{85.2} & \rev{83.4} & \rev{78.7} & \rev{93.2} & \rev{95.5} & \rev{86.4} & \rev{88.8} \\

\rev{GAA~\cite{gaa}} & \rev{98.5} & \rev{82.5} & \rev{81.1} & \rev{75.1} & \rev{91.7} & \rev{95.7} & \rev{83.2} & \rev{86.8} & \rev{98.9} & \rev{83.6} & \rev{82.2} & \rev{77.4} & \rev{92.0} & \rev{93.8} & \rev{84.7} & \rev{87.5} \\

\rev{FoolMix~\cite{foolmix}} & \rev{99.1} & \rev{80.4} & \rev{78.5} & \rev{73.1} & \rev{89.6} & \rev{92.8} & \rev{81.1} & \rev{84.9} & \rev{99.4 } & \rev{81.2} & \rev{79.4} & \rev{74.7} & \rev{89.2} & \rev{91.5} & \rev{82.4} & \rev{85.4} \\

MEF\textsubscript{H}
& 99.7 & \underline{90.6} & \underline{88.7} & \underline{82.7} & \underline{98.3} & \underline{99.4} & \underline{90.4} & \rev{\underline{92.8}} & 99.7 & \underline{91.2} & \underline{88.2} & \underline{82.2} & \underline{97.0} & \underline{98.4} & \underline{90.0} & \rev{\underline{92.4}} \\

MEF\textsubscript{F}
& \textbf{99.9} & \textbf{95.1} & \textbf{94.2} & \textbf{91.3} & \textbf{99.2} & \textbf{99.8} & \textbf{94.7} & \rev{\textbf{96.3}} & \textbf{99.9} & \textbf{95.8} & \textbf{95.2} & \textbf{91.8} & \textbf{98.9} & \textbf{99.4} & \textbf{95.7} & \rev{\textbf{96.7}} \\

\hline
\multirow{2}{*}{Attack} & \multicolumn{8}{c|}{\textbf{Inc-v3} $\Longrightarrow$} & \multicolumn{8}{c}{\textbf{Inc-v4} $\Longrightarrow$}  \\
 & Res-50 & Res-101 & Inc-v4 & IncRes-v2 & VGG-19 & Dense-121 & Xcept& \rev{AVG} & Res-50 & Res-101& Inc-v3 & IncRes-v2 & VGG-19 & Dense-121 & Xcept  & \rev{AVG} \\
\hline
EMI~\cite{pi-emi}
& 78.8 & 74.3 & 78.5 & 72.2 & 79.5 & 80.5 & 78.6 & \rev{77.5} & 79.0 & 75.3 & 86.5 & 71.9 & 85.4 & 82.1 & 82.7 & \rev{80.4} \\

APP~\cite{app}
& 73.5 & 68.0 & 79.2 & 73.1 & 74.7 & 75.7 & 78.5 & \rev{74.7} & 69.0 & 64.7 & 77.1 & 73.5 & 74.7 & 72.5 & 80.7 & \rev{73.2} \\

FEM~\cite{femi}
& 82.1 & 77.1 & 87.2 & 82.1 & 81.8 & 82.1 & 84.5 & \rev{82.4} & 78.3 & 73.9 & 86.0 & 78.9 & 81.6 & 80.4 & 86.9 & \rev{80.9} \\

PGN~\cite{pgn}
& 84.6 & 79.9 & 88.4 & 84.9 & 84.6 & 86.9 & 89.1 & \rev{85.5} & 84.9 & 82.3 & 91.1 & 85.8 & 89.0 & 87.6 & 90.8 & \rev{87.4} \\

\rev{ANDA~\cite{anda}} & \rev{81.1} & \rev{76.6} & \rev{85.2} & \rev{80.6} & \rev{80.8} & \rev{81.9} & \rev{85.2} & \rev{81.6} & \rev{80.9} & \rev{77.7} & \rev{86.3} & \rev{81.0} & \rev{84.2} & \rev{83.8} & \rev{87.5} & \rev{83.1} \\

\rev{FGSRA~\cite{fgsra}} & \rev{84.0} & \rev{79.3} & \rev{87.8} & \rev{84.1} & \rev{83.9} & \rev{85.9} & \rev{88.4} & \rev{84.8} & \rev{84.1} & \rev{81.4} & \rev{90.1} & \rev{84.8} & \rev{88.0} & \rev{86.9} & \rev{90.2} & \rev{86.5} \\

\rev{GI~\cite{gifgsm}} & \rev{81.3} & \rev{76.7} & \rev{85.3} & \rev{81.3} & \rev{81.2} & \rev{82.9} & \rev{85.7} & \rev{82.1} & \rev{81.4} & \rev{78.5} & \rev{87.2} & \rev{81.9} & \rev{85.1} & \rev{84.2} & \rev{87.6} & \rev{83.7} \\

\rev{MUMODIG~\cite{mumodig}} & \rev{83.3} & \rev{78.7} & \rev{87.3} & \rev{83.3} & \rev{83.2} & \rev{84.9} & \rev{87.7} & \rev{84.1} & \rev{83.4} & \rev{80.5} & \rev{89.2} & \rev{83.9} & \rev{87.1} & \rev{86.2} & \rev{89.6} & \rev{85.7} \\

\rev{GAA~\cite{gaa}} & \rev{82.1} & \rev{77.6} & \rev{86.2} & \rev{81.6} & \rev{81.8} & \rev{82.9} & \rev{86.2} & \rev{82.6} & \rev{81.9} & \rev{78.7} & \rev{87.3} & \rev{82.0} & \rev{85.2} & \rev{84.8} & \rev{88.5} & \rev{84.1} \\

\rev{FoolMix~\cite{foolmix}} & \rev{79.3} & \rev{74.7} & \rev{83.3} & \rev{79.3} & \rev{79.2} & \rev{80.9} & \rev{83.7} & \rev{80.1} & \rev{79.4} & \rev{76.5} & \rev{85.2} & \rev{79.9} & \rev{83.1} & \rev{82.2} & \rev{85.6} & \rev{81.7} \\

MEF\textsubscript{H}
& \underline{85.8} & \underline{80.8} & \underline{89.4} & \underline{86.1} & \underline{85.2} & \underline{87.4} & \underline{89.9} & \rev{\underline{86.4}} & \underline{86.2} & \underline{82.4} & \underline{92.8} & \underline{86.1} & \underline{90.1} & \underline{88.5} & \underline{90.8} & \rev{\underline{88.1}} \\

MEF\textsubscript{F}
& \textbf{92.6} & \textbf{91.5} & \textbf{96.5} & \textbf{94.3} & \textbf{92.8} & \textbf{93.4} & \textbf{95.8} & \rev{\textbf{93.8}} & \textbf{91.4} & \textbf{90.3} & \textbf{94.5} & \textbf{91.2} & \textbf{92.2} & \textbf{91.9} & \textbf{95.1} & \rev{\textbf{92.4}} \\
\hline
\end{tabular}
}
\end{small}
\end{center}
\vspace{-20pt}
\end{table*}

\subsection{Experimental Setup}

\noindent\textbf{Dataset.} We conduct experiments on the ImageNet-compatible dataset consisting of 1,000 images (299×299 resolution with annotated labels), which serves as the standard benchmark in previous works\cite{tap,trip,svre,logit,rap,pgn}.

\noindent\textbf{Models.} We select Inception-v3 (Inc-v3)~\cite{incv3}, Inception-v4 (Inc-v4)~\cite{incv4}, ResNet-50 (Res-50)~\cite{resnet} and ResNet-101 (Res-101)~\cite{resnet} with TorchVision-pretrained weights~\cite{torchvision} for generating adversarial examples. Transferability evaluation covers: (1)~\emph{Standard CNNs}: Inc-v3~\cite{incv3}, Inc-v4~\cite{incv4}, Res-50~\cite{resnet}, Res-101~\cite{resnet}, Inception-ResNetv2 (IncRes-v2)~\cite{incv4}, DenseNet-121 (Dense-121)~\cite{dense}, VGG-19bn (VGG-19)~\cite{vgg} and Xception (xcept)~\cite{xception}; (2)~\emph{Cross-architectures}: MobileNet-v2 (MobileNet)~\cite{mobilenet}, PNASNet-5-Large (PNASNet-L)~\cite{pnasnet}, ViT-Base/16 (ViT-B/16)~\cite{vit}, ViT-Large/32 (ViT-L/32)~\cite{vit}, PiT-S~\cite{pit}, MLP-Mixer~\cite{mlp} and ResMLP~\cite{resmlp}; (3)~\emph{Defended models}: Adversarially trained Inc-v3$_{adv}$~\cite{advincv3}, Inc-v3$_{ens3}$~\cite{eat}, Inc-v3$_{ens4}$~\cite{eat} and IncRes-v2$_{ens}$~\cite{eat}; adversarial purification defenses HGD~\cite{hgd} and NRP~\cite{nrp}; certifiable defense RS~\cite{rs}; \rev{diffusion-based purifier DiffPure~\cite{diffpure} and DensePure~\cite{densepure} and state-of-the-art adversarially robust models from RobustBench~\cite{robustbench}: MIMIR~\cite{mimir}, ARES~\cite{ares} and MeanSparse~\cite{meansparse}}. \rev{(4)~\emph{Multimodal Architectures (Vision-Language Models)}: To investigate adversarial transferability towards emerging multimodal paradigms, we evaluate VLMs including LLaVA~\cite{llava}, Unidiffuser~\cite{unidiffuser}, MiniGPT-4~\cite{minigpt-4} and Qwen2.5-VL~\cite{qwen2.5-vl}.}

\noindent\textbf{Compared Methods.} We evaluate against \rev{43} transfer-based adversarial attacks spanning three categories: (1) \rev{11} \emph{data-driven methods} covering input augmentation (DI~\cite{di}, TI~\cite{tim}, SI~\cite{ni}, Admix~\cite{admix}, SSA~\cite{ssa}, \rev{STM~\cite{stm}, USMM~\cite{USMM}, DeCoWA~\cite{decowa}, L2T~\cite{l2t}, BSR~\cite{bsr} and OPS~\cite{ops}}); (2) \rev{8} \emph{model-driven methods} comprising model ensemble (SVRE~\cite{svre}, CWA~\cite{cwa}) and surrogate refinement (SGM~\cite{sgm}, LinBP~\cite{linbp}, \rev{AWT~\cite{awt}, FAUG~\cite{faug}, ANA~\cite{ana}, LL2S~\cite{ll2s}}) strategies; (3) \rev{24} \emph{optimization-driven methods} including feature-based (ILA~\cite{ila}, NAA~\cite{naa}, \rev{BFA~\cite{bfa}, P2FA~\cite{p2fa}}), generative (CDA~\cite{cda}, \rev{DiffAttack~\cite{diffattack}}), gradient stabilization (MI~\cite{mi}, NI~\cite{ni}, PI/EMI~\cite{pi-emi}, VMI/VNI~\cite{vmi-vni}, \rev{ANDA~\cite{anda}, FGSRA~\cite{fgsra}, GI~\cite{gifgsm}, MUMODIG~\cite{mumodig}, GAA~\cite{gaa}, FoolMix~\cite{foolmix}
}), and flatness-enhanced (RAP~\cite{rap}, PGN~\cite{pgn}, GNP~\cite{gnp}, FEM~\cite{femi}, APP~\cite{app}, TPA~\cite{tpa}) approaches. PGN~\cite{pgn} constitutes the state-of-the-art in optimization-driven methods. \rev{We further benchmark against 4 classical \emph{query-based attacks} (SimBA~\cite{simba}, Square~\cite{squareattack}, Boundary~\cite{boundaryattack} and RayS~\cite{rays}) to provide a holistic view under different threat models.}

\noindent\textbf{Hyper-parameters.} All baseline methods use their originally reported configurations unless specified \rev{(see Supp.~\ref{sec:supp_params})}. We unify common parameters across attacks: maximum perturbation $\epsilon=16/255$, and step size $\alpha=\epsilon/T$. Our MEF sets neighborhood radius $\gamma = 2\times\epsilon$, exploration radius $\xi = 0.15\times\epsilon$, outer/inner momentum coefficient $\mu_{outer}/\mu_{inner}=0.5/0.9$. MEF requires only one backpropagation per sample versus PGN (current SOTA)'s two. For comprehensive evaluation, we design two variants: MEF\textsubscript{H} (10 iterations, half PGN's cost) and MEF\textsubscript{F} (20 iterations, fully matching PGN's cost).

\subsection{Effectiveness Benchmarking against Existing Attacks}

\subsubsection{{Comparison with Gradient Stabilization Attacks}}
We \rev{benchmark MEF against 18 representative gradient stabilization attacks} on normally trained models. We generate adversarial examples on four source models and evaluate their transferability across eight diverse target models.

Table \ref{tab:eval_normal_model} demonstrates the significant superiority of our framework. \revminor{To maintain clarity and focus on the state-of-the-art, Table \ref{tab:eval_normal_model} reports methods with an average success rate exceeding 80\%. Results for other classical baselines (e.g., MI~\cite{mi}, RAP~\cite{rap}) are provided in Supp.~\ref{sec:exp_baseline} (Table~\ref{tab:supp_baselines}).} \rev{We compute the overall average success rate by averaging the AVG columns across all four surrogate settings.} \rev{Although FEM~\cite{femi} shows competitive performance on ResNet-based surrogates, PGN remains the second-best method overall with a higher global average (88.20\% vs. 88.13\% for FEM~\cite{femi}).} Notably, MEF\textsubscript{H} attains 89.93\% average success, surpassing the state-of-the-art PGN~\cite{pgn} \rev{(88.20\%)} despite requiring only half the budget. Matching PGN~\cite{pgn}'s cost, MEF\textsubscript{F} elevates the average success rate to 94.84\%, \rev{establishing MEF as the new state-of-the-art for adversarial transferability.}

\begin{table*}[t]
\caption{Transfer attack success rate ($\%$) comparison on diverse architectures. best results are \textbf{bold} and second best are \underline{underlined}.}
\label{tab:eval_diverse_models_main}
\vspace{-20pt}
\begin{center}
\begin{small}
\setlength{\extrarowheight}{0.15em}
\scalebox{0.7}{
\begin{tabular}{c|ccccccc|ccccccc}
\hline
\multirow{2}{*}{Attack} & \multicolumn{7}{c|}{\textbf{Res-50} $\Longrightarrow$} & \multicolumn{7}{c}{\textbf{Inc-v3} $\Longrightarrow$} \\ & MobileNet & PNASNet-L & ViT-B/16 & ViT-L/32 & PiT-S & MLP-Mixer & ResMLP & MobileNet & PNASNet-L & ViT-B/16 & ViT-L/32 & PiT-S & MLP-Mixer & ResMLP \\
\hline
FEM~\cite{femi} & 94.0 & 87.0 & 43.1 & 35.0 & 58.1 & 60.2 & 66.0 & 81.1 & 78.2 & 41.3 & 30.4 & 53.2 & 59.7 & 58.6 \\
PGN~\cite{pgn} & 91.6 & 88.6 & 51.4 & 45.8 & 58.7 & 66.7 & 70.2 & 84.9 & 82.6 & 50.1 & 40.6 & 60.2 & 66.1 & 65.6 \\
\rev{FGSRA~\cite{fgsra}} & \rev{88.8} & \rev{87.0} & \rev{49.8} & \rev{45.6} & \rev{57.1} & \rev{64.5} & \rev{70.9} & \rev{83.7} & \rev{81.7} & \rev{48.1} & \rev{39.6} & \rev{59.6} & \rev{64.2} & \rev{65.2} \\
\rev{MUMODIG~\cite{mumodig}} & \rev{89.5} & \rev{86.3} & \rev{48.8} & \rev{44.1} & \rev{56.9} & \rev{66.3} & \rev{69.0} & \rev{81.3} & \rev{81.1} & \rev{48.6} & \rev{39.0} & \rev{59.5} & \rev{65.7} & \rev{65.1} \\
\hdashline
MEF\textsubscript{H} & \underline{95.2} & \underline{90.3} & \underline{51.9} & \underline{46.0} & \underline{63.3} & \underline{66.9} & \underline{70.5} & \underline{85.2} & \underline{83.9} & \underline{51.7} & \underline{42.2} & \underline{60.4} & \underline{67.3} & \underline{66.3} \\
MEF\textsubscript{F} & \textbf{98.0} & \textbf{95.2} & \textbf{56.0} & \textbf{46.9} & \textbf{71.1} & \textbf{71.7} & \textbf{77.4} & \textbf{92.7} & \textbf{90.8} & \textbf{56.2} & \textbf{45.3} & \textbf{71.4} & \textbf{69.4} & \textbf{73.3} \\
\hline
\end{tabular}
}
\end{small}
\end{center}
\vspace{-15pt}
\end{table*}
\subsubsection{Evaluation on Diverse Network Architectures}
\label{sec:exp_archi_main}
\revminor{To assess cross-architecture transferability, we evaluate attack success rates on seven heterogeneous models, including MobileNet~\cite{mobilenet}, PNASNet~\cite{pnasnet}, Vision Transformers (ViT)~\cite{vit,pit}, and MLP-Mixers~\cite{mlp,resmlp}. Table \ref{tab:eval_diverse_models_main} compares MEF against the top-4 performing baselines (full comparison with 12 baselines is in Supp.~\ref{sec:exp_archi_supp}). Despite the significant inductive bias gap, MEF\textsubscript{F} consistently outperforms the state-of-the-art PGN~\cite{pgn}, achieving a remarkable 46.9\% success rate on ViT-L/32 (vs. PGN's 45.8\%) and 77.4\% on ResMLP (vs. PGN's 70.2\%). This confirms that optimizing zeroth-order flatness effectively extracts geometry-invariant features that transfer across fundamentally different architectural paradigms.}

\begin{table*}[t]
\caption{Transfer attack success rate (\%) comparison on ensemble surrogates. Best results are \textbf{bold} and second best are \underline{underlined}.}
\vspace{-15pt}
\label{tab:attack_ensemble_main}
\begin{center}
\begin{small}
\setlength{\extrarowheight}{0.15em}
\scalebox{0.68}{
\begin{tabular}{c|cccccccccccccc}
\hline
\multirow{2}{*}{Attack} & \multicolumn{14}{c}{\textbf{Res-50 + Inc-v3} $\Longrightarrow$} \\ & Res-101 & Inc-v4 & IncRes-v2 & Inc-v3$_{adv}$ & Inc-v3$_{ens3}$ & Inc-v3$_{ens4}$ & IncRes-v2$_{ens}$ & MobileNet & PNASNet-L & ViT-B/16 & ViT-L/32 & PiT-S & MLP-Mixer & ResMLP \\
\hline
FEM~\cite{femi} & 98.7 & 92.0 & 88.3 & 68.3 & 65.9 & 62.9 & 46.9 & 92.5 & 89.0 & 59.0 & 41.3 & 74.1 & 66.0 & 77.0 \\
PGN~\cite{pgn} & 97.8 & 91.9 & 89.4 & 79.5 & 78.2 & 77.5 & 65.6 & 93.9 & 91.9 & 68.7 & 50.2 & 76.9 & 74.5 & 81.8 \\
\rev{FGSRA~\cite{fgsra}} & \rev{94.9} & \rev{90.3} & \rev{87.4} & \rev{78.9} & \rev{76.4} & \rev{75.2} & \rev{66.3} & \rev{92.6} & \rev{90.9} & \rev{66.5} & \rev{49.1} & \rev{76.1} & \rev{72.5} & \rev{81.2} \\
\rev{MUMODIG~\cite{mumodig}} & \rev{95.5} & \rev{89.6} & \rev{86.0} & \rev{77.0} & \rev{76.0} & \rev{76.9} & \rev{64.5} & \rev{90.1} & \rev{90.2} & \rev{66.8} & \rev{48.4} & \rev{75.8} & \rev{73.9} & \rev{80.9} \\
\hdashline
MEF\textsubscript{H} & 98.9 & \underline{93.2} & \underline{90.2} & \underline{80.2} & \underline{79.8} & \underline{80.5} & \underline{67.0} & \underline{94.1} & \underline{92.4} & \underline{69.5} & \underline{51.7} & \underline{77.8} & \underline{75.3} & \underline{82.1} \\
MEF\textsubscript{F} & \textbf{99.4} & \textbf{96.6} & \textbf{95.1} & \textbf{85.7} & \textbf{83.3} & \textbf{82.6} & \textbf{70.4} & \textbf{96.7} & \textbf{95.1} & \textbf{75.4} & \textbf{55.2} & \textbf{83.8} & \textbf{80.4} & \textbf{89.0} \\
\hline
\end{tabular}
}
\end{small}
\end{center}
\vspace{-10pt}
\end{table*}
\begin{table*}[t]
\caption{Transfer attack success rate (\%) comparison with input augmentation attacks. The best results are \textbf{bold}.}
\label{tab:ia_gs}
\vspace{-15pt}
\begin{center}
\begin{small}
\setlength{\extrarowheight}{0.2pt}
\setlength{\tabcolsep}{4pt}
\scalebox{0.7}{
\begin{tabular}{c|ccccccc|ccccccc}
\hline
\multirow{2}{*}{Attack} & \multicolumn{7}{c|}{{Res-50} $\Longrightarrow$} & \multicolumn{7}{c}{{Inc-v3} $\Longrightarrow$} \\ & IncRes-v2 & Inc-v3$_{ens3}$ & Inc-v4$_{ens4}$ & IncRes-v2$_{ens}$ & ViT-L/32 & MLP-Mixer & \rev{{AVG}} & IncRes-v2 & Inc-v3$_{ens3}$ & Inc-v4$_{ens4}$ & IncRes-v2$_{ens}$ & ViT-L/32 & MLP-Mixer & \rev{{AVG}}
\\
\hline
TI~\cite{tim}
& 51.1 & 42.0 & 39.1 & 30.2 & 34.4 & 55.6 & \rev{42.1}& 52.0 & 37.0 & 37.2 & 26.0 & 28.8 & 53.7 & \rev{39.1}\\
MEF\textsubscript{H}+TI
& {81.7} & {77.4} & {77.8} & {68.2} & {57.3} & {71.1} & \rev{72.2}& {86.4} & {79.6} & {79.5} & {67.4} & {48.7} & {68.3} & \rev{71.6}\\
MEF\textsubscript{F}+TI
& {94.1} & {83.5} & {82.6} & {73.6} & {62.5} & {76.0} & \rev{78.7}& {94.5} & {89.0} & {88.3} & {78.8} & {56.3} & {75.6} & \rev{80.4}\\
\hdashline
\rev{STM~\cite{stm}} & \rev{84.9} & \rev{52.0} & \rev{49.1} & \rev{32.9} & \rev{37.5} & \rev{66.4} & \rev{53.8}& \rev{85.6} & \rev{54.2} & \rev{56.4} & \rev{32.7} & \rev{34.9} & \rev{63.8} & \rev{54.6}\\
\rev{USMM~\cite{USMM}} & \rev{84.8} & \rev{51.7} & \rev{48.8} & \rev{32.6} & \rev{36.9} & \rev{66.0} & \rev{53.5}& \rev{85.4} & \rev{53.4} & \rev{56.3} & \rev{32.3} & \rev{34.9} & \rev{63.3} & \rev{54.3}\\
\rev{DeCowA~\cite{decowa}} & \rev{90.1} & \rev{59.4} & \rev{56.7} & \rev{40.6} & \rev{44.3} & \rev{70.2} & \rev{60.2}& \rev{92.9} & \rev{62.3} & \rev{63.9} & \rev{40.2} & \rev{42.8} & \rev{71.6} & \rev{62.3}\\
\rev{L2T~\cite{l2t}} & \rev{89.6} & \rev{56.8} & \rev{53.8} & \rev{37.5} & \rev{41.6} & \rev{69.7} & \rev{58.2}& \rev{90.3} & \rev{58.5} & \rev{61.1} & \rev{37.0} & \rev{40.1} & \rev{69.3} & \rev{59.4}\\
\rev{BSR~\cite{bsr}} & \rev{89.4} & \rev{57.9} & \rev{55.0} & \rev{38.6} & \rev{43.1} & \rev{69.9} & \rev{59.0}& \rev{91.4} & \rev{59.5} & \rev{62.2} & \rev{38.2} & \rev{40.1} & \rev{70.0} & \rev{60.2}\\
\rev{OPS(10,5,5)~\cite{ops}} & \rev{{96.2}} & \rev{{84.3}} & \rev{{81.9}} & \rev{{72.3}} & \rev{{63.7}} & \rev{{77.7}} & \rev{79.3}& \rev{{95.9}} & \rev{{90.7}} & \rev{{88.8}} & \rev{{79.4}} & \rev{{57.1}} & \rev{{76.3}} & \rev{81.4}\\
\rev{OPS(10,30,30)~\cite{ops}} & \rev{\underline{98.5}} & \rev{\underline{90.6}} & \rev{\underline{89.2}} & \rev{\underline{83.7}} & \rev{\underline{81.1}} & \rev{\underline{98.1}} & \rev{\underline{90.2}} & \rev{\underline{98.7}} & \rev{\underline{93.1}} & \rev{\underline{93.2}} & \rev{\underline{86.6}} & \rev{\underline{75.3}} & \rev{\underline{91.9}} & \rev{\underline{89.8}} \\
\hdashline
MEF\textsubscript{H}
& 83.7 & 64.5 & 63.3 & 50.4 & 40.8 & 65.3 & \rev{61.3}& 86.1 & 65.1 & 65.4 & 45.4 & 35.2 & 65.5 & \rev{60.4}\\
MEF\textsubscript{F}
& 90.3 & 69.2 & 65.1 & 52.2 & 44.9 & 71.7 & \rev{65.6}& 94.3 & 69.0 & 69.6 & 47.9 & 40.3 & 69.4 & \rev{65.1}\\
\rev{MEF+TI(10,5$\times$5)} & \rev{{97.5}} & \rev{{85.2}} & \rev{{83.7}} & \rev{{75.7}} & \rev{{64.3}} & \rev{{78.9}} & \rev{80.9}& \rev{{96.3}} & \rev{{92.8}} & \rev{{90.8}} & \rev{{80.7}} & \rev{{59.4}} & \rev{{77.1}} & \rev{82.8}\\
\rev{MEF+TI(10,30$\times$30)} & \rev{\textbf{99.2}} & \rev{\textbf{92.9}} & \rev{\textbf{92.3}} & \rev{\textbf{87.3}} & \rev{\textbf{82.5}} & \rev{\textbf{99.2}} & \rev{\textbf{92.2}} & \rev{\textbf{99.7}} & \rev{\textbf{95.5}} & \rev{\textbf{96.9}} & \rev{\textbf{90.5}} & \rev{\textbf{78.0}} & \rev{\textbf{95.4}} & \rev{\textbf{92.7}}\\
\hline
\end{tabular}
}
\end{small}
\end{center}
\vspace{-10pt}
\end{table*}
\subsubsection{Attacking on Ensemble of Models}
\label{sec:exp_ensemble_main}
\revminor{We extend our evaluation to ensemble attacks by integrating Inc-v3~\cite{incv3} and Res-50~\cite{resnet} through logit averaging as the surrogate model. Table \ref{tab:attack_ensemble_main} compares MEF against the top-4 performing baselines (full comparison with 14 methods, including ensemble-specialized SVRE~\cite{svre} and CWA~\cite{cwa}, is provided in Supp.~\ref{sec:exp_ensemble_supp}). MEF\textsubscript{F} demonstrates consistent superiority, outperforming the state-of-the-art PGN~\cite{pgn} by margins of 1.6\%-8.2\% across diverse targets. This confirms that MEF captures transferable adversarial patterns across model ensembles.}

\subsubsection{{Integrated with Input Augmentation Attacks}}
\label{sec:exp_combination}
We compare MEF\textsubscript{H} and MEF\textsubscript{F} against \rev{10} representative input augmentation attacks (e.g., SSA~\cite{ssa}, \rev{DeCoWA~\cite{decowa}}). Following the standard setup, we generate adversarial examples on Res-50~\cite{resnet} and Inc-v3~\cite{incv3} and evaluate their transferability to challenging targets, including ensemble models and heterogeneous architectures. Furthermore, to demonstrate compatibility, we integrate sample transformations from \rev{five classical augmentation methods (e.g., TI~\cite{tim}, DI~\cite{di})} into MEF variants without introducing additional gradient computations.

As illustrated in Table~\ref{tab:ia_gs}, MEF\textsubscript{H} outperforms standard baselines like SSA~\cite{ssa} in terms of average attack success rate. \revminor{Due to space limits, Table~\ref{tab:ia_gs} focuses on the integration with TI~\cite{tim} and comparisons with advanced SOTA methods. The extensive integration results with DI~\cite{di}, SI~\cite{ni}, Admix~\cite{admix}, and SSA~\cite{ssa} are detailed in Supp.~\ref{sec:exp_supp_ia} (Table~\ref{tab:supp_ia}).} Integrating TI~\cite{tim} further boosts MEF\textsubscript{F} to an average success rate of \rev{79.55\%}, demonstrating strong compatibility with input transformations. \rev{We specifically compare against OPS~\cite{ops}, which relies on stacking multiple transformations. For fair comparison, we align the computational overhead by adjusting sampling counts ($5\times5$ and $30\times30$). Remarkably, using only a simple TI~\cite{tim} integration, MEF+TI(10,$30\times30$) achieves 92.45\%, surpassing the complex OPS~\cite{ops} framework (90.00\%). This high-sampling strategy effectively bridges the architectural gap, boosting success rates on ViT-L/32 from 56.8\% to 80.3\%. This confirms that optimizing intrinsic flatness is more effective than heuristic transformation ensembles.}

\begin{table}[t]
\caption{Transfer attack success rate (\%) comparison across attack paradigms. The best results are \textbf{bold}.}
\label{tab:exp_attacktype}
\vspace{-10pt}
\begin{center}
\begin{small}
\setlength{\tabcolsep}{4pt}
\setlength{\extrarowheight}{0.25pt}
\scalebox{0.7}{
\begin{tabular}{c|ccccccc}
\hline
Attack & Inc-v3 & Inc-v4 & Res-101 & Inc-v3$_{ens3}$ & Inc-v3$_{ens4}$ & IncRes-v2$_{ens}$ & AVG \\
\hline
SGM~\cite{sgm}
& 71.6
& 35.3
& 92.8
& 12.3
& 12.5
& 6.1
& 38.4
\\
LinBP~\cite{linbp}
& 72.7
& 47.9
& 95.4
& 13.2
& 13.3
& 6.8
& 41.6
\\
\rev{AWT~\cite{awt}} & \rev{89.1} & \rev{87.5} & \rev{97.9} & \rev{63.1} & \rev{61.4} & \rev{49.1} & \rev{74.7} \\
\rev{FAUG~\cite{faug}} & \rev{68.0} & \rev{64.1} & \rev{78.6} & \rev{41.5} & \rev{42.2} & \rev{27.4} & \rev{53.6} \\
\rev{ANA~\cite{ana}} & \rev{79.8} & \rev{75.4} & \rev{86.0} & \rev{51.4} & \rev{50.3} & \rev{38.9} & \rev{63.6} \\
\rev{LL2S~\cite{ll2s}} & \rev{88.2} & \rev{86.6} & \rev{96.7} & \rev{61.9} & \rev{60.6} & \rev{47.8} & \rev{73.6} \\
\hdashline
ILA~\cite{ila}
& 70.5
& 74.8
& 97.9
& 16.5
& 15.3
& 9.0
& 47.3
\\
NAA~\cite{naa}
& 82.7
& 82.4
& 98.3
& 30.1
& 29.6
& 21.0
& 57.4
\\
\rev{BFA~\cite{bfa}} & \rev{89.1} & \rev{87.3} & \rev{98.6} & \rev{63.2} & \rev{60.9} & \rev{48.6} & \rev{74.6} \\
\rev{P2FA(10,70)~\cite{p2fa}} & \rev{\underline{95.4}} & \rev{\underline{95.7}} & \rev{\underline{99.8}} & \rev{\underline{70.1}} & \rev{\underline{66.4}} & \rev{\underline{54.2}} & \rev{\underline{81.5}} \\
\hdashline
CDA~\cite{cda}
& 69.4
& 83.1
& 94.5
& 49.8
& 52.7
& 43.3
& 65.5
\\
\rev{DiffAttack~\cite{diffattack}} & \rev{41.4} & \rev{41.9} & \rev{52.6} & \rev{15.9} & \rev{14.0} & \rev{5.4} & \rev{28.5} \\
\hdashline
MEF\textsubscript{H}
& 90.6
& 88.7
& 99.7
& 64.5
& 63.3
& 50.4
& 76.2
\\
MEF\textsubscript{F}
& 95.1
& 94.2
& \textbf{99.9}
& 69.2
& 65.1
& 52.2
& 79.3
\\
\rev{MEF(10,70)} & \rev{\textbf{97.5}} & \rev{\textbf{96.6}} & \rev{\textbf{99.9}} & \rev{\textbf{73.1}} & \rev{\textbf{68.5}} & \rev{\textbf{55.7}} & \rev{\textbf{83.6}} \\
\hline
\end{tabular}
}
\end{small}
\end{center}
\vspace{-5pt}
\end{table}
\subsubsection{{Comparison with Other Types of Transfer-based Attacks}}
We validate MEF against three paradigms: surrogate refinement (SGM~\cite{sgm}, LinBP~\cite{linbp}, \rev{AWT~\cite{awt}, FAUG~\cite{faug}, ANA~\cite{ana}, LL2S~\cite{ll2s}}), feature-based (ILA~\cite{ila}, NAA~\cite{naa}, \rev{BFA~\cite{bfa}, P2FA~\cite{p2fa}}), and generative attacks (CDA~\cite{cda}, \rev{DiffAttack~\cite{diffattack}}). As shown in Table \ref{tab:exp_attacktype}, MEF consistently outperforms these approaches on Res-50~\cite{resnet}. Notably, MEF\textsubscript{F} (79.3\%) surpasses leading baselines like AWT~\cite{awt} (74.7\%) and BFA~\cite{bfa} (74.6\%) solely through geometry-aware input space optimization, avoiding reliance on architectural priors.

\rev{Among feature-based methods, the state-of-the-art method P2FA~\cite{p2fa} attains 81.5\% but incurs $3.5\times$ higher computational cost. Under a budget-aligned setting, our MEF(10,70) variant reclaims the lead with 83.6\% (+2.1\% over P2FA), confirming the fundamental superiority of zeroth-order flatness optimization over intermediate feature manipulation.}

\begin{table}[t]
\centering
\rev{
\caption{Comparison of Efficiency, Constraint Compliance, and Stealthiness on Inc-v3 ($\epsilon=16/255$).}
\vspace{-5pt}
\label{tab:query_vs_transfer}
\resizebox{1.0\linewidth}{!}{
\begin{tabular}{l|c|c|c|c|c}
\toprule
\textbf{Method} & \textbf{Category} & \textbf{Queries} & \textbf{ASR (\%)} & \textbf{LPIPS} $\downarrow$ & \textbf{SSIM} $\uparrow$ \\
\midrule
Boundary~\cite{boundaryattack} & Decision-based & 500 & 12.3 & 0.85 & 0.291 \\
Boundary~\cite{boundaryattack} & Decision-based & 5,000 & 24.5 & 0.72 & 0.346 \\
RayS~\cite{rays} & Decision-based & 500 & 59.8 & 0.18 & 0.887 \\
RayS~\cite{rays} & Decision-based & 5,000 & 94.3 & 0.05 & 0.947 \\
\midrule
SimBA~\cite{simba} & Score-based & 500 & 7.80 & 0.01 & 0.991 \\
SimBA~\cite{simba} & Score-based & 5,000 & 41.8 & 0.01 & 0.992 \\
Square~\cite{squareattack} & Score-based & 500 & 94.9 & 0.31 & 0.583 \\
Square~\cite{squareattack} & Score-based & 5,000 & 99.7 & 0.30 & 0.579 \\
\midrule
\rowcolor{gray!10} \textbf{MEF\textsubscript{H}} & \textbf{Transfer-based} & \textbf{0} & \textbf{90.6} & \textbf{0.24} & \textbf{0.751} \\
\rowcolor{gray!10} \textbf{MEF\textsubscript{F}} & \textbf{Transfer-based} & \textbf{0} & \textbf{95.1} & \textbf{0.30} & \textbf{0.692} \\
\bottomrule
\end{tabular}
}
}
\vspace{-5pt}
\end{table}
\subsubsection{{Comparison with Query-based Attacks}}
\rev{We benchmark MEF against score-based (SimBA~\cite{simba}, Square~\cite{squareattack}) and decision-based (Boundary~\cite{boundaryattack}, RayS~\cite{rays}) attacks on Inc-v3~\cite{incv3}. Unlike feedback-dependent baselines, MEF operates in a zero-query setting using a ResNet-50 surrogate. As shown in Table~\ref{tab:query_vs_transfer}, while Square Attack~\cite{squareattack} achieves marginally higher success (99.7\%), it incurs prohibitive costs (5,000 queries) and severe visual degradation (SSIM 0.579). In contrast, our zero-query $\text{MEF}_{\text{F}}$ attains a competitive 95.1\%, outperforming the decision-based SOTA RayS~\cite{rays} (94.3\% at 5,000 queries). Even the half-cost $\text{MEF}_{\text{H}}$ maintains 90.6\% success with superior stealth (SSIM 0.751). This confirms MEF offers a practical solution without the detection risks and computational overhead of high-frequency queries.}

\subsection{Practical Threats to Real-World Applications}

\rev{
\begin{table}[t]
  \centering
  \caption{\rev{Transfer Attack Success Rates on Six Commercial Vision APIs. The best results are \textbf{bold} and the second best are \underline{underlined}.}}
  \label{tab:real_world_apis}
  \setlength{\tabcolsep}{2.5pt}
  \renewcommand{\arraystretch}{1.5}
  \vspace{-5pt}
  \resizebox{\columnwidth}{!}{
  \rev{
  \begin{tabular}{lccccccccc}
    \toprule
    \multirow{2}{*}{Method} & \multicolumn{3}{c}{Google Cloud Vision} & \multicolumn{3}{c}{Amazon Rekognition} & \multicolumn{3}{c}{Azure AI Vision} \\
    \cmidrule(lr){2-4} \cmidrule(lr){5-7} \cmidrule(lr){8-10}
    & T-1 & T-5 & T-10 & T-1 & T-5 & T-10 & T-1 & T-5 & T-10 \\
    \midrule
    PGN~\cite{pgn} & 67.20\% & 44.30\% & 40.00\% & 54.50\% & 36.60\% & 26.30\% & 59.10\% & 32.00\% & 27.40\% \\
    FEM~\cite{femi} & 65.30\% & 44.10\% & 39.00\% & 57.10\% & 40.00\% & 31.00\% & 54.40\% & 26.10\% & 21.20\% \\
    FGSRA~\cite{fgsra} & 68.60\% & 46.10\% & 40.90\% & 55.90\% & 39.70\% & 29.20\% & 51.70\% & 22.30\% & 18.50\% \\
    \hdashline
    \textbf{MEF\textsubscript{H}} & \underline{71.60\%} & \underline{50.20\%} & \underline{46.10\%} & \underline{58.50\%} & \underline{42.80\%} & \underline{33.90\%} & \underline{60.90\%} & \underline{34.70\%} & \underline{29.10\%} \\
    \textbf{MEF\textsubscript{F}} & \textbf{79.40\%} & \textbf{63.30\%} & \textbf{59.00\%} & \textbf{68.80\%} & \textbf{55.50\%} & \textbf{46.70\%} & \textbf{67.18\%} & \textbf{40.33\%} & \textbf{34.88\%} \\
    \midrule
    \multirow{2}{*}{Method} & \multicolumn{3}{c}{Baidu AI Cloud} & \multicolumn{3}{c}{Alibaba Cloud} & \multicolumn{3}{c}{Tencent Cloud} \\
    \cmidrule(lr){2-4} \cmidrule(lr){5-7} \cmidrule(lr){8-10}
    & T-1 & T-5 & T-10 & T-1 & T-5 & T-10 & T-1 & T-5 & T-10 \\
    \midrule
    PGN~\cite{pgn} & 79.80\% & 61.10\% & 61.10\% & 86.40\% & 66.10\% & 66.10\% & 74.90\% & 53.10\% & 51.90\% \\
    FEM~\cite{femi} & 80.60\% & 61.30\% & 61.30\% & 83.40\% & 64.20\% & 64.20\% & 72.50\% & 50.80\% & 49.80\% \\
    FGSRA~\cite{fgsra} & 74.50\% & 57.80\% & 57.80\% & 83.00\% & 66.00\% & 66.00\% & \underline{76.00\%} & 52.70\% & 51.50\% \\
    \hdashline
    \textbf{MEF\textsubscript{H}} & \underline{83.40\%} & \underline{62.70\%} & \underline{62.70\%} & \underline{86.80\%} & \underline{66.80\%} & \underline{66.80\%} &       {75.10\%} & \underline{54.70\%} & \underline{53.60\%} \\
    \textbf{MEF\textsubscript{F}} & \textbf{86.30\%} & \textbf{70.10\%} & \textbf{70.10\%} & \textbf{91.70\%} & \textbf{77.30\%} & \textbf{77.30\%} & \textbf{84.70\%} & \textbf{67.10\%} & \textbf{66.10\%} \\
    \bottomrule
  \end{tabular}
  }
}
\vspace{-10pt}
\end{table}
}
\subsubsection{{Evaluation on Commercial Vision APIs}}
\rev{To assess practical threats beyond open-source models, we evaluate transferability against six commercial black-box systems. Unlike local models, these APIs run proprietary architectures with unknown preprocessing defenses. Given query costs, we compare MEF against the three leading baselines using the \textit{Top-k Exclusion} (T-k) metric, where the correct ground-truth label is excluded from the top-k predictions.}

\rev{Table~\ref{tab:real_world_apis} demonstrates that MEF consistently outperforms baselines across all platforms. Notably, on the rigorous Google Cloud Vision API, MEF\textsubscript{F} achieves 79.40\% T-1, surpassing the SOTA PGN (67.20\%) by a significant margin of 12.2\%. Even the half-cost MEF\textsubscript{H} (71.60\%) exceeds the strongest baseline. High success rates are also observed on Alibaba Cloud (91.70\%) and others. These results confirm that the flat local minima identified by our framework generalize effectively to unknown production pipelines. For intuitive visual evidence of these successful semantic subversions, we provide qualitative comparisons of API outputs in Fig.~\ref{fig:gcv_visualization} of the Supp.~\ref{sec:supp_gcv}.}

\begin{table}[t]
    \centering
    \caption{\rev{Transferability comparison (Fooling Rate \%) against Multimodal Large Models. The best results are \textbf{bold} and the second best are \underline{underlined}}}
    \label{tab:mllm_results}
    \renewcommand{\arraystretch}{1.2}
    \resizebox{0.95\linewidth}{!}{
    \rev{
    \begin{tabular}{l|cccc|c}
        \hline
        \textbf{Attack Method} & \textbf{LLaVA} & \textbf{Unidiffuser} & \textbf{MiniGPT-4} & \textbf{Qwen2.5-VL} & \textbf{Average} \\
        
        \hline
        \rev{APP~\cite{app}} & \rev{19.1} & \rev{24.9} & \rev{23.8} & \rev{23.4} & \rev{22.8} \\
        FEM~\cite{femi} & {22.5} & 29.3 & 26.5 & {29.4} & {26.9} \\
        PGN~\cite{pgn} & 22.0 & {32.4} & {27.8} & 24.5 & 26.7 \\
        FGSRA~\cite{fgsra} & 20.3 & 26.4 & 25.7 & 24.7 & 24.3 \\
        \rev{ANDA~\cite{anda}} & \rev{19.6} & \rev{25.5} & \rev{24.9} & \rev{23.9} & \rev{23.5} \\
        \rev{GI~\cite{gifgsm}} & \rev{19.8} & \rev{25.7} & \rev{25.0} & \rev{24.1} & \rev{23.6} \\
        \rev{MUMODIG~\cite{mumodig}} & \rev{20.2} & \rev{26.2} & \rev{25.5} & \rev{24.5} & \rev{24.1} \\
        \rev{GAA~\cite{gaa}} & \rev{19.8} & \rev{25.8} & \rev{25.1} & \rev{24.1} & \rev{23.7} \\
        \rev{FoolMix~\cite{foolmix}} & \rev{19.4} & \rev{25.2} & \rev{24.5} & \rev{23.6} & \rev{23.2} \\
        \hline
        \textbf{MEF\textsubscript{H}} & \underline{26.7} & \underline{35.4} & \underline{30.5} & \underline{33.4} & \underline{31.5} \\
        \textbf{MEF\textsubscript{H}} & \textbf{30.5} & \textbf{38.4} & \textbf{35.5} & \textbf{41.5} & \textbf{36.5} \\
        \hline
        
    \end{tabular}
    }
    }
\vspace{-10pt}
\end{table}
\subsubsection{{Transferability to Multimodal Large Language Models}}
\rev{To investigate transferability against Multimodal Large Language Models (MLLMs), we evaluate MEF on LLaVA~\cite{llava}, Unidiffuser~\cite{unidiffuser}, MiniGPT-4~\cite{minigpt-4}, and Qwen2.5-VL~\cite{qwen2.5-vl}. We generate adversarial examples on Res-50~\cite{resnet} and transfer them to these models. This challenging cross-task scenario requires disrupting semantic alignment across modalities.}

\revminor{We prompt models with ``Describe the image'' to generate captions $C_{cln}$ and $C_{adv}$ for clean and adversarial images. Let $D(\cdot)$ denote the CLIP text encoder~\cite{clip}. We report the \textit{CLIP-based Fooling Rate}, which is defined as the proportion of samples satisfying the semantic divergence condition:
\begin{equation}
\label{eq:mllm_metric}
\text{CosSim}\big(D(C_{cln}), D(C_{adv})\big) \le \tau,
\end{equation}
where $\tau=0.4$ is an empirical threshold adopted to distinguish successful attacks from benign caption variations.}

\rev{As shown in Table~\ref{tab:mllm_results}, MEF demonstrates superior cross-paradigm transferability. MEF\textsubscript{F} achieves the highest average fooling rate of 36.5\% across all models. Notably, even the efficient MEF\textsubscript{H} (31.5\% average) outperforms the best baseline FEM (26.9\%). On Qwen2.5-VL, MEF\textsubscript{F} reaches 41.5\%, surpassing baselines by over 12\%. This confirms that optimizing multi-order flatness creates robust adversarial features capable of compromising the visual encoders of advanced MLLMs (see Supp.~\ref{mllm_visual_exp} for qualitative examples).}

\subsection{Robustness against Defense Mechanisms}

\begin{table*}[tbp!]
\caption{Transfer attack success rate ($\%$) against modern defense mechanisms. Evaluating across three defense categories on seven protected targets, our MEF\textsubscript{H}/\textsubscript{F} outperform all competitors. The best results are \textbf{bold} and the second best are \underline{underlined}.}
\label{tab:eval_robust_model}
\vspace{-10pt}
\begin{center}
\begin{small}
\setlength{\extrarowheight}{0.25pt}
\scalebox{0.68}{
\begin{tabular}{c|cccccccc|cccccccc}
\hline
\multirow{2}{*}{Attack} & \multicolumn{8}{c|}{\textbf{Res-50} $\Longrightarrow$} & \multicolumn{8}{c}{\textbf{Inc-v3} $\Longrightarrow$} \\ & Inc-v3$_{adv}$ & Inc-v3$_{ens3}$ & Inc-v3$_{ens4}$ & IncRes-v2$_{ens}$ & RS & HGD & NRP & \rev{\textbf{AVG}} & Inc-v3$_{adv}$ & Inc-v3$_{ens3}$ & Inc-v3$_{ens4}$ & IncRes-v2$_{ens}$ & RS & HGD & NRP & \rev{\textbf{AVG}} \\
\hline
APP~\cite{app} & 50.1 & 49.0 & 48.5 & 34.6 & 38.6 & 60.3 & 88.0 & \rev{52.7} & 52.7 & 46.2 & 46.3 & 28.9 & 30.4 & 23.1 & 43.2 & \rev{38.7} \\
FEM~\cite{femi} & 50.2 & 49.3 & 46.4 & 31.4 & 36.4 & 72.2 & 89.5 & \rev{53.6} & 56.2 & 49.7 & 49.4 & 28.5 & 29.9 & 23.5 & 43.0 & \rev{40.0} \\
PGN~\cite{pgn} & 60.2 & 58.4 & 59.6 & 45.6 & 48.1 & 62.5 & 92.9 & \rev{61.0} & 69.2 & 61.1 & 62.2 & 41.6 & 38.1 & 27.7 & 54.6 & \rev{50.6} \\
\rev{ANDA~\cite{anda}} & \rev{50.2} & \rev{48.9} & \rev{47.3} & \rev{37.7} & \rev{41.0} & \rev{62.5} & \rev{87.7} & \rev{53.6} & \rev{58.2} & \rev{50.2} & \rev{50.7} & \rev{29.5} & \rev{29.4} & \rev{22.3} & \rev{45.1} & \rev{40.8} \\
\rev{FGSRA~\cite{fgsra}} & \rev{52.8} & \rev{51.4} & \rev{52.1} & \rev{36.7} & \rev{39.8} & \rev{63.6} & \rev{91.6} & \rev{55.4} & \rev{57.9} & \rev{50.5} & \rev{51.1} & \rev{31.9} & \rev{30.5} & \rev{26.1} & \rev{45.2} & \rev{41.9} \\
\rev{GI~\cite{gifgsm}} & \rev{50.5} & \rev{50.9} & \rev{48.1} & \rev{37.1} & \rev{37.9} & \rev{60.7} & \rev{89.7} & \rev{53.6} & \rev{57.5} & \rev{51.3} & \rev{52.2} & \rev{32.0} & \rev{32.1} & \rev{22.5} & \rev{44.1} & \rev{41.7} \\
\rev{MUMODIG~\cite{mumodig}} & \rev{53.2} & \rev{51.6} & \rev{51.8} & \rev{36.9} & \rev{40.2} & \rev{63.7} & \rev{87.7} & \rev{55.0} & \rev{57.2} & \rev{51.0} & \rev{52.1} & \rev{32.2} & \rev{32.1} & \rev{23.0} & \rev{46.1} & \rev{42.0} \\
\rev{GAA~\cite{gaa}} & \rev{49.7} & \rev{49.0} & \rev{51.1} & \rev{36.9} & \rev{39.5} & \rev{62.8} & \rev{85.6} & \rev{53.5} & \rev{55.2} & \rev{49.7} & \rev{50.0} & \rev{31.5} & \rev{31.7} & \rev{25.0} & \rev{45.1} & \rev{41.2} \\
\rev{FoolMix~\cite{foolmix}} & \rev{49.6} & \rev{50.7} & \rev{46.2} & \rev{34.9} & \rev{38.1} & \rev{62.2} & \rev{84.5} & \rev{52.3} & \rev{55.6} & \rev{48.5} & \rev{49.5} & \rev{29.7} & \rev{32.3} & \rev{23.4} & \rev{42.3} & \rev{40.2} \\
MEF\textsubscript{H} & \underline{62.4} & \underline{64.5} & \underline{63.3} & \underline{50.4} & \underline{49.2} & \underline{80.6} & \underline{94.6} & \rev{66.4} & \underline{69.7} & \underline{65.1} & \underline{65.4} & \underline{45.4} & \underline{38.3} & \underline{46.3} & \underline{54.8} & \rev{55.0} \\
MEF\textsubscript{F} & \textbf{70.1} & \textbf{69.2} & \textbf{65.1} & \textbf{52.2} & \textbf{51.5} & \textbf{83.0} & \textbf{97.2} & \rev{69.8} & \textbf{75.8} & \textbf{69.0} & \textbf{69.6} & \textbf{47.9} & \textbf{40.4} & \textbf{52.4} & \textbf{56.2} & \rev{58.8} \\
\hline
\end{tabular}
}
\end{small}
\end{center}
\vspace{-15pt}
\end{table*}
\subsubsection{{Evaluation on Standard Defense Methods}}
To evaluate robustness, we transfer adversarial examples generated on Inc-v3 and Res-50 to seven representative protected targets spanning adversarial training~\cite{advincv3,eat}, purification~\cite{hgd,nrp}, and certifiable defense~\cite{rs}, as detailed in Sec. IV-A. Table \ref{tab:eval_robust_model} compares MEF\textsubscript{H}/MEF\textsubscript{F} with \revminor{competitive} gradient stabilization attacks \revminor{that achieve an average success rate exceeding 40\%. Results for other baselines (e.g., MI~\cite{mi}, VMI~\cite{vmi-vni}) are provided in Supp.~\ref{sec:exp_supp_robust} (Table~\ref{tab:supp_robust}).} MEF\textsubscript{H} achieves \rev{60.70\%} average success rate across all defenses using half the computational budget of PGN~\cite{pgn}, \rev{a 4.86\% improvement over} PGN's \rev{55.84\%} baseline. When matching PGN's cost, MEF\textsubscript{F} further elevates the average success rate to \rev{64.30\%}, a \rev{8.46\%} absolute improvement. Notably, MEF\textsubscript{F} outperforms all competitors on individual defense types. This consistent superiority demonstrates our method's capability to bypass diverse defense mechanisms, establishing new state-of-the-art robustness against protected models.

\begin{table}[t]
\rev{
\caption{\rev{Transfer attack success rate ($\%$) against advanced defenses. The best results are \textbf{bold} and the second best are \underline{underlined}.}}
\label{tab:eval_robustbench_diffusion}
\vspace{-15pt}
\begin{center}
\begin{small}
\setlength{\tabcolsep}{2.5pt}
\setlength{\extrarowheight}{0.25pt}
\resizebox{\linewidth}{!}{
\rev{
\begin{tabular}{l|ccccc}
\hline
\multirow{2}{*}{\textbf{Attack}} & \multicolumn{2}{c}{\textbf{Diffusion-based Purifier}} & \multicolumn{3}{c}{\textbf{RobustBench Models}} \\
 & DiffPure\cite{diffpure} & DensePure\cite{densepure} & MIMIR\cite{mimir} & ARES\cite{ares} & MeanSparse\cite{meansparse} \\
\hline
\multicolumn{5}{c}\textbf{Surrogate Model: ResNet-50} \\
\hline
APP~\cite{app} & 40.2 & 40.0 & 8.1 & 9.1 & 9.9 \\
FEM~\cite{femi} & 42.8 & 42.0 & 8.1 & 9.1 & 9.9 \\
PGN~\cite{pgn} & \underline{46.9} & \underline{46.7} & \underline{9.5} & \underline{10.7} & \underline{11.6} \\
ANDA~\cite{anda} & 39.3 & 39.1 & 8.0 & 8.9 & 9.7 \\
FGSRA~\cite{fgsra} & 41.3 & 43.7 & 9.1 & 10.2 & 11.1 \\
GI~\cite{gifgsm} & 40.9 & 40.7 & 8.3 & 9.3 & 10.1 \\
MUMODIG~\cite{mumodig} & 41.4 & 41.3 & 8.4 & 9.4 & 10.2 \\
GAA~\cite{gaa} & 39.4 & 39.2 & 8.0 & 9.0 & 9.7 \\
FoolMix~\cite{foolmix} & 40.7 & 40.5 & 8.3 & 9.3 & 10.1 \\
\textbf{MEF\textsubscript{H}} & 44.5 & 44.3 & 8.8 & 9.8 & 10.7 \\
\textbf{MEF\textsubscript{F}} & \textbf{48.6} & \textbf{47.2} & \textbf{17.3} & \textbf{19.4} & \textbf{21.1} \\
\hline
\multicolumn{5}{c}\textbf{Surrogate Model: Inception-v3} \\
\hline
APP~\cite{app} & 34.8 & 34.4 & 7.0 & 7.9 & 8.6 \\
FEM~\cite{femi} & 41.6 & 40.9 & 8.4 & 9.4 & 10.2 \\
PGN~\cite{pgn} & \underline{45.1} & \underline{44.5} & \underline{9.1} & \underline{10.2} & \underline{11.1} \\
ANDA~\cite{anda} & 37.1 & 36.6 & 7.5 & 8.4 & 9.1 \\
FGSRA~\cite{fgsra} & 44.9 & 44.2 & 8.7 & 9.8 & 10.6 \\
GI~\cite{gifgsm} & 37.9 & 37.4 & 7.6 & 8.6 & 9.3 \\
MUMODIG~\cite{mumodig} & 37.6 & 37.1 & 7.6 & 8.5 & 9.3 \\
GAA~\cite{gaa} & 36.7 & 36.2 & 7.4 & 8.3 & 9.0 \\
FoolMix~\cite{foolmix} & 35.8 & 35.3 & 7.2 & 8.1 & 8.8 \\
\textbf{MEF\textsubscript{H}} & 43.7 & 42.9 & 8.5 & 9.6 & 10.4 \\
\textbf{MEF\textsubscript{F}} & \textbf{46.1} & \textbf{45.1} & \textbf{15.8} & \textbf{17.8} & \textbf{19.3} \\
\hline
\end{tabular}
}
}
\end{small}
\end{center}
}
\vspace{-10pt}
\end{table}
\subsubsection{{Evaluation on Advanced State-of-the-Art Defenses}}
\rev{To address the rapid evolution of defenses, we extend evaluation to state-of-the-art mechanisms: diffusion-based purification (DiffPure~\cite{diffpure}, DensePure~\cite{densepure}) and top-performing robust models from RobustBench~\cite{robustbench} (MIMIR~\cite{mimir}, ARES~\cite{ares}, MeanSparse~\cite{meansparse}). For comparison, we select the top-9 performing baselines from Table~\ref{tab:eval_robust_model}. As shown in Table~\ref{tab:eval_robustbench_diffusion}, these advanced defenses pose significant challenges. Nevertheless, MEF\textsubscript{F} consistently demonstrates superior resilience. Against purification defenses, it outperforms the runner-up PGN~\cite{pgn} by 1.7\% on DiffPure~\cite{diffpure}. More notably, against the highly robust MeanSparse model~\cite{meansparse}, MEF\textsubscript{F} achieves a remarkable success rate of 21.1\% (ResNet-50~\cite{resnet} as surrogate model), nearly doubling PGN~\cite{pgn}'s 11.6\%. This suggests that MEF's average-case flatness optimization generates adversarial examples that are structurally stable enough to survive stochastic purification and robust feature alignment.}

\subsection{Analysis of Efficiency, Stealthiness, and Sensitivity}

\subsubsection{{Convergence Analysis and Iteration Selection}}
\label{subsec:convergence}
\begin{figure}[tbp!]
\rev{
    \centering
    \includegraphics[width=0.8\columnwidth]{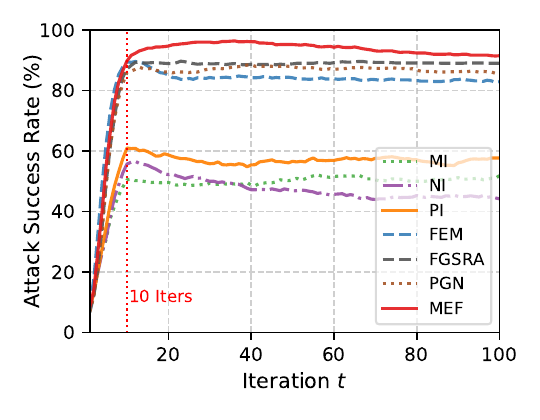}
    \vspace{-15pt}
    \caption{Attack Success Rate (ASR) convergence curve over 100 iterations (Source: Res-50~\cite{resnet} $\rightarrow$ Target: Inc-v3~\cite{incv3}). The dashed red line marks the conventional $T=10$ iteration cutoff.}
    \vspace{-5pt}
    \label{fig:convergence_curve}
}
\end{figure}
\begin{table}[t]
\caption{\rev{Convergence and Performance Comparison. $\Delta$ denotes the improvement from $T_{10}$ to $T_{max}$. The best results are \textbf{bold}.}}
\vspace{-5pt}
\centering
\label{tab:convergence_performance_full}
\resizebox{1.0\linewidth}{!}{
\rev{
\begin{tabular}{lcccc}
\toprule
\textbf{Method} & $\mathbf{T_{10}}$ ASR (\%) & $\mathbf{T_{\text{max}}}$ ASR (\%) & $\mathbf{\Delta}$ ASR (\%) & $\mathbf{T_{\text{max}}}$ Iter. \\
\midrule
MI~\cite{mi} & 50.50\% & 51.20\% & 0.70\% & 56 \\
NI~\cite{ni} & 55.70\% & 56.60\% & 0.90\% & 12 \\
PI~\cite{pi-emi} & 61.00\% & 61.00\% & 0.00\% & 10 \\
FEM~\cite{femi} & \underline{89.50\%} & \underline{89.60\%} & 0.10\% & 12 \\
FGSRA~\cite{fgsra} & 87.10\% & 89.70\% & 2.60\% & 23 \\
PGN~\cite{pgn} & 86.00\% & 88.50\% & 2.50\% & 40 \\
\hdashline
MEF (Ours) & \textbf{89.80\%} & \textbf{96.40\%} & 6.60\% & 36 \\
\bottomrule
\end{tabular}
}
}
\vspace{-10pt}
\end{table}

\rev{To evaluate the impact of iteration count $T$ on transferability, we conducted an extended convergence analysis ($T=100$) using Res-50~\cite{resnet} as the surrogate and Inc-v3~\cite{incv3} as the target. We selected six representative baselines, comprising both classic and SOTA methods. Since attackers in realistic black-box scenarios lack target model feedback to identify the optimal stopping point, assessing whether attacks can sustain transferability over extended iterations is crucial.}

\rev{The results, visualized in Fig.~\ref{fig:convergence_curve} and detailed in Table~\ref{tab:convergence_performance_full}, reveal distinct behaviors. For most baselines, including the SOTA PGN~\cite{pgn}, performance rapidly plateaus or even declines after reaching an early peak. This trend aligns with established findings~\cite{mi}, indicating that excessive iterations lead to overfitting on the surrogate model, thereby hampering transferability. Consequently, the gap between their $T_{10}$ performance and absolute peak ($T_{max}$) is marginal ($<3\%$), validating $T=10$ as a cost-effective standard for these methods. In contrast, MEF effectively mitigates overfitting through flatness optimization, exhibiting a significantly higher optimization ceiling. It continues to accumulate gains up to $T=36$, achieving a peak ASR of 96.40\% (a 6.6\% improvement over $T=10$). Crucially, MEF remains strictly superior to all competitors throughout the entire 100-iteration process, confirming its dominance for both resource-constrained ($T=10$) and determined attackers.}

\subsubsection{{Perceptual Quality Analysis}}
\label{subsec:quality}
\begin{figure}[t]
    \centering
    \includegraphics[width=0.8\linewidth]
    {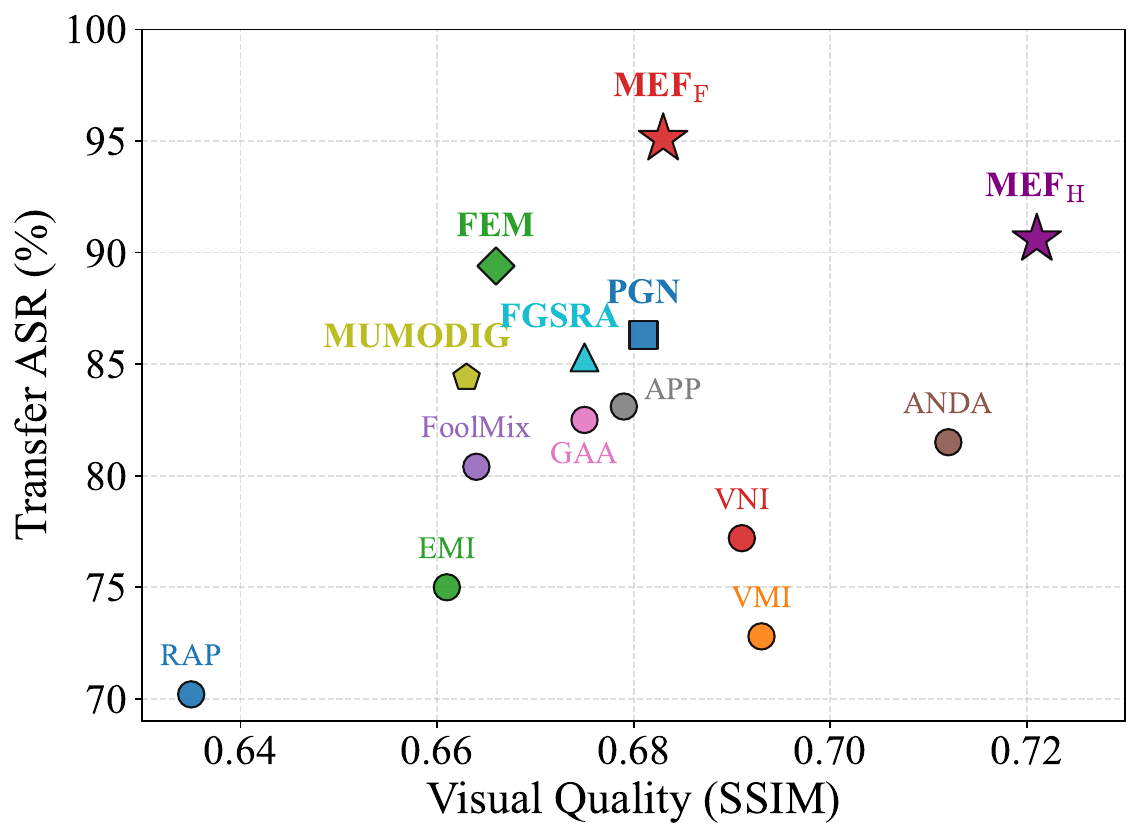}
    \vspace{-10pt}
    \caption{\rev{Comparison of Transferability (ASR) vs. Visual Quality (SSIM). Transfer ASR are valuated on Res-50~\cite{resnet} $\to$ Inc-v3~\cite{incv3} with $\epsilon=16/255$.}}
    \label{fig:visual_comparison}
    \vspace{-15pt}
\end{figure}
\begin{table}[t]
    \centering
    \caption{\rev{Quantitative comparison of visual perceptibility. The best results are \textbf{bold} and the second best are \underline{underlined}.}}
    \vspace{-10pt}
    \label{tab:top5_visual_quality}
    \renewcommand{\arraystretch}{1.3}
    \resizebox{1.0\linewidth}{!}{
    \rev{
    \begin{tabular}{l|c|cccc}
        \hline
        \multirow{2}{*}{\textbf{Method}} & \textbf{Transferability} & \multicolumn{4}{c}{\textbf{Visual Quality Metrics}} \\
        \cline{2-6} 
         & \textbf{ASR (\%)} $\uparrow$ & \textbf{SSIM} $\uparrow$ & \textbf{PSNR} $\uparrow$ & \textbf{LPIPS} $\downarrow$ & \textbf{DISTS} $\downarrow$ \\
        \hline
        MUMODIG~\cite{mumodig} & 84.4 & \underline{0.696} & \underline{26.86} & 0.323 & 0.247 \\
        FGSRA~\cite{fgsra}     & 85.3 & 0.683 & {26.62} & 0.313 & 0.252 \\
        PGN~\cite{pgn}         & 86.3 & 0.681 & 26.36 & \underline{0.294} &  0.250 \\
        FEM~\cite{femi}        & 89.4 & 0.666 & 26.44 & 0.322 & 0.251 \\
        \hdashline
        \textbf{MEF\textsubscript{H}} & \underline{90.6} & \textbf{0.721} & \textbf{28.35} & \textbf{0.250} & \textbf{0.241} \\
        \textbf{MEF\textsubscript{F}} & \textbf{95.1} & {0.692} & 26.56 & {0.295} & \underline{0.244} \\
        \hline
    \end{tabular}
    }
    }
\vspace{-10pt}
\end{table}
\rev{We evaluate perceptual quality ($\epsilon=16/255$) via the ASR-SSIM trade-off (Fig.~\ref{fig:visual_comparison}) and quantitative metrics (Table~\ref{tab:top5_visual_quality}). As shown in Fig.~\ref{fig:visual_comparison}, MEF occupies the optimal top-right quadrant, balancing strength and stealth. Quantitatively, $\text{MEF}_{\text{H}}$ achieves the best fidelity (SSIM 0.721, LPIPS 0.250), surpassing the prior leader MUMODIG~\cite{mumodig}. Meanwhile, $\text{MEF}_{\text{F}}$ yields the highest transferability (95.1\%) with competitive stealth (outperforming PGN~\cite{pgn} in DISTS). This confirms MEF's superiority stems from geometric efficiency. Qualitative results are in Supp.~\ref{sec:supp_visual}.}

\subsubsection{{Computational Efficiency Analysis}}
\begin{table}[t]
    \centering
    \caption{\rev{Computational efficiency and resource cost comparison. We report \revminor{computational cost (Backpropagations/BP), Peak GPU Memory (MiB), and} average runtime (seconds/image).}}
    \vspace{-5pt}
    \label{tab:efficiency}
    \renewcommand{\arraystretch}{1.2}
    \resizebox{0.9\linewidth}{!}{
    \rev{
    \begin{tabular}{l|c|c|c|c}
        \hline
        \textbf{Method} & \textbf{Time (s/img)} $\downarrow$ & \revminor{\textbf{BP (count)} $\downarrow$} & \revminor{\textbf{Peak Mem. (MiB)} $\downarrow$} & \textbf{ASR (\%)} $\uparrow$ \\
        \hline
        MI~\cite{mi} & 0.29 & \revminor{10} & \revminor{1397} & 50.5 \\
        NI~\cite{ni} & 0.29 & \revminor{10} & \revminor{1400} & 56.5 \\
        PI~\cite{pi-emi} & 0.30 & \revminor{10} & \revminor{1403} & 60.2 \\
        \hdashline
        VMI~\cite{vmi-vni} & 0.56 & \revminor{200} & \revminor{4316} & 72.8 \\
        VNI~\cite{vmi-vni} & 0.59 & \revminor{200} & \revminor{4315} & 77.2 \\
        EMI~\cite{pi-emi} & 0.48 & \revminor{200} & \revminor{4309} & 75.0 \\
        FEM~\cite{femi} & 2.80 & \revminor{210} & \revminor{1,788} & 89.4 \\
        FGSRA~\cite{fgsra} & 0.96 & \revminor{400} & \revminor{4100} & 85.3 \\
        PGN~\cite{pgn} & 1.13 & \revminor{400} & \revminor{4488} & 86.3 \\
        \hdashline
        \textbf{MEF\textsubscript{H}} & 0.46 & \revminor{200} & \revminor{4090} & \underline{90.6} \\
        \textbf{MEF\textsubscript{F}} & 0.97 & \revminor{400} & \revminor{4097} & \textbf{95.1} \\
        \hline
    \end{tabular}
    }
    }
    \vspace{-15pt}
\end{table}
\rev{We evaluate computational efficiency by measuring \revminor{the hardware-agnostic Backpropagation (BP) counts, Peak GPU Memory, and} the average wall-clock runtime required to generate an adversarial example (averaged over 100 randomly selected samples on an NVIDIA RTX 4090). For fair comparison, all sampling-based methods employ $N=20$ points with full synchronous parallel computation. We benchmark against basic single-point attacks and representative, well-established sampling-based methods.}

\rev{As detailed in Table~\ref{tab:efficiency}, single-point methods are naturally the fastest \revminor{(10 BPs, $\approx$ 1400 MiB)} but suffer from limited transferability. In contrast, MEF\textsubscript{H} achieves a remarkable 90.6\% success rate \revminor{with only 200 BPs and 4090 MiB memory}. Crucially, compared to the flatness-based SOTA PGN~\cite{pgn} \revminor{(400 BPs, 4488 MiB), MEF\textsubscript{H} halves the gradient cost and reduces memory footprint}, delivering a $\mathbf{2.4\times}$ speedup while improving ASR by 4.3\%. Even our full-cost variant MEF\textsubscript{F} remains faster than PGN \revminor{with lower memory usage}, and significantly outpaces heavy methods like FEM~\cite{femi} (2.80s). This confirms that our GBO strategy effectively mitigates the computational burden, achieving an optimal efficiency-performance ratio.}

\begin{figure*}[htbp!]
    \centering

    \subfloat[$\gamma$ \label{fig:gamma}]{
        \begin{minipage}[t]{0.33\textwidth}
            \centering
            \includegraphics[width=\linewidth]{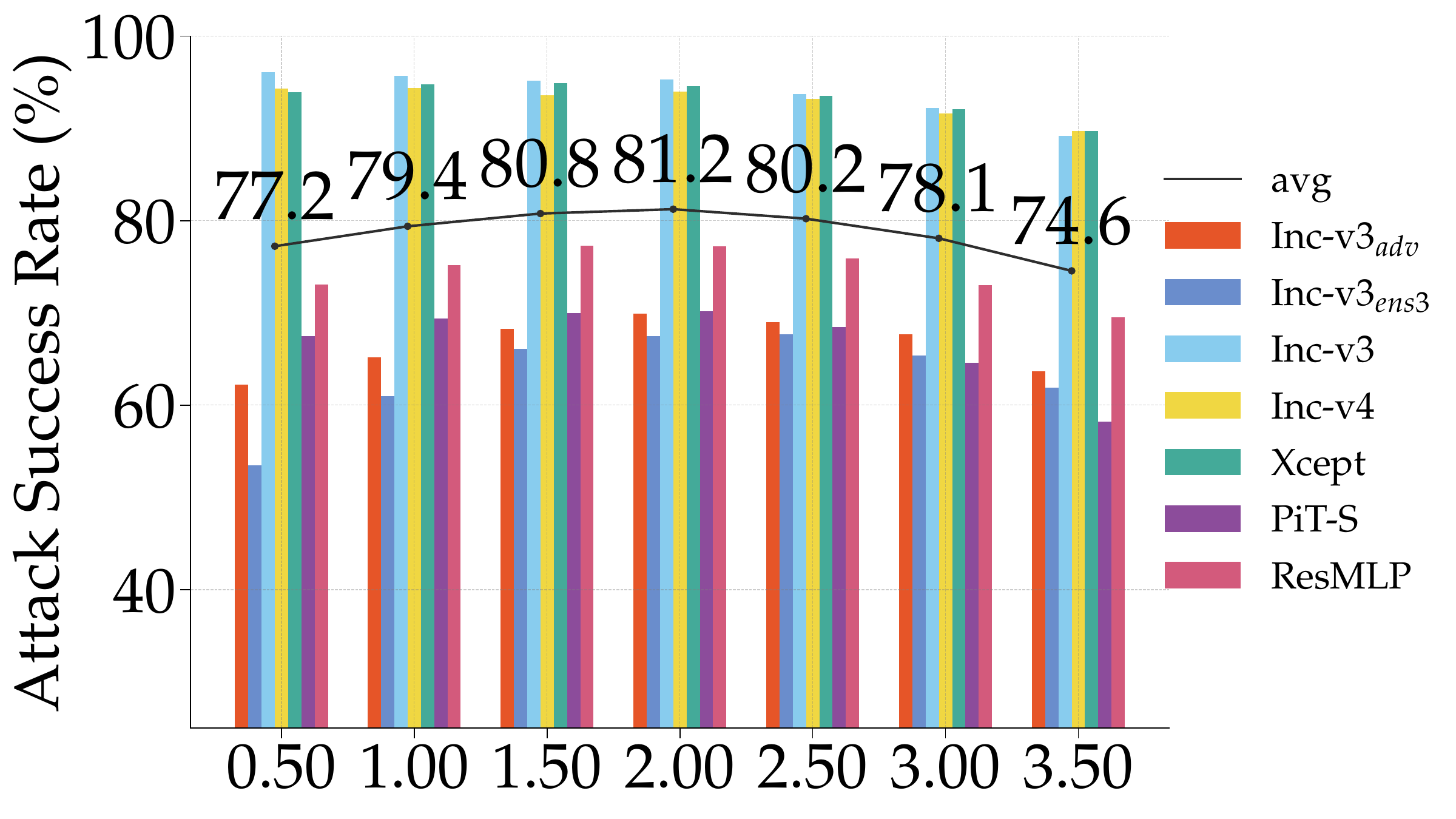}
        \end{minipage}
    }
    \hspace{-15pt}
    \subfloat[$\xi$ \label{fig:xi}]{
        \begin{minipage}[t]{0.33\textwidth}
            \centering
            \includegraphics[width=\linewidth]{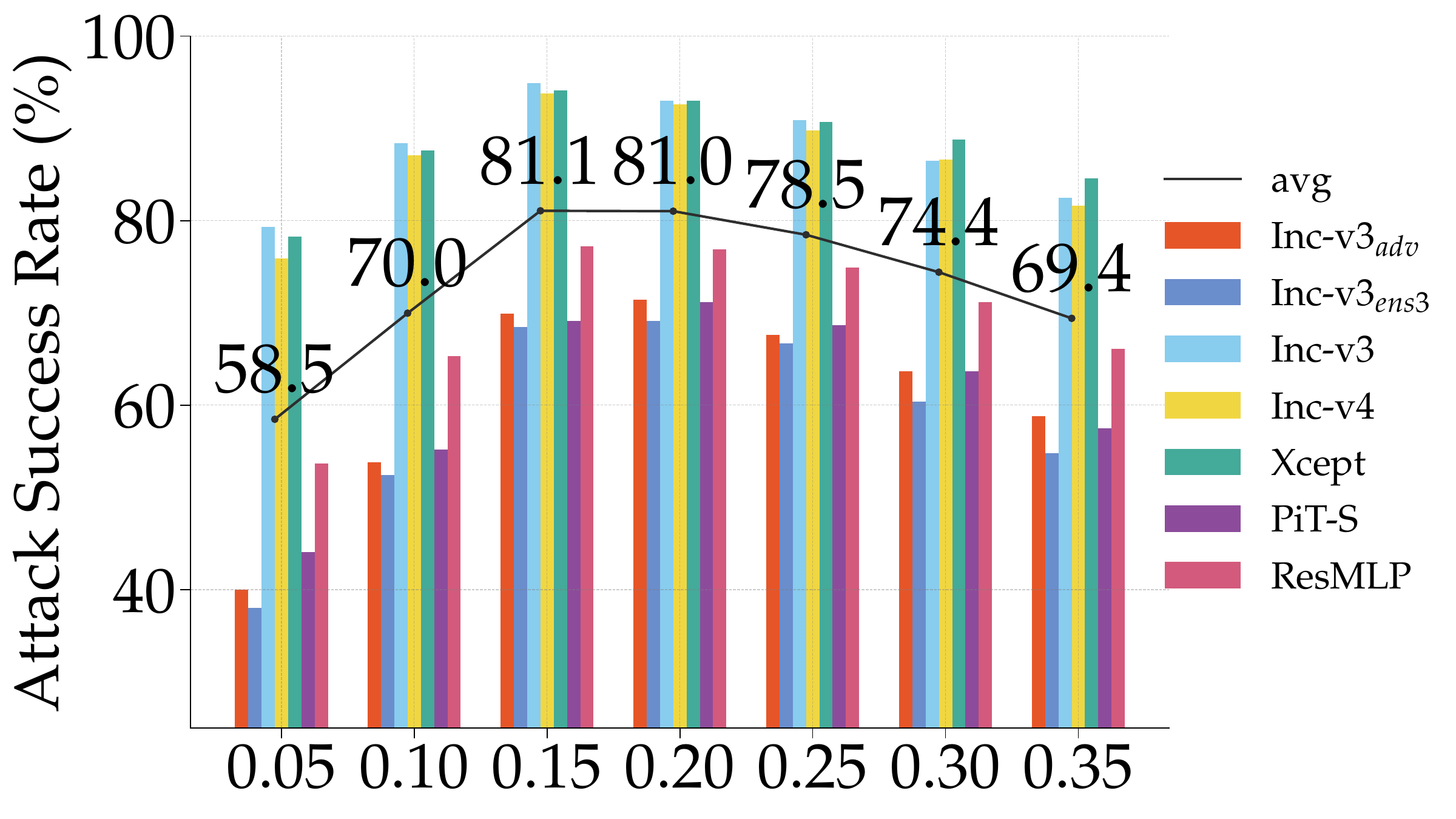}
        \end{minipage}
    }
    \hspace{-15pt}
    \subfloat[$N$ \label{fig:n}]{
        \begin{minipage}[t]{0.33\textwidth}
            \centering
            \includegraphics[width=\linewidth]{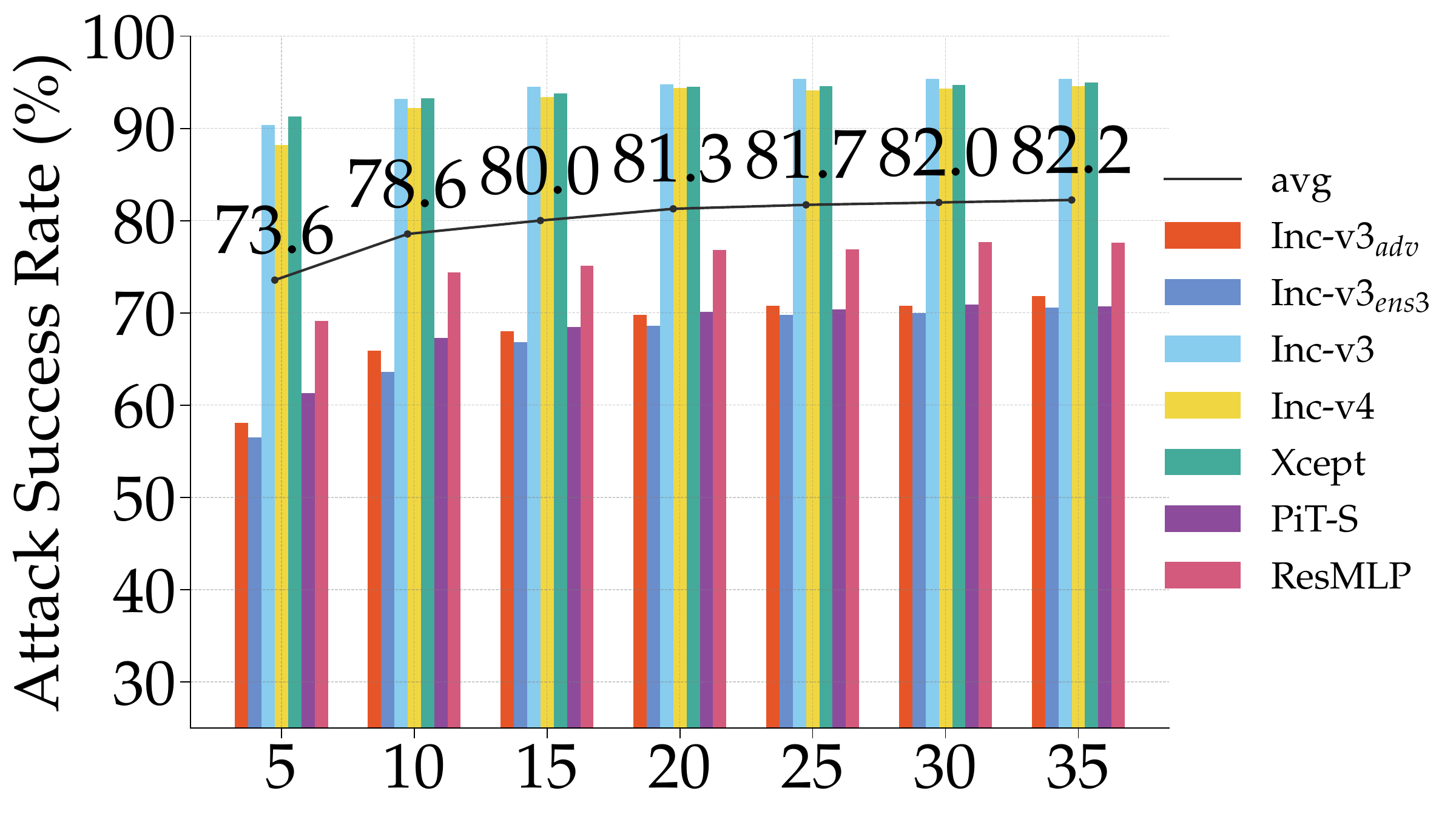}
        \end{minipage}
    }

    \vspace{-10pt}
    
    \subfloat[$\mu_{outer}$ \label{fig:muouter}]{
        \begin{minipage}[t]{0.33\textwidth}
            \centering
            \includegraphics[width=\linewidth]{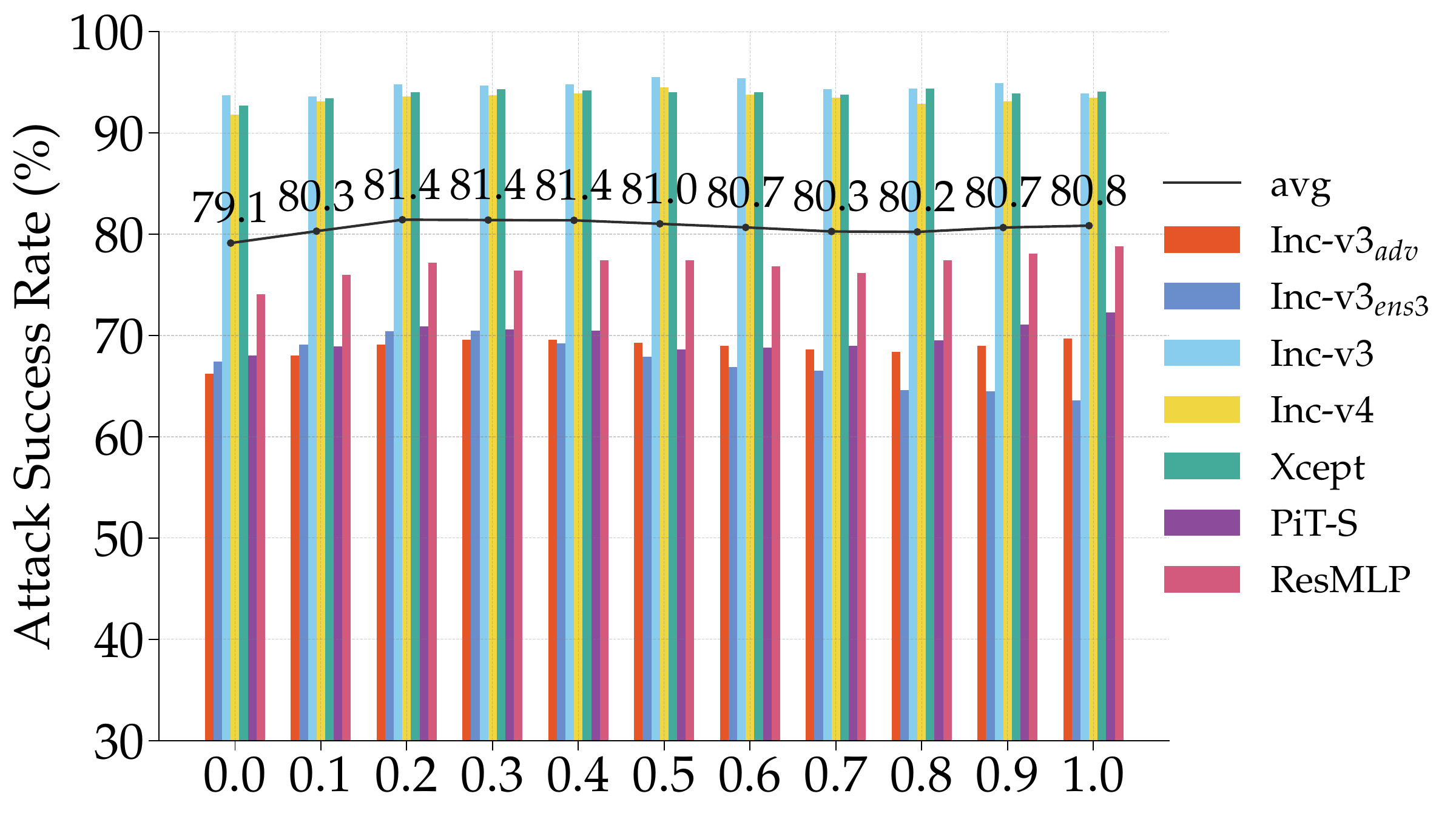}
        \end{minipage}
    }
    \hspace{-15pt}
    \subfloat[$\mu_{inner}$ \label{fig:muinner}]{
        \begin{minipage}[t]{0.33\textwidth}
            \centering
            \includegraphics[width=\linewidth]{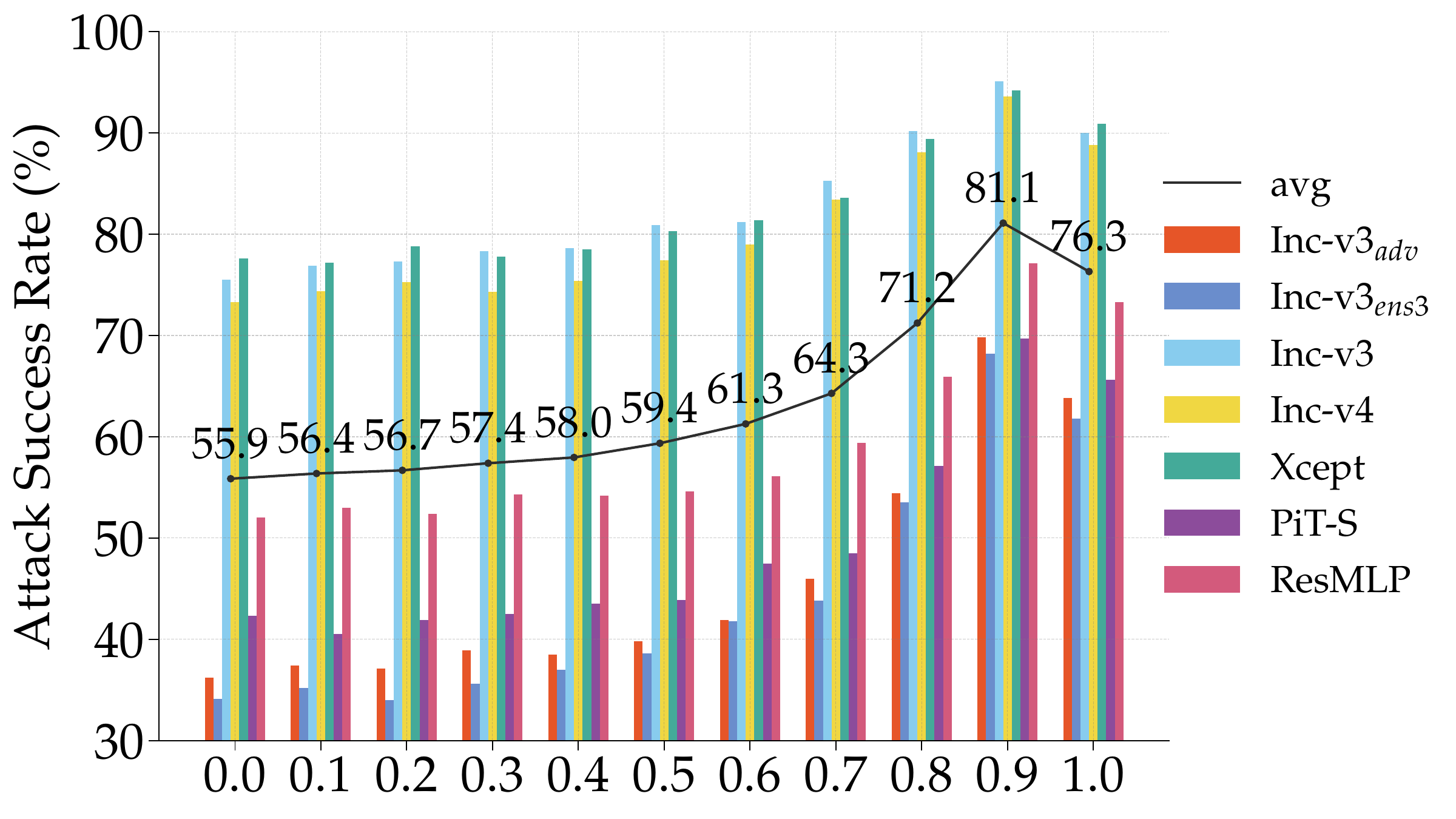}
        \end{minipage}
    }
    \hspace{-15pt}
    \subfloat[$T$ \label{fig:t}]{
        \begin{minipage}[t]{0.33\textwidth}
            \centering
            \includegraphics[width=\linewidth]{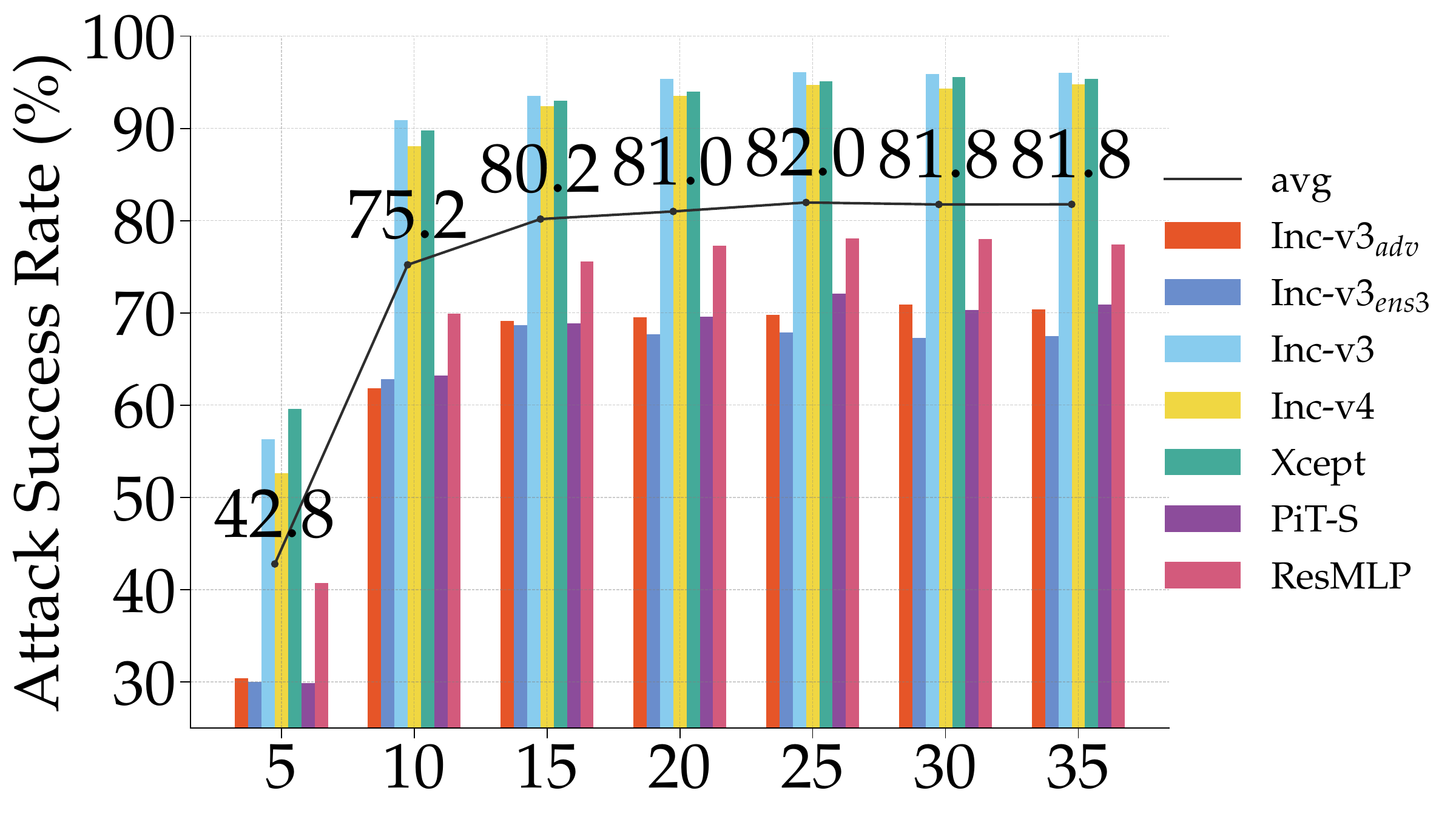}
        \end{minipage}
    }

    \vspace{-2.5pt}
    
    \caption{Parameter ablation study on MEF, evaluating $\gamma$/$\xi$, $\mu_{inner}$/$\mu_{outer}$, and $N$/$T$ impacts on transfer success rates across six challenging targets.}
    \label{fig:ablation_result}
    \vspace{-15pt}
\end{figure*}
\subsubsection{{Sensitivity to Perturbation Budget}}
\label{subsec:epsilon}
\rev{
To assess attack robustness under stricter imperceptibility constraints, we evaluate MEF against the top-4 performing baselines (Res-50 $\to$ Inc-v3) while varying $\epsilon$ from $2/255$ to $16/255$. As illustrated in Fig.~\ref{fig:epsilon_ablation}, MEF consistently outperforms competitors across all magnitudes. Notably, at the restrictive budget of $\epsilon=4/255$, MEF achieves a success rate of 41.3\%, surpassing the runner-up PGN (31.7\%) by a significant margin of 9.6\%. This confirms that MEF's superiority stems from the precise identification of geometrically flat directions rather than reliance on large perturbation norms. For detailed visual inspections across these budgets, please refer to Supp.~\ref{sec:supp_visual}.
}
\begin{figure}[tbp!]
    \centering
    \includegraphics[width=0.7\linewidth]{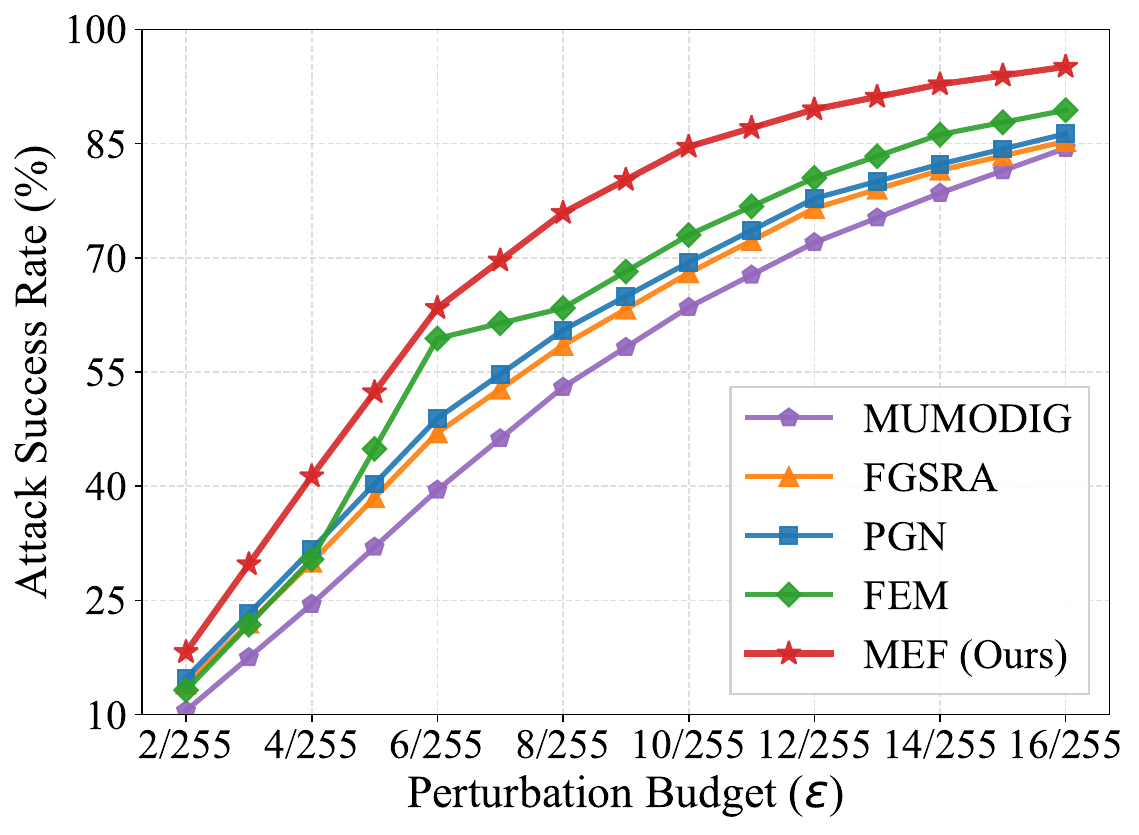}
    \vspace{-12.5pt}
    \caption{\rev{Attack success rate under varying perturbation budgets ($\epsilon$).}}
    \vspace{-10pt}
    \label{fig:epsilon_ablation}
\end{figure}

\subsubsection{{Parameter Sensitivity}}
To analyze the impact of key components in MEF, \rev{we evaluate six parameters ($\gamma, \xi, \mu_{outer}, \mu_{inner}, N, T$) across seven diverse targets, ranging from standard CNNs to defended models.} The results illustrated in Figure~\ref{fig:ablation_result} reveals three trends. First, \rev{both neighborhood radius $\gamma=2\epsilon$ and exploration radius $\xi=0.15\epsilon$ generalize as optimal values across all targets, demonstrating that MEF is robust to parameter choices and does not require laborious model-specific tuning.} Second, \rev{inner momentum $\mu_{inner}=0.9$ proves indispensable (delivering a 10\% boost), while outer momentum provides marginal gains.} Third, \rev{increasing sampling size $N$ shows continuous improvement, while extending iterations $T$ to 20 yields substantial gains (37.2\%).}

\begin{table}[t]
\centering
\caption{Quantitative Comparison of Zeroth/First-Order Flatness.}
\vspace{-8pt}
\label{tab:flatness_compare}
\begin{scriptsize}
\newcolumntype{Y}{>{\centering\arraybackslash}X}
\renewcommand{\arraystretch}{1}

\begin{tabularx}{0.7\columnwidth}{lYY} 
\toprule
\multirow{2}{*}{\textbf{Attack}} & \textbf{Res-50}~\cite{resnet} & \textbf{Inc-v3}~\cite{incv3} \\
 & ($\overline{R}^{(0)} / \overline{R}^{(1)}$) & ($\overline{R}^{(0)} / \overline{R}^{(1)}$) \\
\midrule
VMI~\cite{vmi-vni} & 29.03 / 0.49 & 8.32 / 0.64 \\
VNI~\cite{vmi-vni} & 23.04 / 0.41 & 7.74 / 0.59 \\
EMI~\cite{pi-emi}  & 25.57 / 0.45 & 8.77 / 0.62 \\
RAP~\cite{rap}     & 15.98 / 0.45 & 7.80 / 0.62 \\
GNP~\cite{gnp}     & 50.49 / 0.53 & 29.37 / 0.86 \\
FEM~\cite{femi}    & 12.20 / 0.55 & 3.52 / 0.79 \\
APP~\cite{app}     & 2.41 / 0.47  & 1.75 / 0.65 \\
TPA~\cite{tpa}     & 2.71 / 0.18  & 1.24 / 0.27 \\
PGN~\cite{pgn}     & 2.93 / 0.37  & 2.11 / 0.56 \\
\midrule
MEF\textsubscript{H} & \textbf{2.20} / 0.26 & 1.09 / 0.27 \\
MEF\textsubscript{F} & 2.34 / \textbf{0.10} & \textbf{1.06} / \textbf{0.21} \\
\bottomrule
\end{tabularx}
\end{scriptsize}
\vspace{-20pt}
\end{table}
\subsubsection{Landscape Flatness Analysis}
\label{sec:exp_flatness_analysis}
\revminor{We quantify landscape smoothness via Monte Carlo estimation of zeroth- ($\overline{R}^{(0)}_{\xi}$) and first-order ($\overline{R}^{(1)}_{\xi}$) flatness ($\xi = 32/255$). Table~\ref{tab:flatness_compare} reports the results on Res-50~\cite{resnet} and Inc-v3~\cite{incv3}. MEF\textsubscript{H} consistently achieves the lowest $\overline{R}^{(0)}_{\xi}$, while MEF\textsubscript{F} further minimizes $\overline{R}^{(1)}_{\xi}$, demonstrating that optimizing zeroth-order flatness inherently suppresses gradient variations. These metrics statistically confirm the correlation between flatness minimization and the enhanced transferability observed in Sec. VI-D. For loss landscape visualizations, please refer to Supp.~\ref{sec:supp_flatness_vis}.}
\section{Discussion and Conclusion}
\label{sec:conclusion}

\noindent\textbf{Discussion.} While MEF significantly outperforms state-of-the-art baselines across challenging \textit{cross-task} and \textit{cross-architecture} settings, strictly bridging the structural gap between surrogate and target models remains an open challenge. \revminor{Moreover, we acknowledge the inherent trade-off between transferability and visual stealthiness in the zero-query setting. Unlike white-box or query-based approaches, transfer-based attacks generally necessitate a larger perturbation budget to ensure the generalization of adversarial features to unknown targets. Nevertheless, MEF establishes a superior trade-off, achieving the most favorable balance between ASR and visual quality among existing transfer-based methods.} \revminor{Addressing these limitations constitutes a meaningful direction for future research. Specifically, future exploration of agnostic priors aims to reconcile the trade-off between cross-domain transferability and visual stealthiness.}

\noindent\textbf{Conclusion.} This work establishes the first theoretical foundation linking multi-order flatness to adversarial transferability. By unifying fragmented definitions and resolving the exploitation-exploration dilemma, we proposed the Maximin Expected Flatness (MEF) attack. Extensive evaluations demonstrate that MEF achieves state-of-the-art transferability and superior efficiency across diverse models and defenses. We hope this work inspires both deeper and broader adoption of principled flatness optimization in AI security.

\small{
\bibliographystyle{IEEEtran}
\bibliography{IEEEabrv, refer}
}

\clearpage

\begin{IEEEbiography}[{\includegraphics[width=1in,height=1.25in,clip,keepaspectratio]{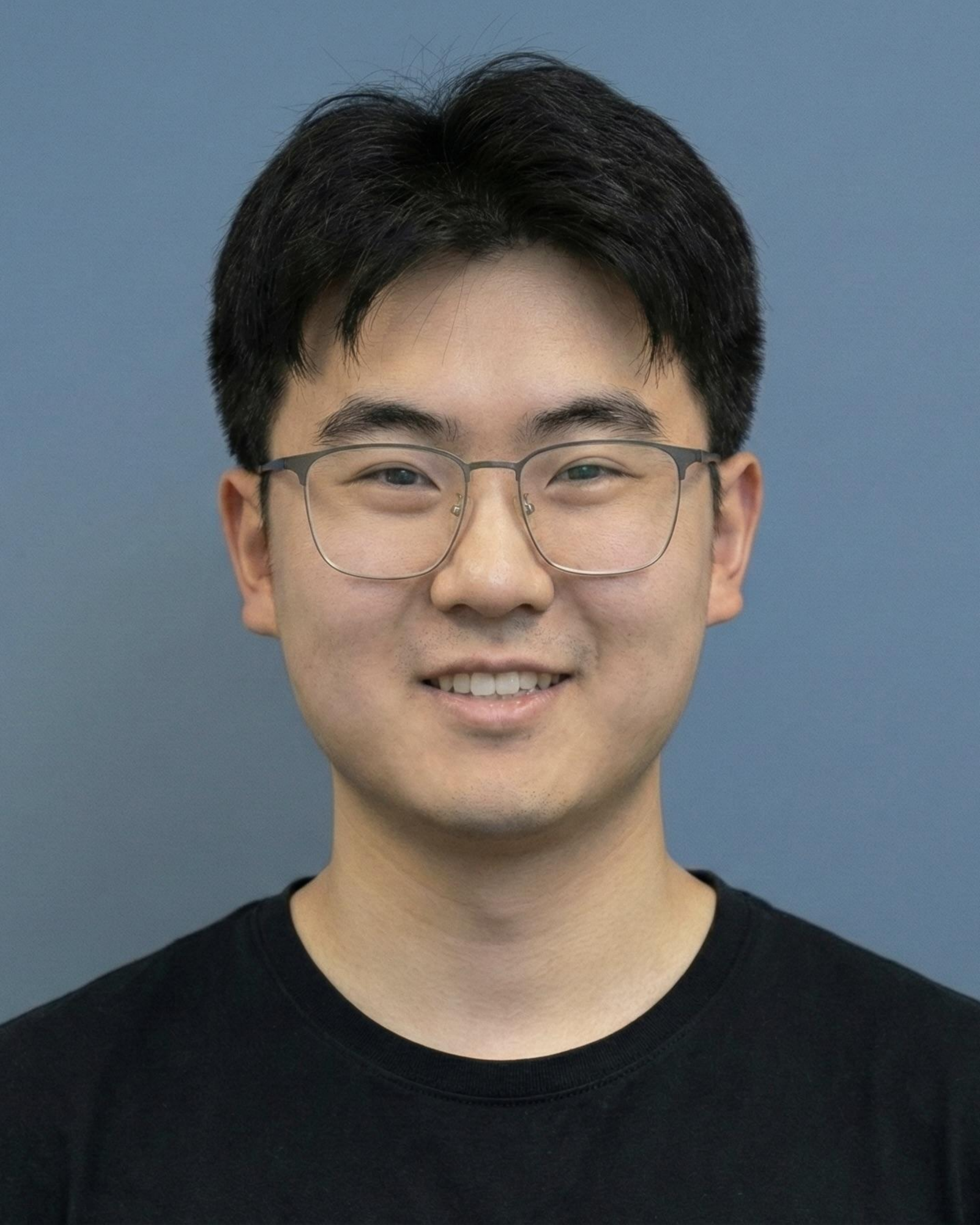}}]{Chunlin Qiu}
is working towards the Ph.D. degree in the School of Cyber Science and Engineering, Wuhan University, China. He received the B.E. degree in Information Security from Wuhan University, China, in 2023. His research interests primarily focus on AI Security and AI Safety.
\end{IEEEbiography}
\begin{IEEEbiography}[{\includegraphics[width=1in,height=1.25in,clip,keepaspectratio]{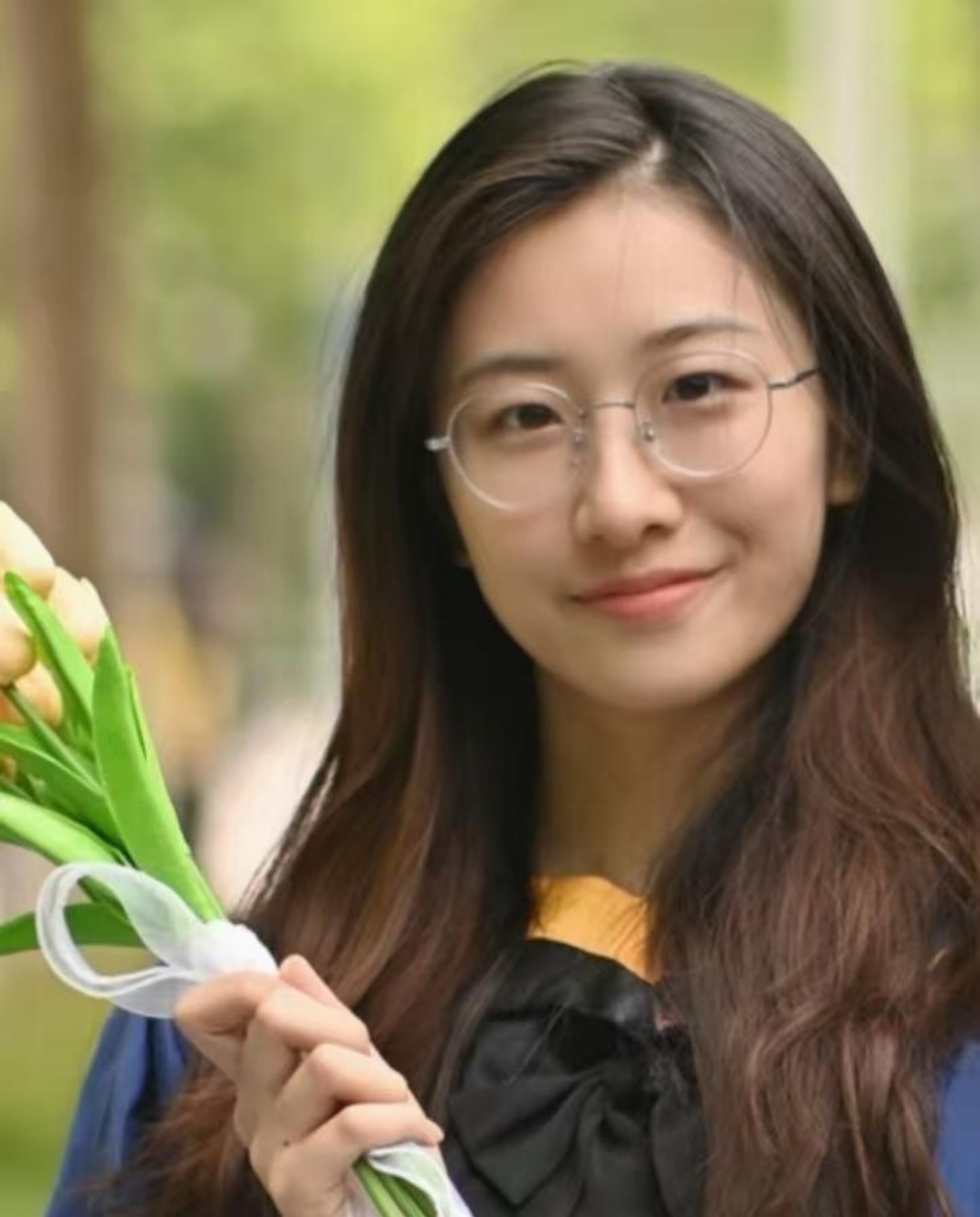}}]{Ang Li}
is currently pursuing the Ph.D. degree with the School of Cyber Science and Engineering, Wuhan University, China. She received the B.E. degree in Automation and the M.S. degree in Control Science and Engineering from the China University of Petroleum (East China), in 2022 and 2024, respectively. Her research interests include machine learning and AI security.
\end{IEEEbiography}
\begin{IEEEbiography}[{\includegraphics[width=1in,height=1.25in,clip,keepaspectratio]{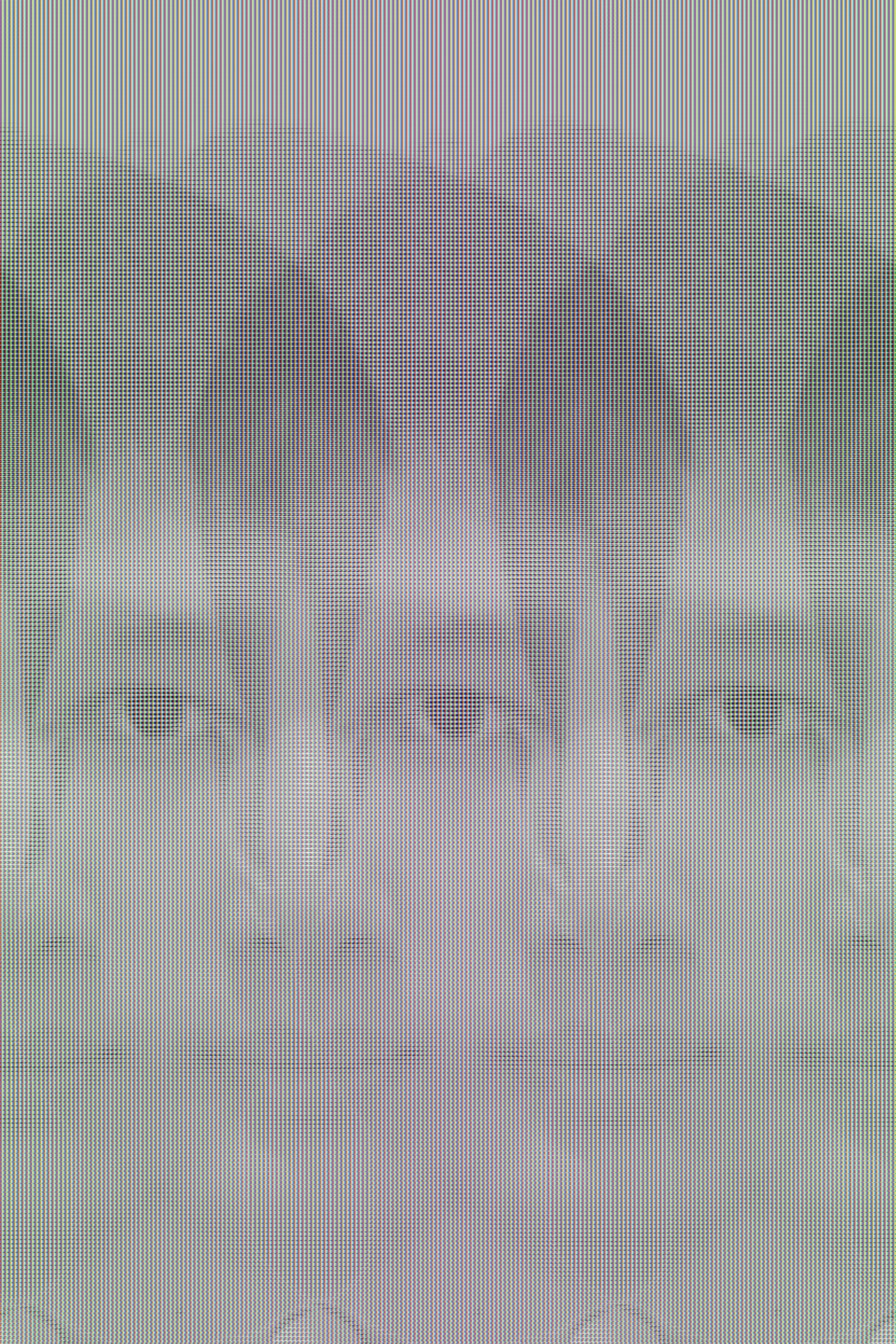}}]{Yiheng Duan}
received the B.E. degree from the School of Computer Science, Wuhan University, China, in 2021, and the M.S. degree from the School of Cyber Science and Engineering in 2024. He is currently a Staff Engineer with the Alibaba Group. He was affiliated with the Key Laboratory of Aerospace Information Security and Trusted Computing, Ministry of Education, during his graduate studies. His research interests include machine learning and AI security.
\end{IEEEbiography}
\begin{IEEEbiography}[{\includegraphics[width=1in,height=1.25in,clip,keepaspectratio]{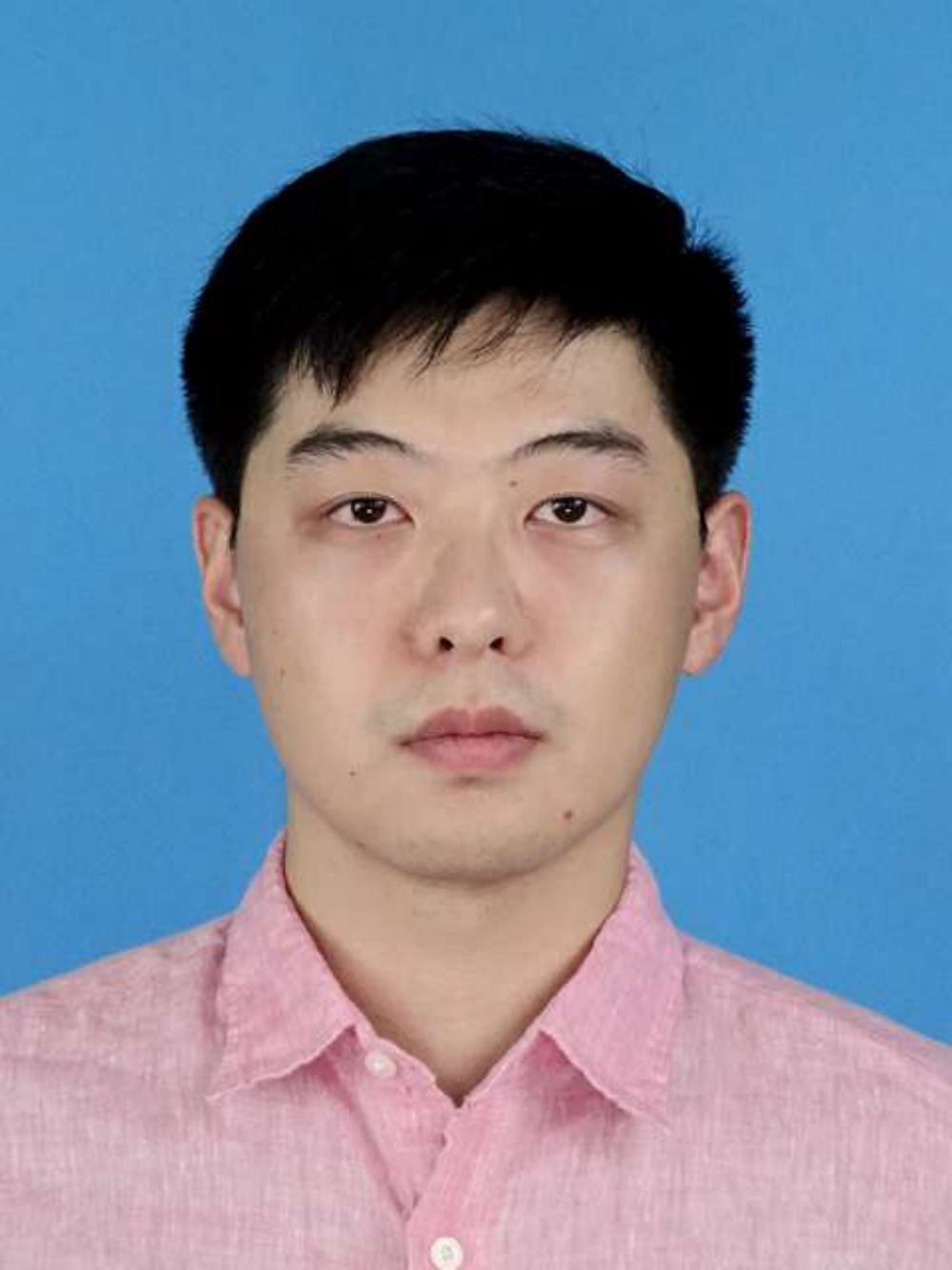}}]{Shenyi Zhang}
received the Ph.D. degree in Cyberspace Security from Wuhan University, China, in 2025. Before that, he received the B.E. degree in Communication Engineering from Shandong University in 2019 and the M.S. degree in Electronic Information from Wuhan University in 2022. His research interests include machine learning and AI security.
\end{IEEEbiography}
\begin{IEEEbiography}[{\includegraphics[width=1in,height=1.25in,clip,keepaspectratio]{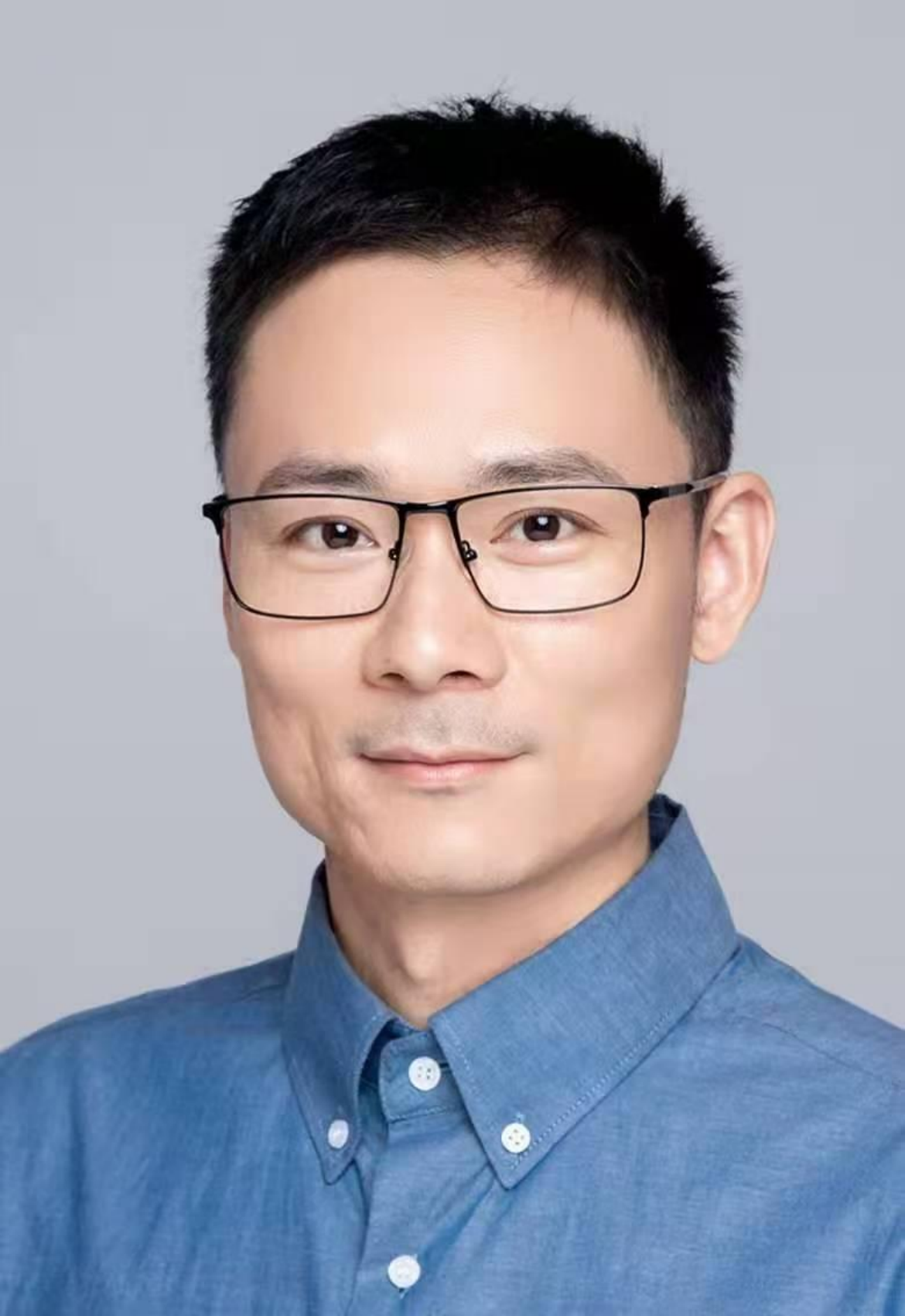}}]{Yuanjie Zhang}
is working toward the Ph.D. degree in the School of Cyber Science and Engineering, Wuhan University, China. His research interests include AI security.
\end{IEEEbiography}
\begin{IEEEbiography}[{\includegraphics[width=1in,height=1.25in,clip,keepaspectratio]{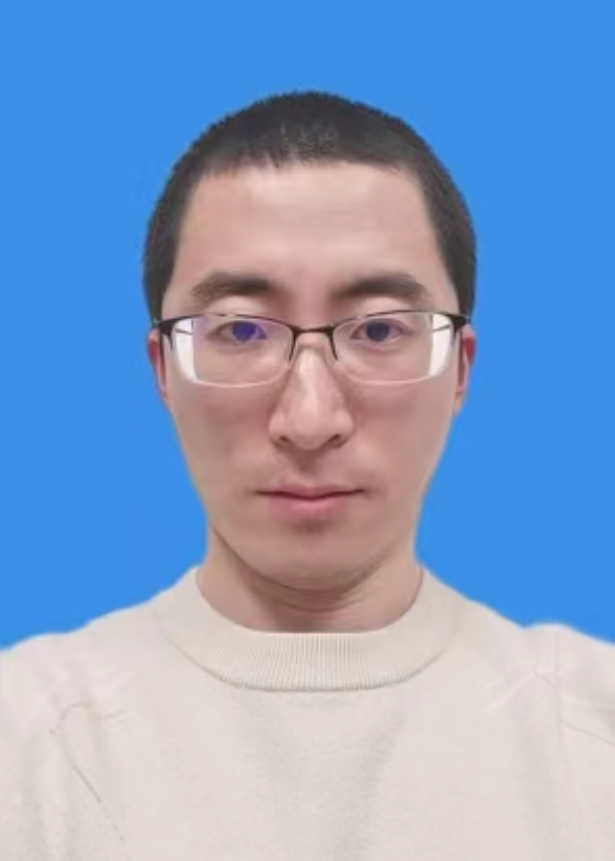}}]{Lingchen Zhao}
is currently an associate professor with the School of Cyber Science and Engineering, Wuhan University, China. He received his Ph.D. degree in Cyberspace Security in 2021, from Wuhan University, China, and his B.E. degree in Information Security in 2016, from Central South University, China. He was a Postdoctoral Researcher with the City University of Hong Kong, Hong Kong, from 2021 to 2022. His research interests include data security and AI security.
\end{IEEEbiography}
\begin{IEEEbiography}[{\includegraphics[width=1in,height=1.25in,clip,keepaspectratio]{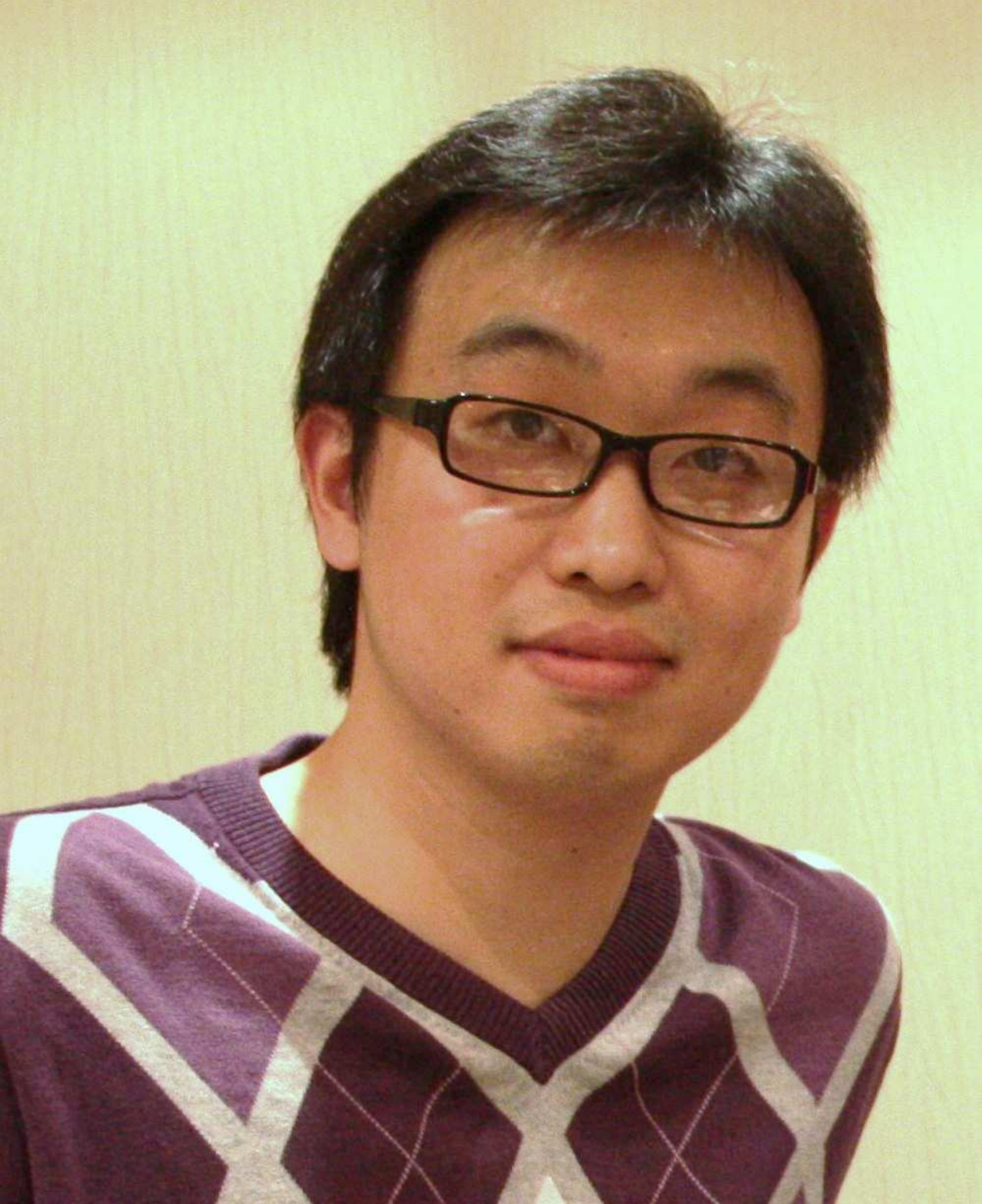}}]{Qian wang}
is a Professor in the School of Cyber Science and Engineering at Wuhan University, China. He was selected into the National Highlevel Young Talents Program of China, and listed among the World’s Top 2\% Scientists by Stanford University. He also received the National Science Fund for Excellent Young Scholars of China in 2018. He has long been engaged in the research of cyberspace security, with focus on AI security, data outsourcing security and privacy, wireless systems security, and applied cryptography. He was a recipient of the 2018 IEEE TCSC Award for Excellence in Scalable Computing (early career researcher) and the 2016 IEEE ComSoc Asia-Pacific Outstanding Young Researcher Award. He has published 200+ papers, with 120+ publications in top-tier international conferences, including USENIX NSDI, ACM CCS, USENIX Security, NDSS, ACM MobiCom, ICML, etc., with 20000+ Google Scholar citations. He is also a co-recipient of 8 Best Paper and Best Student Paper Awards from prestigious conferences, including ICDCS, IEEE ICNP, etc. In 2021, his PhD student was selected under Huawei’s “Top Minds” Recruitment Program. He serves as Associate Editors for IEEE Transactions on Dependable and Secure Computing (TDSC) and IEEE Transactions on Information Forensics and Security (TIFS). He is a fellow of the IEEE, and a member of the ACM.
\end{IEEEbiography}

\clearpage

\appendices
\appendix[Supplementary Material]

\setcounter{page}{1}
\setcounter{equation}{0}
\renewcommand{\theequation}{A-\arabic{equation}}
\setcounter{figure}{0}
\renewcommand{\thefigure}{A-\arabic{figure}}
\setcounter{table}{0}
\renewcommand{\thetable}{A-\arabic{table}}

\subsection{Theoretical Justification and Limitations of ATG}
\label{sec:atg_justificaton}
\rev{
In Section IV-A of the main paper, we introduced the \emph{Adversarial Transferability Gap (ATG)} to formally characterize the discrepancy in perturbation efficacy between models. Here, we provide the rigorous theoretical basis for this metric, analyze its reasonableness under standard settings, and discuss its empirical behavior when theoretical assumptions are violated.

\subsubsection{Formal Assumptions}
The validity of using ATG as an unbiased proxy for transferability relies on three structural assumptions regarding the surrogate model $\mathcal{F}$ and the target model $\mathcal{F'}$:
\begin{itemize}
    \item Assumption 1 (Shared Label Space): Both models map the input space $\mathcal{X}$ to the same label space $\mathcal{Y}$. This ensures that the semantics of the ground-truth label $y$ are consistent across models.
    \item Assumption 2 (Identical Loss Function): Both models employ the same loss function $J(\cdot)$ (e.g., Cross-Entropy Loss) to measure prediction error, ensuring the output values are mathematically comparable.
    \item Assumption 3 (Aligned Clean Performance): The models exhibit comparable performance on clean (benign) examples. Specifically, for a given clean input $\mathbf{x}$, the intrinsic loss difference is negligible:
    \revminor{
    \small
    \begin{equation}
        \label{eq:clean_align}
        |J(\mathbf{x}, y; \mathcal{F'}) - J(\mathbf{x}, y; \mathcal{F})| \approx 0.
    \end{equation}
    }
\end{itemize}

\subsubsection{Reasonableness and Theoretical Necessity}
Based on the assumptions formalized above, we justify why minimizing $|\operatorname{ATG}|$ is a necessary condition for maximizing transferability.

First, Assumptions 1 and 2 ensure that the outputs of the surrogate $\mathcal{F}$ and target $\mathcal{F'}$ share the same semantic space and metric scale, making their loss difference meaningful. Crucially, by invoking \textbf{Assumption 3} (aligned clean performance), we can decompose the adversarial loss on the target model as:
\revminor{
\small
\begin{align}
    \underbrace{J(\mathbf{x}+\boldsymbol{\delta}; \mathcal{F'}) - J(\mathbf{x}; \mathcal{F'})}_{\text{Target Loss Increment}} &\approx \underbrace{J(\mathbf{x}+\boldsymbol{\delta}; \mathcal{F}) - J(\mathbf{x}; \mathcal{F})}_{\text{Source Loss Increment}} + \operatorname{ATG}.
\end{align}
}
In practical transfer scenarios, model mismatch typically causes the attack effectiveness to degrade on the target, implying that the Target Loss Increment is smaller than the Source Loss Increment (i.e., $\operatorname{ATG} < 0$). Consequently, $\operatorname{ATG}$ represents the \textbf{performance drop} or \emph{transferability leakage}. Therefore, the optimization objective must be twofold: (1) maximize the source loss increment via gradient ascent; and (2) minimize the magnitude $|\operatorname{ATG}|$. Since $\operatorname{ATG}$ is inherently negative, minimizing $|\operatorname{ATG}|$ pushes the value closer to zero, thereby mitigating the performance drop and preserving the adversarial efficacy on the target. This logically validates $|\operatorname{ATG}|$ as the correct regularization term.

\subsubsection{Limitations}
We explicitly acknowledge that ATG serves as an unbiased proxy only when the models are comparable. The primary limitation arises if Assumption 3 is significantly violated, for instance, when transferring from a high-performance surrogate to a significantly weaker target. In such cases where the intrinsic clean loss gap $|J(\mathbf{x}; \mathcal{F'}) - J(\mathbf{x}; \mathcal{F})| \gg 0$, the absolute value of ATG would be biased by this performance discrepancy rather than solely reflecting perturbation transferability. However, in standard benchmarks involving mainstream classifiers on ImageNet, this condition is naturally satisfied.

\subsubsection{Empirical Generalization beyond Assumptions}
While our theoretical derivation relies on strict assumptions, our experiments demonstrate that MEF exhibits remarkable transferability even when these conditions are violated. For instance, MEF achieves high success rates when transferring to Commercial Vision APIs (violating Assumption 1 \& 3 due to unknown label spaces) and Multimodal Large Language Models (violating Assumptions 2 \& 3 due to distinct generative tasks). These results suggest that the \emph{zeroth-order average-case flatness} optimized by MEF captures intrinsic vulnerabilities in the visual feature representation space (e.g., texture bias), creating robust disruptions that transcend the specific theoretical constraints of the ATG metric itself.
}

\subsection{Proof of Theorem~\ref{thm:flatness-transferability}}
Here, we provide the detailed derivation for Theorem~\ref{thm:flatness-transferability} using a finite-order Taylor expansion, acknowledging the piecewise linearity of neural networks.

\begin{proof}[Proof of Theorem~\ref{thm:flatness-transferability}]
\label{proof:flat_transfer}
Consider the loss function \(J(\cdot)\) which we assume to be \(N\)-times differentiable in the local neighborhood. Instead of an infinite series, we apply the multivariate Taylor's theorem up to a finite order \(N\). For both the target model \(\mathcal{F'}\) and the surrogate model \(\mathcal{F}\), we have:
\revminor{
\small
\begin{align}
J(\mathbf x+\boldsymbol\delta,y;F')
&= \sum_{n=0}^N \frac{1}{n!}\,\nabla^n_{\mathbf x}J(\mathbf x,y;F')\bigl[\boldsymbol\delta^{\otimes n}\bigr] + R_N(\mathbf x, \boldsymbol\delta; F'), \label{eq:taylor_1} \\
J(\mathbf x+\boldsymbol\delta,y;F)
&= \sum_{n=0}^N \frac{1}{n!}\,\nabla^n_{\mathbf x}J(\mathbf x,y;F)\bigl[\boldsymbol\delta^{\otimes n}\bigr] + R_N(\mathbf x, \boldsymbol\delta; F). \label{eq:taylor_2}
\end{align}
}
where \(R_N(\cdot)\) represents the Lagrange remainder term of order \(N\), capturing all higher-order variations.

Subtracting the surrogate loss from the target loss yields the Adversarial Transferability Gap (ATG):
\revminor{
\small
\begin{equation}
\begin{split}
\operatorname{aATG}\bigl((\mathbf{x},y),\boldsymbol\delta;\mathcal{F},\mathcal{F'}\bigr)
&= \sum_{n=0}^N \frac{1}{n!}
   \Bigl(\nabla^n_{\mathbf x}J(\mathbf x,y;F')\\
    &-\nabla^n_{\mathbf x}J(\mathbf x,y;F)\Bigr)
   \bigl[\boldsymbol\delta^{\otimes n}\bigr] \\
   &\quad + \Bigl(R_N(\mathbf x, \boldsymbol\delta; F') - R_N(\mathbf x, \boldsymbol\delta; F)\Bigr).
\end{split}
\end{equation}
}
For each order \(0 \le n \le N\), define the derivative discrepancy term:
\revminor{
\small
\begin{align}
T_n(\mathbf{x})
  = \nabla^n_{\mathbf{x}}J(\mathbf{x},y;F')
  - \nabla^n_{\mathbf{x}}J(\mathbf{x},y;F).
\end{align}
}
Following the same logic as the zeroth-order flatness derivation, we bound \(\|T_n(\mathbf{x})\|\) by considering an arbitrary point \(\mathbf{x'}\in B_{\xi}(\mathbf{x})\), applying the triangle inequality, and taking the expectation over \(\mathbf{x'}\):
\revminor{
\small
\begin{equation}
\begin{split}
\|T_n(\mathbf{x})\| \le
& \underbrace{\mathbb{E}_{\mathbf{x'}\in B_{\xi}(\mathbf{x})} \|\nabla^nJ(\mathbf{x'},y;F)\!-\!\nabla^nJ(\mathbf{x},y;F)\|}_{\bar R^{(n)}_\xi(\mathbf{x};F)} \\
+\; & \underbrace{\mathbb{E}_{\mathbf{x'}\in B_{\xi}(\mathbf{x})} \|\nabla^nJ(\mathbf{x'},y;F')\!-\!\nabla^nJ(\mathbf{x},y;F')\|}_{\bar R^{(n)}_\xi(\mathbf{x};F')} \\
+\; & \underbrace{\mathbb{E}_{\mathbf{x'}\in B_{\xi}(\mathbf{x})} \|\nabla^nJ(\mathbf{x'},y;F')\!-\!\nabla^nJ(\mathbf{x'},y;F)\|}_{C_n(\mathbf{x})}.
\end{split}
\end{equation}
}

Let \(\Delta R_N = |R_N(\mathbf x, \boldsymbol\delta; F') - R_N(\mathbf x, \boldsymbol\delta; F)|\) denote the bound on the residual difference. Using the property \(\bigl|T_n(\mathbf{x})[\boldsymbol\delta^{\otimes n}]\bigr|\le\|T_n(\mathbf{x})\|\|\boldsymbol\delta\|^n\), we obtain the final bound:
\revminor{
\small
\begin{equation}
\begin{aligned}
\bigl|\mathrm{aATG}\bigr|
&\le \sum_{n=0}^N\frac{\|\boldsymbol\delta\|^n}{n!}
\Bigl[
  \bar R^{(n)}_\xi(\mathbf{x};F) + C_n(\mathbf{x}) + \bar R^{(n)}_\xi(\mathbf{x};F') \Bigr] \\
&\quad + \Delta R_N.
\end{aligned}
\end{equation}
}
For ReLU-based networks, typically \(N=1\) is sufficient as higher-order derivatives vanish or are undefined, making the remainder \(\Delta R_N\) the implicit error bound for high-order curvature effects.
\end{proof}

\subsection{Detailed Hyper-parameter Settings}
\label{sec:supp_params}

\rev{
To ensure a fair and rigorous comparison, we adhere to a standardized evaluation protocol while respecting the optimal configurations reported in the original papers of the baseline methods.

\textbf{Unified Experimental Settings}
We unify the following general parameters across all attacks to ensure consistency:
\begin{itemize}
    \item \textbf{Perturbation Budget:} The maximum perturbation magnitude is set to $\epsilon = 16/255$ for all methods (unless analyzing budget sensitivity).
    \item \textbf{Step Size \& Iterations:} For standard iterative gradient-based attacks, we set the number of iterations $T = 10$ and the step size $\alpha = \epsilon / T = 1.6/255$.
    \item \textbf{Sampling-based Methods:} Since the transferability of sampling-based methods typically correlates positively with the number of sampling points, we fix the sample size $N = 20$ for all such methods (e.g., VMI, EMI, MEF) to ensure fair comparison under the same computational load. This is double the typical setting ($N=10$) used in prior works, ensuring robust gradient estimation.
    \item \textbf{Data-driven Methods:} For methods involving input transformations (e.g., Admix, DIM, SIM), we set the number of image copies (or scale copies) to 5.
\end{itemize}

\textbf{Method-Specific Configurations}
For method-specific hyper-parameters, we adopt the optimal values recommended in their respective original papers.
\begin{itemize}
    \item \textbf{RAP~\cite{rap}:} As an exception to the standard iteration rule, RAP requires a multi-stage optimization process. We follow its official setting with total iterations $K = 400$, late-start phase $K_{LS} = 100$, and a fixed step size $\alpha = 1.6/255$.
    \item \textbf{MEF (Ours):} For our proposed method, we set the neighborhood radius $\gamma = 2 \cdot \epsilon$, the exploration radius $\xi = 0.15 \cdot \epsilon$, and the outer/inner momentum coefficients $\mu_{outer} = 0.5$, $\mu_{inner} = 0.9$.
\end{itemize}
}

\begin{table*}[tbp!]
\caption{\rev{Supplemental comparison with additional baselines. This table presents the transfer attack success rate ($\%$) of classical or less competitive gradient-stabilized methods. Even compared to these diverse baselines, our MEF framework maintains consistent superiority.}}
\vspace{-10pt}
\label{tab:supp_baselines}
\begin{center}
\begin{small}
\setlength{\tabcolsep}{4pt}
\setlength{\extrarowheight}{0.2pt}
\scalebox{0.75}{
\begin{tabular}{c|cccccccc|cccccccc}
\hline
\multirow{2}{*}{Attack} & \multicolumn{8}{c|}{\textbf{Res-50} $\Longrightarrow$} & \multicolumn{8}{c}{\textbf{Res-101} $\Longrightarrow$}  \\
 & Res-101 & Inc-v3 & Inc-v4 & IncRes-v2 & VGG-19 & Dense-121 & Xcept & \rev{AVG} & Res-50 & Inc-v3 & Inc-v4 & IncRes-v2 & VGG-19 & Dense-121& Xcept  & \rev{AVG} \\
\hline
MI~\cite{mi}
& 93.3 & 50.5 & 46.0 & 31.8 & 79.1 & 86.1 & 53.3 & \rev{62.9} & 94.4 & 50.6 & 43.9 & 33.7 & 73.4 & 82.2 & 51.8 & \rev{61.4} \\

NI~\cite{ni}
& 96.7 & 56.5 & 49.1 & 35.9 & 84.0 & 88.5 & 56.4 & \rev{66.7} & 97.8 & 55.2 & 50.9 & 37.8 & 79.0 & 85.8 & 55.0 & \rev{65.9} \\

PI~\cite{pi-emi}
& 98.0 & 60.2 & 54.4 & 38.1 & 88.1 & 92.8 & 60.7 & \rev{70.3} & 98.5 & 62.4 & 55.9 & 42.2 & 83.4 & 90.7 & 60.2 & \rev{70.5} \\

TPA~\cite{tpa}
& 96.2 & 60.5 & 55.0 & 40.4 & 85.9 & 90.5 & 60.7 & \rev{69.9} & 97.2 & 58.1 & 54.2 & 43.2 & 81.4 & 88.8 & 60.9 & \rev{69.1} \\

GNP~\cite{gnp}
& 98.7 & 69.8 & 64.7 & 50.6 & 91.7 & 95.7 & 68.2 & \rev{77.1} & 99.1 & 70.3 & 64.5 & 52.0 & 87.9 & 93.9 & 68.5 & \rev{76.6} \\

VMI~\cite{vmi-vni}
& 98.2 & 72.8 & 68.6 & 57.4 & 91.3 & 95.7 & 73.2 & \rev{79.6} & 98.6 & 73.2 & 68.4 & 57.6 & 88.5 & 93.6 & 70.5 & \rev{78.6} \\

VNI~\cite{vmi-vni}
& 98.8 & 77.2 & 73.4 & 61.6 & 93.7 & 96.8 & 75.7 & \rev{82.5} & 99.4 & 77.4 & 73.1 & 62.8 & 91.3 & 95.4 & 75.0 & \rev{82.1} \\

RAP~\cite{rap}
& 97.8 & 70.2 & 65.9 & 50.6 & 91.3 & 94.3 & 70.9 & \rev{77.3} & 98.7 & 70.1 & 65.1 & 52.1 & 89.3 & 93.3 & 70.2 & \rev{77.0} \\

MEF\textsubscript{H}
& \underline{99.7} & \underline{90.6} & \underline{88.7} & \underline{82.7} & \underline{98.3} & \underline{99.4} & \underline{90.4} & \rev{\underline{92.8}} & \underline{99.7} & \underline{91.2} & \underline{88.2} & \underline{82.2} & \underline{97.0} & \underline{98.4} & \underline{90.0} & \rev{\underline{92.4}} \\

MEF\textsubscript{F}
& \textbf{99.9} & \textbf{95.1} & \textbf{94.2} & \textbf{91.3} & \textbf{99.2} & \textbf{99.8} & \textbf{94.7} & \rev{\textbf{96.3}} & \textbf{99.9} & \textbf{95.8} & \textbf{95.2} & \textbf{91.8} & \textbf{98.9} & \textbf{99.4} & \textbf{95.7} & \rev{\textbf{96.7}} \\

\hline
\multirow{2}{*}{Attack} & \multicolumn{8}{c|}{\textbf{Inc-v3} $\Longrightarrow$} & \multicolumn{8}{c}{\textbf{Inc-v4} $\Longrightarrow$}  \\
 & Res-50 & Res-101 & Inc-v4 & IncRes-v2 & VGG-19 & Dense-121 & Xcept& \rev{AVG} & Res-50 & Res-101& Inc-v3 & IncRes-v2 & VGG-19 & Dense-121 & Xcept  & \rev{AVG} \\
\hline
MI~\cite{mi}
& 54.6 & 48.8 & 51.2 & 43.3 & 56.0 & 55.7 & 56.1 & \rev{52.2} & 56.1 & 51.3 & 58.2 & 43.3 & 61.7 & 58.3 & 56.8 & \rev{55.1} \\

NI~\cite{ni}
& 63.8 & 58.0 & 61.3 & 53.3 & 63.7 & 64.3 & 62.5 & \rev{61.0} & 62.3 & 56.7 & 65.5 & 49.7 & 68.9 & 64.2 & 64.2 & \rev{61.6} \\

PI~\cite{pi-emi}
& 66.1 & 60.7 & 64.3 & 56.8 & 65.6 & 67.7 & 64.9 & \rev{63.7} & 64.3 & 59.9 & 69.8 & 54.0 & 71.9 & 68.3 & 66.6 & \rev{65.0} \\

TPA~\cite{tpa}
& 59.4 & 54.2 & 60.4 & 54.2 & 64.3 & 60.9 & 64.7 & \rev{59.7} & 61.6 & 53.3 & 66.7 & 56.1 & 70.7 & 64.2 & 72.0 & \rev{63.5} \\

GNP~\cite{gnp}
& 60.5 & 55.1 & 65.9 & 56.0 & 64.7 & 61.5 & 66.3 & \rev{61.4} & 58.3 & 51.9 & 64.6 & 56.6 & 66.4 & 60.0 & 71.2 & \rev{61.3} \\

VMI~\cite{vmi-vni}
& 69.6 & 65.4 & 71.3 & 64.7 & 70.9 & 72.0 & 71.8 & \rev{69.4} & 70.7 & 66.6 & 77.6 & 66.8 & 76.3 & 74.6 & 73.0 & \rev{72.2} \\

VNI~\cite{vmi-vni}
& 76.1 & 71.2 & 76.9 & 71.3 & 75.7 & 77.7 & 76.4 & \rev{75.0} & 75.3 & 71.4 & 82.1 & 71.0 & 80.8 & 77.9 & 77.4 & \rev{76.6} \\

RAP~\cite{rap}
& 81.3 & 74.2 & 75.4 & 67.7 & 80.7 & 79.5 & 78.1 & \rev{76.7} & 81.9 & 76.9 & 82.1 & 65.5 & 87.9 & 82.4 & 82.3 & \rev{79.9} \\

MEF\textsubscript{H}
& \underline{85.8} & \underline{80.8} & \underline{89.4} & \underline{86.1} & \underline{85.2} & \underline{87.4} & \underline{89.9} & \rev{\underline{86.4}} & \underline{86.2} & \underline{82.4} & \underline{92.8} & \underline{86.1} & \underline{90.1} & \underline{88.5} & \underline{90.8} & \rev{\underline{88.1}} \\

MEF\textsubscript{F}
& \textbf{92.6} & \textbf{91.5} & \textbf{96.5} & \textbf{94.3} & \textbf{92.8} & \textbf{93.4} & \textbf{95.8} & \rev{\textbf{93.8}} & \textbf{91.4} & \textbf{90.3} & \textbf{94.5} & \textbf{91.2} & \textbf{92.2} & \textbf{91.9} & \textbf{95.1} & \rev{\textbf{92.4}} \\
\hline
\end{tabular}
}
\end{small}
\end{center}
\vspace{-10pt}
\end{table*}
\subsection{Extended Analysis of Gradient-based Baselines}
\label{sec:exp_baseline}
\revminor{Due to space constraints in the main manuscript, we present the comparison results against classical or lower-performing gradient stabilization methods in this section. Table~\ref{tab:supp_baselines} reports the transfer attack success rates of 8 additional baselines, including MI~\cite{mi}, NI~\cite{ni}, PI~\cite{pi-emi}, TPA~\cite{tpa}, GNP~\cite{gnp}, VMI~\cite{vmi-vni}, VNI~\cite{vmi-vni}, and RAP~\cite{rap}. Although these methods have established foundations in the field, their average transferability falls below the 80\% threshold. For a comprehensive comparison, we also include the results of MEF\textsubscript{H} and MEF\textsubscript{F} in Table~\ref{tab:supp_baselines}. The results confirm that our MEF framework maintains a significant lead (over 10\% improvement in most cases) even when compared across this broader range of diverse baselines.}

\begin{table*}[ht]
\caption{\revminor{Transfer attack success rate (\%) of MEF integrated with additional input augmentation methods (DI, SI, Admix, SSA). Comparing the rows (e.g., DI vs. MEF+DI) reveals that our framework consistently amplifies the attack transferability of these classical augmentation techniques.}}
\label{tab:supp_ia}
\vspace{-10pt}
\begin{center}
\begin{small}
\setlength{\extrarowheight}{0.2pt}
\setlength{\tabcolsep}{4pt}
\scalebox{0.7}{
\begin{tabular}{c|ccccccc|ccccccc}
\hline
\multirow{2}{*}{Attack} & \multicolumn{7}{c|}{{Res-50} $\Longrightarrow$} & \multicolumn{7}{c}{{Inc-v3} $\Longrightarrow$} \\ & IncRes-v2 & Inc-v3$_{ens3}$ & Inc-v4$_{ens4}$ & IncRes-v2$_{ens}$ & ViT-L/32 & MLP-Mixer & \rev{{AVG}} & IncRes-v2 & Inc-v3$_{ens3}$ & Inc-v4$_{ens4}$ & IncRes-v2$_{ens}$ & ViT-L/32 & MLP-Mixer & \rev{{AVG}}
\\
\hline
DI~\cite{di}
& 70.2 & 37.5 & 36.5 & 23.1 & 29.5 & 59.2 & \rev{42.7}& 66.7 & 37.1 & 35.7 & 18.8 & 26.3 & 54.3 & \rev{39.8}\\
MEF\textsubscript{H}+DI
& {89.7} & {77.5} & {76.1} & {66.1} & {50.3} & {74.5} & \rev{72.4}& {88.3} & {74.8} & {74.7} & {56.2} & {41.6} & {67.8} & \rev{67.2}\\
MEF\textsubscript{F}+DI
& {95.3} & {83.8} & {80.5} & {68.6} & {55.1} & {79.8} & \rev{77.2}& {94.2} & {78.0} & {77.6} & {58.0} & {46.4} & {74.1} & \rev{71.4}\\
\hdashline
SI~\cite{ni}
& 56.8 & 29.6 & 29.5 & 17.4 & 24.9 & 49.5 & \rev{34.6}& 65.2 & 34.7 & 35.3 & 18.3 & 25.1 & 53.7 & \rev{38.7}\\
MEF\textsubscript{H}+SI
& {86.9} & {71.3} & {69.7} & {56.8} & {44.1} & {67.1} & \rev{66.0}& {89.4} & {75.7} & {75.0} & {55.2} & {39.4} & {64.2} & \rev{66.5}\\
MEF\textsubscript{F}+SI
& {92.8} & {78.6} & {74.7} & {61.2} & {48.8} & {73.0} & \rev{71.5}& {95.9} & {81.7} & {80.5} & {57.8} & {44.9} & {72.2} & \rev{72.2}\\
\hdashline
Admix~\cite{admix}
& 52.5 & 23.6 & 22.2 & 13.1 & 23.9 & 50.4 & \rev{30.9}& 64.2 & 29.7 & 27.7 & 14.7 & 23.6 & 52.8 & \rev{35.4}\\
MEF\textsubscript{H}+Admix
& {70.5} & {39.2} & {37.0} & {24.5} & {27.0} & {55.5} & \rev{42.3}& {83.5} & {45.3} & {45.8} & {27.1} & {26.5} & {56.5} & \rev{47.5}\\
MEF\textsubscript{F}+Admix
& {86.4} & {46.6} & {42.1} & {27.7} & {31.6} & {64.1} & \rev{49.8}& {94.7} & {51.1} & {51.9} & {30.1} & {31.0} & {67.6} & \rev{54.4}\\
\hdashline
SSA~\cite{ssa}
& {83.0} & 50.2 & 47.3 & 31.3 & {35.6} & {64.7} & \rev{52.0}& 83.7 & 52.4 & 54.7 & 31.0 & {33.3} & {62.3} & \rev{52.9}\\
MEF\textsubscript{H}+SSA
& 81.6 & {53.9} & {51.1} & {38.0} & 33.0 & 61.1 & \rev{53.1}& {84.4} & {54.9} & {56.3} & {34.2} & 30.9 & 58.2 & \rev{53.1}\\
MEF\textsubscript{F}+SSA
& {90.1} & {58.1} & {53.5} & {39.0} & {35.7} & {68.0} & \rev{57.4}& {92.2} & {58.2} & {60.3} & {35.1} & {35.3} & {65.7} & \rev{57.8}\\
\hline
\end{tabular}
}
\end{small}
\end{center}
\vspace{-15pt}
\end{table*}
\subsection{Extended Analysis of Input Augmentation Integration}
\label{sec:exp_supp_ia}
\revminor{In this section, we provide the detailed experimental results of integrating MEF with four additional input augmentation methods: DI~\cite{di}, SI~\cite{ni}, Admix~\cite{admix}, and SSA~\cite{ssa}. Table~\ref{tab:supp_ia} shows the transfer attack success rates on six target models. The results consistently demonstrate that MEF is highly compatible with diverse input transformation strategies, significantly boosting the baseline performance across all tested scenarios.}

\revminor{
\begin{table*}[t]
\caption{Transfer attack success rate ($\%$) of classical baseline methods against defense mechanisms. While these methods struggle to breach robust defenses (average ASR below 40\%), MEF variants maintain superior performance.}
\label{tab:supp_robust}
\vspace{-10pt}
\begin{center}
\begin{small}
\setlength{\extrarowheight}{0.25pt}
\scalebox{0.68}{
\begin{tabular}{c|cccccccc|cccccccc}
\hline
\multirow{2}{*}{Attack} & \multicolumn{8}{c|}{\textbf{Res-50} $\Longrightarrow$} & \multicolumn{8}{c}{\textbf{Inc-v3} $\Longrightarrow$} \\ & Inc-v3$_{adv}$ & Inc-v3$_{ens3}$ & Inc-v3$_{ens4}$ & IncRes-v2$_{ens}$ & RS & HGD & NRP & \rev{\textbf{AVG}} & Inc-v3$_{adv}$ & Inc-v3$_{ens3}$ & Inc-v3$_{ens4}$ & IncRes-v2$_{ens}$ & RS & HGD & NRP & \rev{\textbf{AVG}} \\
\hline
MI~\cite{mi} & 21.6 & 17.2 & 16.4 & 9.4 & 25.6 & 19.4 & 53.5 & \rev{23.3} & 27.3 & 22.4 & 23.3 & 11.2 & 24.5 & 8.9 & 30.7 & \rev{21.2} \\
NI~\cite{ni} & 21.8 & 18.3 & 17.3 & 9.2 & 25.7 & 17.0 & 55.4 & \rev{23.5} & 27.8 & 22.8 & 22.6 & 11.1 & 24.4 & 8.4 & 30.5 & \rev{21.1} \\
PI~\cite{pi-emi} & 22.7 & 19.1 & 17.5 & 9.8 & 26.6 & 21.6 & 60.7 & \rev{25.4} & 30.9 & 24.5 & 24.3 & 12.6 & 24.7 & 9.8 & 30.9 & \rev{22.5} \\
TPA~\cite{tpa} & 25.7 & 18.1 & 15.5 & 9.3 & 28.4 & 6.7 & 58.2 & \rev{23.1} & 27.8 & 18.0 & 16.2 & 8.6 & 25.8 & 2.3 & 33.2 & \rev{18.8} \\
GNP~\cite{gnp} & 29.2 & 24.8 & 23.2 & 13.3 & 28.1 & 26.5 & 70.7 & \rev{30.8} & 30.5 & 24.0 & 23.5 & 11.7 & 25.1 & 6.5 & 34.4 & \rev{22.2} \\
VMI~\cite{vmi-vni} & 36.5 & 35.4 & 34.3 & 21.8 & 30.6 & 45.8 & 73.7 & \rev{39.7} & 45.3 & 41.2 & 41.6 & 25.1 & 27.3 & 24.5 & 38.5 & \rev{34.8} \\
VNI~\cite{vmi-vni} & 37.8 & 35.5 & 34.7 & 23.8 & 30.9 & 47.4 & 75.9 & \rev{40.9} & 47.8 & 44.4 & 43.9 & 26.8 & 28.0 & 26.7 & 38.9 & \rev{36.6} \\
EMI~\cite{pi-emi} & 29.5 & 23.6 & 22.4 & 12.5 & 29.0 & 28.9 & 71.4 & \rev{31.0} & 39.8 & 33.5 & 31.4 & 17.7 & 27.2 & 13.5 & 38.3 & \rev{28.8} \\
RAP~\cite{rap} & 29.4 & 17.4 & 16.9 & 8.8 & 32.1 & 4.0 & 63.2 & \rev{24.5} & 32.2 & 16.3 & 16.2 & 7.9 & 27.3 & 1.5 & 27.6 & \rev{18.4} \\
MEF\textsubscript{H} & \underline{62.4} & \underline{64.5} & \underline{63.3} & \underline{50.4} & \underline{49.2} & \underline{80.6} & \underline{94.6} & \rev{66.4} & \underline{69.7} & \underline{65.1} & \underline{65.4} & \underline{45.4} & \underline{38.3} & \underline{46.3} & \underline{54.8} & \rev{55.0} \\
MEF\textsubscript{F} & \textbf{70.1} & \textbf{69.2} & \textbf{65.1} & \textbf{52.2} & \textbf{51.5} & \textbf{83.0} & \textbf{97.2} & \rev{69.8} & \textbf{75.8} & \textbf{69.0} & \textbf{69.6} & \textbf{47.9} & \textbf{40.4} & \textbf{52.4} & \textbf{56.2} & \rev{58.8} \\
\hline
\end{tabular}
}
\end{small}
\end{center}
\vspace{-15pt}
\end{table*}
}
\subsection{Extended Robustness Evaluation}
\label{sec:exp_supp_robust}
\revminor{In this section, we provide the evaluation results of classical gradient stabilization attacks against defense mechanisms. Table~\ref{tab:supp_robust} reports the attack success rates of 9 baseline methods (including MI~\cite{mi}, NI~\cite{ni}, PI~\cite{pi-emi}, TPA~\cite{tpa}, GNP~\cite{gnp}, VMI~\cite{vmi-vni}, VNI~\cite{vmi-vni}, EMI~\cite{pi-emi}, and RAP~\cite{rap}) on seven robustly defended models. These methods generally exhibit lower transferability (average ASR $<$ 40\%) compared to the state-of-the-art methods presented in the main text.}

\begin{table*}[htbp!]
\caption{Transfer attack success rate ($\%$) comparison of gradient-stabilized attacks on cross-architecture models. Evaluating MEF\textsubscript{H}/\textsubscript{F} against 12 gradient stabilization methods on seven diverse architectures (CNNs, Transformers, MLPs). The best results are \textbf{bold} and the second best are \underline{underlined}.}
\vspace{-15pt}
\label{tab:eval_on_diverse_models}
\begin{center}
\begin{small}
\setlength{\extrarowheight}{0.15em}
\scalebox{0.72}{
\begin{tabular}{c|ccccccc|ccccccc}
\hline
\multirow{2}{*}{Attack} & \multicolumn{7}{c|}{\textbf{Res-50} $\Longrightarrow$} & \multicolumn{7}{c}{\textbf{Inc-v3} $\Longrightarrow$} \\ & MobileNet & PNASNet-L & ViT-B/16 & ViT-L/32 & PiT-S & MLP-Mixer & ResMLP & MobileNet & PNASNet-L & ViT-B/16 & ViT-L/32 & PiT-S & MLP-Mixer & ResMLP \\
\hline
MI~\cite{mi}
& 68.4
& 51.0
& 20.2
& 20.9
& 25.1
& 42.6
& 32.4
& 61.7
& 48.4
& 22.5
& 21.0
& 27.1
& 47.5
& 32.8
\\
NI~\cite{ni}
& 71.5
& 52.4
& 20.1
& 20.4
& 25.7
& 44.4
& 31.1
& 68.3
& 56.5
& 22.6
& 20.8
& 29.8
& 49.3
& 35.8
\\
PI~\cite{pi-emi}
& 75.4
& 56.9
& 21.4
& 20.9
& 27.9
& 45.7
& 34.2
& 69.9
& 57.9
& 23.5
& 21.8
& 32.0
& 49.9
& 37.7
\\
TPA~\cite{tpa}
& 74.9
& 58.5
& 23.6
& 23.8
& 30.7
& 47.6
& 36.6
& 63.7
& 57.1
& 23.6
& 21.2
& 28.9
& 50.7
& 37.0
\\
GNP~\cite{gnp}
& 81.9
& 67.2
& 28.6
& 24.6
& 36.6
& 51.0
& 44.8
& 65.1
& 57.1
& 25.4
& 21.4
& 31.0
& 50.9
& 37.3
\\
VMI~\cite{vmi-vni}
& 84.3
& 70.2
& 33.8
& 28.1
& 45.4
& 51.3
& 51.9
& 71.8
& 65.0
& 35.8
& 27.7
& 45.0
& 54.1
& 49.1
\\
VNI~\cite{vmi-vni}
& 87.1
& 73.6
& 35.6
& 29.2
& 46.8
& 53.9
& 53.9
& 78.3
& 67.8
& 37.1
& 27.6
& 47.5
& 56.4
& 50.2
\\
EMI~\cite{pi-emi}
& 88.2
& 70.8
& 27.6
& 24.7
& 37.4
& 50.7
& 45.3
& 81.5
& 69.5
& 32.2
& 24.7
& 41.5
& 56.6
& 48.0
\\
RAP~\cite{rap}
& 85.9
& 66.4
& 26.4
& 25.8
& 36.3
& 52.8
& 44.1
& 82.7
& 67.3
& 24.5
& 23.8
& 36.7
& 54.5
& 43.7
\\
APP~\cite{app}
& 91.0
& 85.1
& 45.6
& 38.0
& 55.4
& 60.6
& 64.1
& 76.1
& 70.4
& 38.5
& 32.2
& 47.5
& 57.2
& 53.5
\\
FEM~\cite{femi}
& 94.0
& 87.0
& 43.1
& 35.0
& 58.1
& 60.2
& 66.0
& 81.1
& 78.2
& 41.3
& 30.4
& 53.2
& 59.7
& 58.6
\\
PGN~\cite{pgn}
& 91.6
& 88.6
& 51.4
& 45.8
& 58.7
& 66.7
& 70.2
& 84.9
& 82.6
& 50.1
& 40.6
& 60.2
& 66.1
& 65.6
\\
\rev{ANDA~\cite{anda}} & \rev{87.0} & \rev{83.5} & \rev{49.2} & \rev{44.8} & \rev{55.2} & \rev{62.8} & \rev{67.9} & \rev{80.9} & \rev{77.5} & \rev{47.9} & \rev{37.9} & \rev{56.4} & \rev{62.7} & \rev{60.0} \\
\rev{FGSRA~\cite{fgsra}} & \rev{88.8} & \rev{87.0} & \rev{49.8} & \rev{45.6} & \rev{57.1} & \rev{64.5} & \rev{70.9} & \rev{83.7} & \rev{81.7} & \rev{48.1} & \rev{39.6} & \rev{59.6} & \rev{64.2} & \rev{65.2} \\
\rev{GI~\cite{gifgsm}} & \rev{86.9} & \rev{84.3} & \rev{48.5} & \rev{45.6} & \rev{56.0} & \rev{62.6} & \rev{67.9} & \rev{79.8} & \rev{79.1} & \rev{45.9} & \rev{37.4} & \rev{57.7} & \rev{63.9} & \rev{62.8} \\
\rev{MUMODIG~\cite{mumodig}} & \rev{89.5} & \rev{86.3} & \rev{48.8} & \rev{44.1} & \rev{56.9} & \rev{66.3} & \rev{69.0} & \rev{81.3} & \rev{81.1} & \rev{48.6} & \rev{39.0} & \rev{59.5} & \rev{65.7} & \rev{65.1} \\
\rev{GAA~\cite{gaa}} & \rev{86.7} & \rev{84.4} & \rev{49.5} & \rev{44.8} & \rev{55.6} & \rev{63.6} & \rev{66.0} & \rev{80.0} & \rev{79.8} & \rev{49.3} & \rev{38.7} & \rev{58.6} & \rev{63.6} & \rev{62.1} \\
\rev{FoolMix~\cite{foolmix}} & \rev{85.7} & \rev{84.1} & \rev{47.9} & \rev{44.2} & \rev{52.1} & \rev{63.0} & \rev{65.5} & \rev{78.8} & \rev{77.0} & \rev{44.7} & \rev{37.6} & \rev{56.4} & \rev{63.1} & \rev{60.6}
\\
MEF\textsubscript{H}
& \underline{95.2}
& \underline{90.3}
& \underline{51.9}
& \underline{46.0}
& \underline{63.3}
& \underline{66.9}
& \underline{70.5}
& \underline{85.2}
& \underline{83.9}
& \underline{51.7}
& \underline{42.2}
& \underline{60.4}
& \underline{67.3}
& \underline{66.3}
\\
MEF\textsubscript{F}
& \textbf{98.0}
& \textbf{95.2}
& \textbf{56.0}
& \textbf{46.9}
& \textbf{71.1}
& \textbf{71.7}
& \textbf{77.4}
& \textbf{92.7}
& \textbf{90.8}
& \textbf{56.2}
& \textbf{45.3}
& \textbf{71.4}
& \textbf{69.4}
& \textbf{73.3}
\\
\hline
\end{tabular}
}
\end{small}
\end{center}
\vspace{-15pt}
\end{table*}
\subsection{Extended Analysis of Diverse Architectures Evaluation}
\label{sec:exp_archi_supp}
\revminor{This section provides the comprehensive evaluation results on diverse network architectures, complementing Section~\ref{sec:exp_archi_main}. Table~\ref{tab:eval_on_diverse_models} reports the full comparison against 12 gradient stabilization baselines. Consistent with the main text, MEF variants demonstrate superior generalization across all architectures.}

\rev{However, we observe a noticeable performance drop when transferring from CNNs to Transformers (e.g., MEF\textsubscript{F} averages $\sim$44\% on ViTs vs. $>$90\% on CNNs). Theoretically, this is attributed to the large Cross-model Gradient Discrepancy caused by the fundamental inductive bias difference between local convolutions and global self-attention. While MEF significantly narrows this gap compared to baselines, the structural discrepancy remains a challenge. As shown in Table III (Main Text), this can be further mitigated by combining MEF with TI and increasing sampling density (MEF+TI $30\times30$), which achieves 80\% success rate on ViT-L/32.}

\begin{table*}[htbp!]
\caption{Transfer attack success rate (\%) comparison on ensemble surrogates. Using logits-averaged Res-50 and Inc-v3 as surrogate models, our MEF\textsubscript{H}/\textsubscript{F} variants outperform 12 gradient-stabilized attacks and two ensemble-specialized methods (SVRE~\cite{svre}, CWA~\cite{cwa}). The best results are \textbf{bold} and the second best are \underline{underlined}.}
\vspace{-15pt}
\label{tab:attack_an_ensemble_model}
\begin{center}
\begin{small}
\setlength{\extrarowheight}{0.15em}
\scalebox{0.7}{
\begin{tabular}{c|cccccccccccccc}
\hline
\multirow{2}{*}{Attack} & \multicolumn{14}{c}{\textbf{Res-50 + Inc-v3} $\Longrightarrow$} \\ & Res-101 & Inc-v4 & IncRes-v2 & Inc-v3$_{adv}$ & Inc-v3$_{ens3}$ & Inc-v3$_{ens4}$ & IncRes-v2$_{ens}$ & MobileNet & PNASNet-L & ViT-B/16 & ViT-L/32 & PiT-S & MLP-Mixer & ResMLP \\
\hline
MI~\cite{mi}
& 94.1
& 68.3
& 56.1
& 33.6
& 30.4
& 30.7
& 16.9
& 76.6
& 63.4
& 32.7
& 25.4
& 39.1
& 51.6
& 47.9
\\
NI~\cite{ni}
& 96.8
& 71.4
& 61.8
& 36.4
& 33.2
& 30.4
& 16.2
& 80.6
& 66.6
& 32.7
& 24.7
& 42.6
& 54.2
& 49.7
\\
PI~\cite{pi-emi}
& 98.0
& 78.2
& 68.8
& 38.4
& 34.7
& 34.5
& 19.2
& 83.5
& 72.9
& 35.1
& 25.8
& 47.7
& 54.6
& 54.4
\\
TPA~\cite{tpa}
& 90.3
& 63.6
& 54.1
& 27.8
& 20.8
& 20.1
& 12.4
& 72.1
& 60.3
& 29.6
& 25.1
& 39.1
& 50.5
& 43.7
\\
GNP~\cite{gnp}
& 92.8
& 73.4
& 66.9
& 37.1
& 32.1
& 30.8
& 17.1
& 76.9
& 66.5
& 32.5
& 25.7
& 45.9
& 52.8
& 51.0
\\
VMI~\cite{vmi-vni}
& 97.6
& 83.4
& 77.3
& 54.8
& 55.3
& 54.6
& 37.1
& 87.4
& 80.7
& 49.0
& 35.1
& 62.1
& 62.0
& 66.2
\\
VNI~\cite{vmi-vni}
& 98.9
& 86.5
& 80.3
& 58.5
& 57.2
& 56.4
& 37.8
& 90.1
& 82.5
& 50.2
& 35.7
& 65.2
& 62.3
& 68.4
\\
EMI~\cite{pi-emi}
& \underline{99.1}
& 90.9
& 83.2
& 51.9
& 46.6
& 44.6
& 25.7
& 93.1
& 85.9
& 45.1
& 32.8
& 62.8
& 61.8
& 69.3
\\
RAP~\cite{rap}
& 98.4
& 86.1
& 80.5
& 42.3
& 28.0
& 24.9
& 14.0
& 92.4
& 80.8
& 42.3
& 31.1
& 58.1
& 63.7
& 65.8
\\
APP~\cite{app}
& 97.8
& 88.3
& 84.6
& 64.2
& 63.1
& 63.3
& 47.1
& 88.9
& 84.5
& 54.2
& 42.6
& 67.1
& 65.2
& 73.8
\\
FEM~\cite{femi}
& 98.7
& 92.0
& 88.3
& 68.3
& 65.9
& 62.9
& 46.9
& 92.5
& 89.0
& 59.0
& 41.3
& 74.1
& 66.0
& 77.0
\\
SVRE~\cite{svre}
& 92.0
& 69.2
& 58.8
& 31.0
& 26.5
& 26.6
& 13.2
& 77.8
& 64.0
& 29.0
& 23.5
& 41.1
& 50.7
& 47.4
\\
CWA~\cite{cwa}
& 94.9
& 53.8
& 45.3
& 26.3
& 16.3
& 14.5
& 8.0
& 74.6
& 56.5
& 20.8
& 20.9
& 26.3
& 49.0
& 34.2
\\
PGN~\cite{pgn}
& 97.8
& 91.9
& 89.4
& 79.5
& 78.2
& 77.5
& 65.6
& 93.9
& 91.9
& 68.7
& 50.2
& 76.9
& 74.5
& 81.8
\\
\rev{ANDA~\cite{anda}} & \rev{92.9} & \rev{86.7} & \rev{85.1} & \rev{76.6} & \rev{73.6} & \rev{73.0} & \rev{63.5} & \rev{89.4} & \rev{86.3} & \rev{65.4} & \rev{46.9} & \rev{72.2} & \rev{70.6} & \rev{75.3} \\
\rev{FGSRA~\cite{fgsra}} & \rev{94.9} & \rev{90.3} & \rev{87.4} & \rev{78.9} & \rev{76.4} & \rev{75.2} & \rev{66.3} & \rev{92.6} & \rev{90.9} & \rev{66.5} & \rev{49.1} & \rev{76.1} & \rev{72.5} & \rev{81.2} \\
\rev{GI~\cite{gifgsm}} & \rev{92.8} & \rev{87.5} & \rev{84.8} & \rev{77.8} & \rev{74.7} & \rev{72.9} & \rev{63.5} & \rev{88.4} & \rev{88.0} & \rev{63.6} & \rev{46.6} & \rev{73.6} & \rev{71.9} & \rev{78.3} \\
\rev{MUMODIG~\cite{mumodig}} & \rev{95.5} & \rev{89.6} & \rev{86.0} & \rev{77.0} & \rev{76.0} & \rev{76.9} & \rev{64.5} & \rev{90.1} & \rev{90.2} & \rev{66.8} & \rev{48.4} & \rev{75.8} & \rev{73.9} & \rev{80.9} \\
\rev{GAA~\cite{gaa}} & \rev{92.7} & \rev{87.5} & \rev{85.8} & \rev{77.0} & \rev{74.3} & \rev{73.9} & \rev{61.6} & \rev{88.6} & \rev{88.7} & \rev{67.0} & \rev{47.9} & \rev{74.5} & \rev{71.6} & \rev{77.6} \\
\rev{FoolMix~\cite{foolmix}} & \rev{91.5} & \rev{87.2} & \rev{83.3} & \rev{75.6} & \rev{70.2} & \rev{73.0} & \rev{61.2} & \rev{87.2} & \rev{85.7} & \rev{62.0} & \rev{46.5} & \rev{72.0} & \rev{70.9} & \rev{75.7} \\
MEF\textsubscript{H}
& 98.9
& \underline{93.2}
& \underline{90.2}
& \underline{80.2}
& \underline{79.8}
& \underline{80.5}
& \underline{67.0}
& \underline{94.1}
& \underline{92.4}
& \underline{69.5}
& \underline{51.7}
& \underline{77.8}
& \underline{75.3}
& \underline{82.1}
\\
MEF\textsubscript{F}
& \textbf{99.4}
& \textbf{96.6}
& \textbf{95.1}
& \textbf{85.7}
& \textbf{83.3}
& \textbf{82.6}
& \textbf{70.4}
& \textbf{96.7}
& \textbf{95.1}
& \textbf{75.4}
& \textbf{55.2}
& \textbf{83.8}
& \textbf{80.4}
& \textbf{89.0}
\\
\hline
\end{tabular}
}
\end{small}
\end{center}
\vspace{-15pt}
\end{table*}
\subsection{Extended Analysis of Attacking Ensemble of Models}
\label{sec:exp_ensemble_supp}
\revminor{This section complements Section~\ref{sec:exp_ensemble_main} by providing the full comparison results on ensemble attacks. Table~\ref{tab:attack_an_ensemble_model} details the performance against 12 gradient stabilization attacks and two ensemble-specialized approaches (SVRE~\cite{svre}, CWA~\cite{cwa}).}

\rev{While MEF achieves SOTA performance overall, we note that MEF\textsubscript{H} marginally trails EMI on the structurally similar target ResNet-101 (by 0.2\%). This occurs because baselines like EMI tend to overfit the specific architectural features of the surrogate (ResNet-50 is part of the ensemble). While this benefits transfer to same-family models, it causes failure on heterogeneous targets (e.g., ViT-L/32). In contrast, MEF prioritizes architecture-agnostic flatness, achieving comprehensive dominance on complex targets (e.g., +45\% on ViT over EMI) by avoiding such architectural overfitting.}

\begin{figure*}[htbp!]
    \centering
    \rev{
    \includegraphics[width=0.75\linewidth]{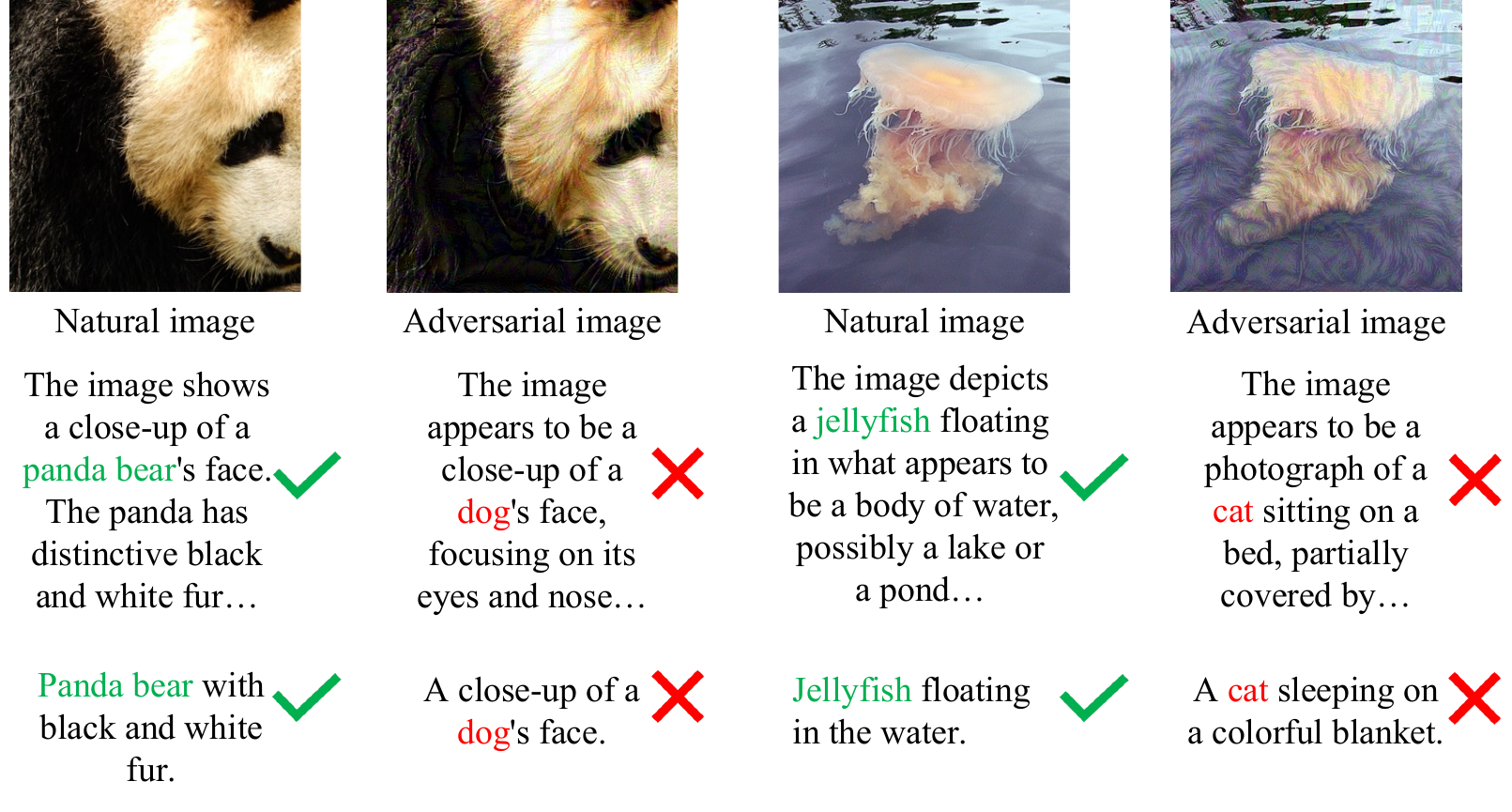} \\
    \vspace{10pt}
    \includegraphics[width=0.75\linewidth]{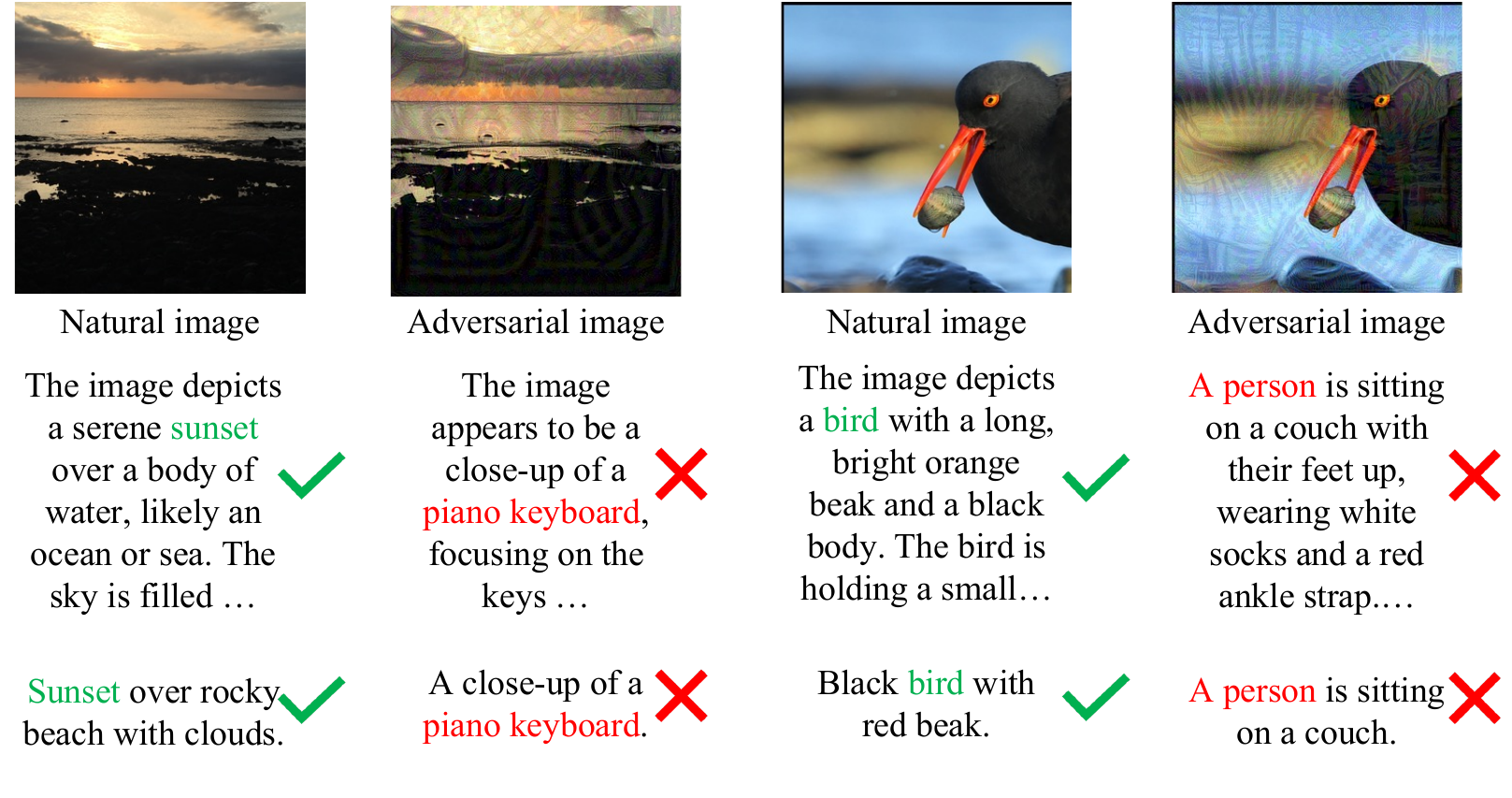} \\
    \vspace{10pt}
    \includegraphics[width=0.75\linewidth]{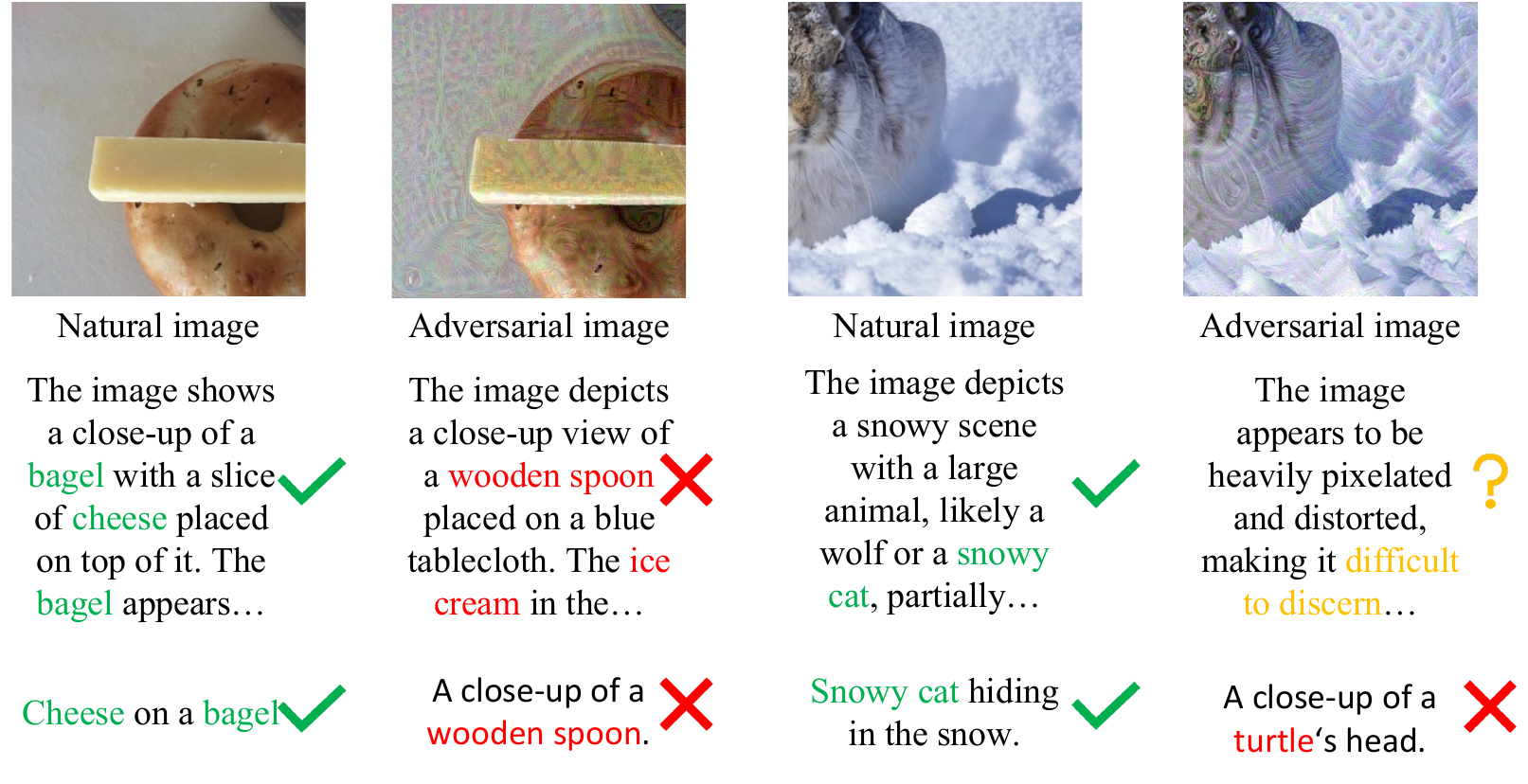} \\
    
    \vspace{-10pt}
    
    \caption{Qualitative visualization of adversarial transferability against \textbf{Qwen2.5-VL}. 
    Adversarial examples are generated on the surrogate model ResNet-50 using the MEF method and then transferred to the target Qwen2.5-VL model. 
    For each visual example, we display the model's responses to two different prompts to demonstrate the robustness of the attack. 
    The \textbf{upper text block} for each image corresponds to the prompt ``Describe the image.'', while the \textbf{lower text block} corresponds to the prompt ``Describe the image briefly.''. 
    The results show that MEF successfully misleads the model into generating incorrect semantic descriptions under both prompt settings.}
    \label{fig:qwen2.5_vl_viz}
    }
\end{figure*}
\subsection{Additional Qualitative Results on MLLMs}
\label{mllm_visual_exp}
\rev{
We evaluated the transferability on Qwen2.5-VL by generating adversarial examples on a ResNet-50 surrogate using MEF. To comprehensively assess the model's behavior, we employed two distinct prompts: the standard ``\textit{Describe the image.}'' and the concise ``\textit{Describe the image briefly.}''. The concise prompt was introduced to mitigate potential noise in CLIP-based metric calculations caused by the excessive length of standard descriptions.

As illustrated in Fig.~\ref{fig:qwen2.5_vl_viz}, MEF achieves strong semantic disruption. In the vast majority of cases, the perturbations cause the model to confidently hallucinate incorrect objects (e.g., misidentifying a \textit{bagel} as a \textit{wooden spoon}). We observed a rare phenomenon—occurring in only 5 out of 1,000 test samples—where the model initially refused to recognize content under the standard prompt, citing ``pixelation'' or ``distortion'' (e.g., the bottom-right example). However, when switched to the concise prompt, the model was forced to generate a specific prediction and was successfully misled into a semantic error (e.g., describing a ``turtle's head''), confirming the attack's efficacy even in these edge cases.
}

\begin{figure*}[htbp!]
    \centering
    \includegraphics[width=0.95\linewidth]{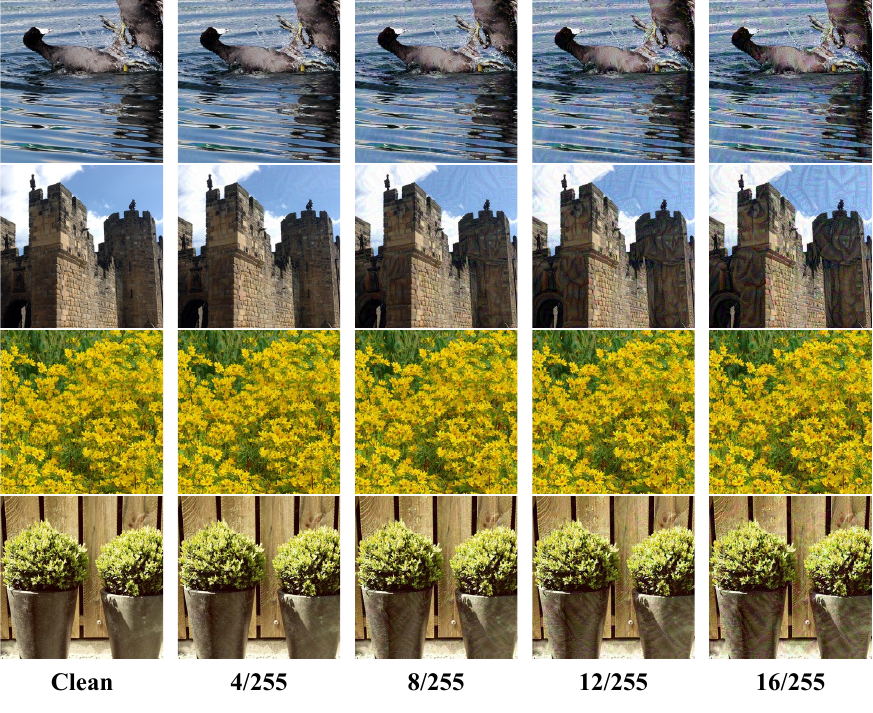}
    \caption{\rev{Qualitative visualization of MEF adversarial examples under varying perturbation budgets ($\epsilon$). From left to right: Clean Image, and Adversarial Examples with $\epsilon=4/255, 8/255, 12/255, 16/255$. While perturbations at $\epsilon=16/255$ may be visible in smooth regions (top row), they remain unobtrusive in textured areas (bottom rows). MEF's robust optimization allows it to remain effective even at lower, strictly imperceptible budgets (e.g., $\epsilon=4/255$).}}
    \label{fig:supp_epsilon_visual}
\end{figure*}
\subsection{Additional Qualitative Results on Perceptual Quality}
\label{sec:supp_visual}
\rev{
To further address concerns regarding the visual imperceptibility of $L_\infty$-constrained adversarial examples, we provide a detailed qualitative inspection of MEF-generated samples across varying perturbation budgets ($\epsilon \in \{4/255, 8/255, 12/255, 16/255\}$).

As illustrated in \textbf{Fig.~\ref{fig:supp_epsilon_visual}}, the perceptibility of perturbations is highly dependent on the image content. For images with complex textures or high-frequency details (e.g., the bottom two rows), the adversarial noise at $\epsilon=16/255$ remains largely imperceptible to the human eye due to the visual masking effect. Conversely, for images containing large smooth regions or plain backgrounds (e.g., the top row), artifacts become noticeable at the standard $\epsilon=16/255$ budget.

This observation underscores the trade-off between attack strength and strict imperceptibility. Crucially, as demonstrated in our parameter sensitivity analysis (Fig.~\ref{fig:epsilon_ablation}), MEF maintains a high transfer success rate even at stricter budgets (e.g., achieving $>40\%$ ASR at $\epsilon=4/255$, significantly outperforming baselines). This indicates that while $\epsilon=16/255$ serves as a standard benchmark for comparing transferability limits, MEF provides the flexibility to operate effectively in scenarios requiring strict stealthiness ($\epsilon \le 8/255$), where perturbations are visually undetectable.
}

\begin{figure*}[htbp!]
  \centering
  \vspace{-5pt}
  
  \subfloat[ResNet-50~\cite{resnet} \label{fig:flat_res50}]{
    \begin{minipage}[b]{0.25\textwidth}
      \centering
      \includegraphics[width=\linewidth]{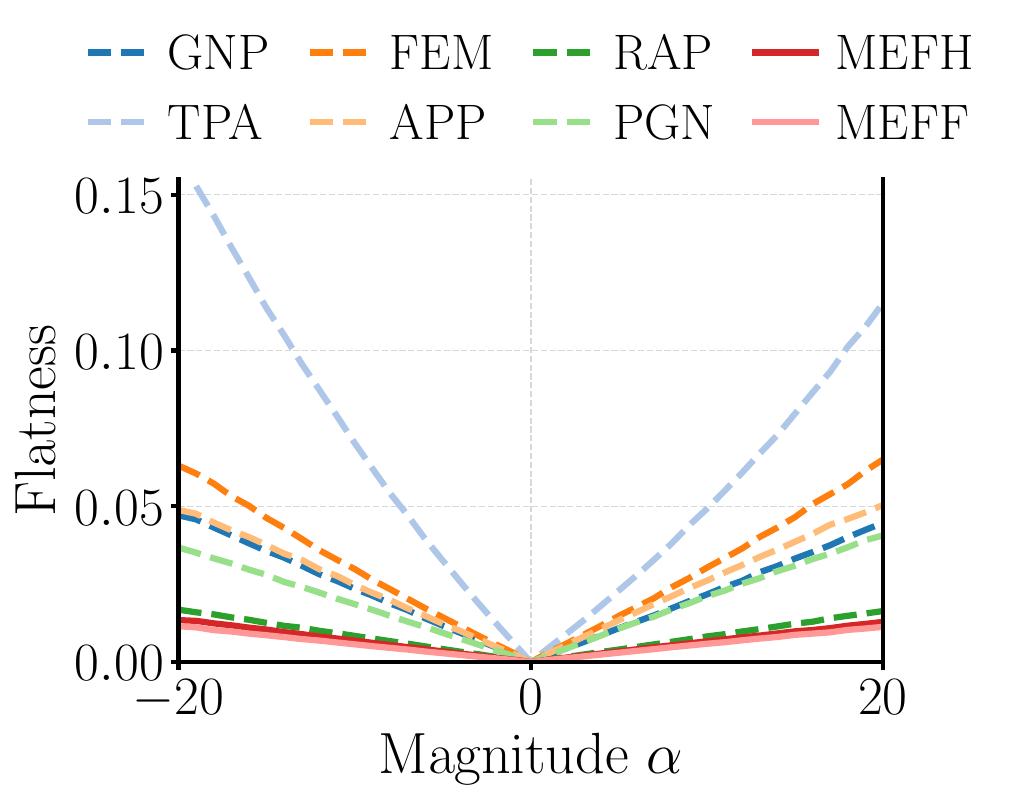}
    \end{minipage}
  }
  \hspace{-15pt}
  \subfloat[ResNet-101 \label{fig:flat_res101}]{
    \begin{minipage}[b]{0.25\textwidth}
      \centering
      \includegraphics[width=\linewidth]{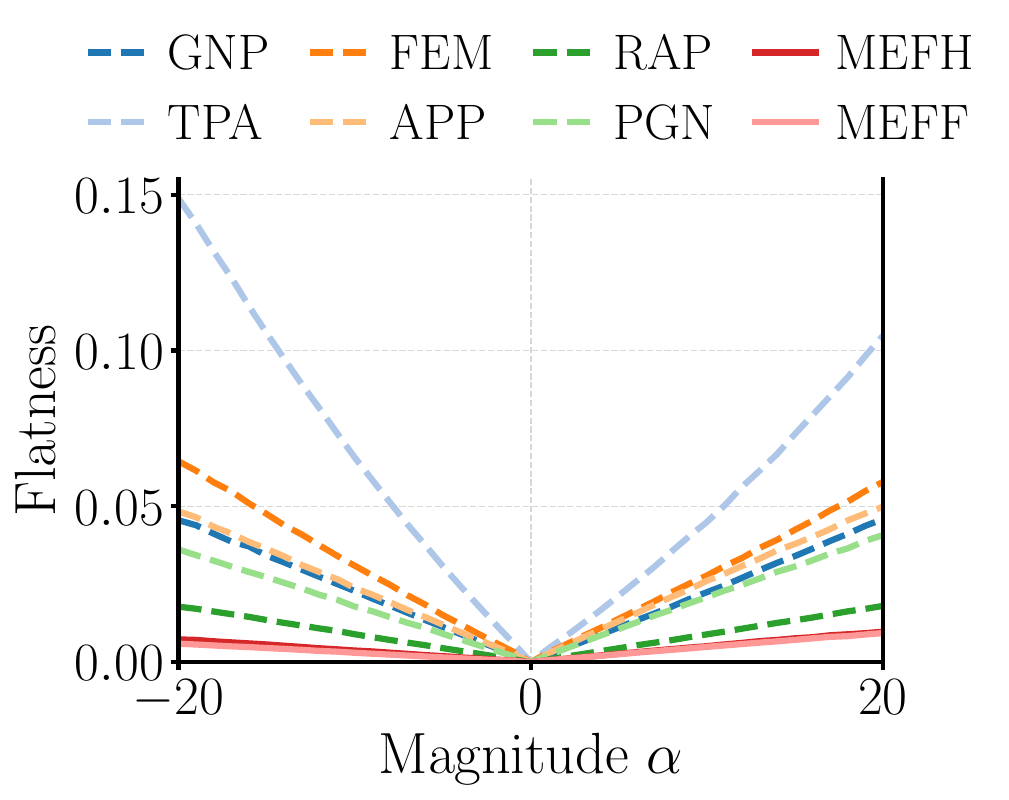}
    \end{minipage}
  }
  \hspace{-15pt}
  \subfloat[Inc-v3~\cite{incv3} \label{fig:flat_inc3}]{
    \begin{minipage}[b]{0.25\textwidth}
      \centering
      \includegraphics[width=\linewidth]{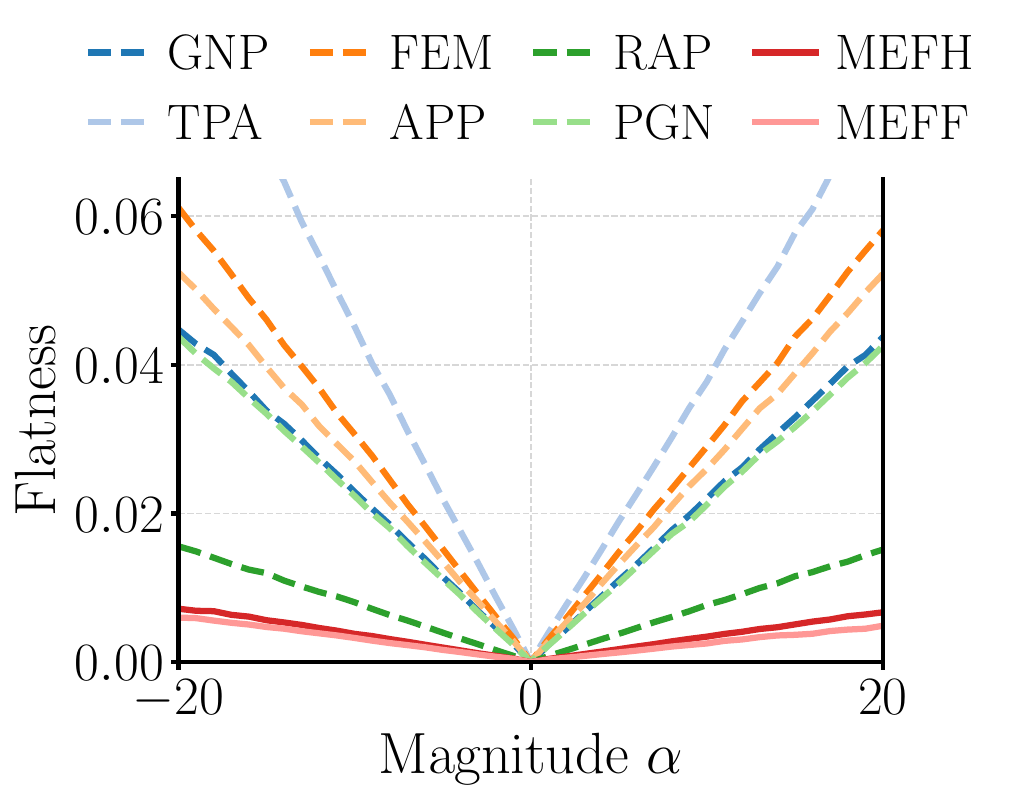}
    \end{minipage}
  }
  \hspace{-15pt}
  \subfloat[Inc-v4 \label{fig:flat_inc4}]{
    \begin{minipage}[b]{0.25\textwidth}
      \centering
      \includegraphics[width=\linewidth]{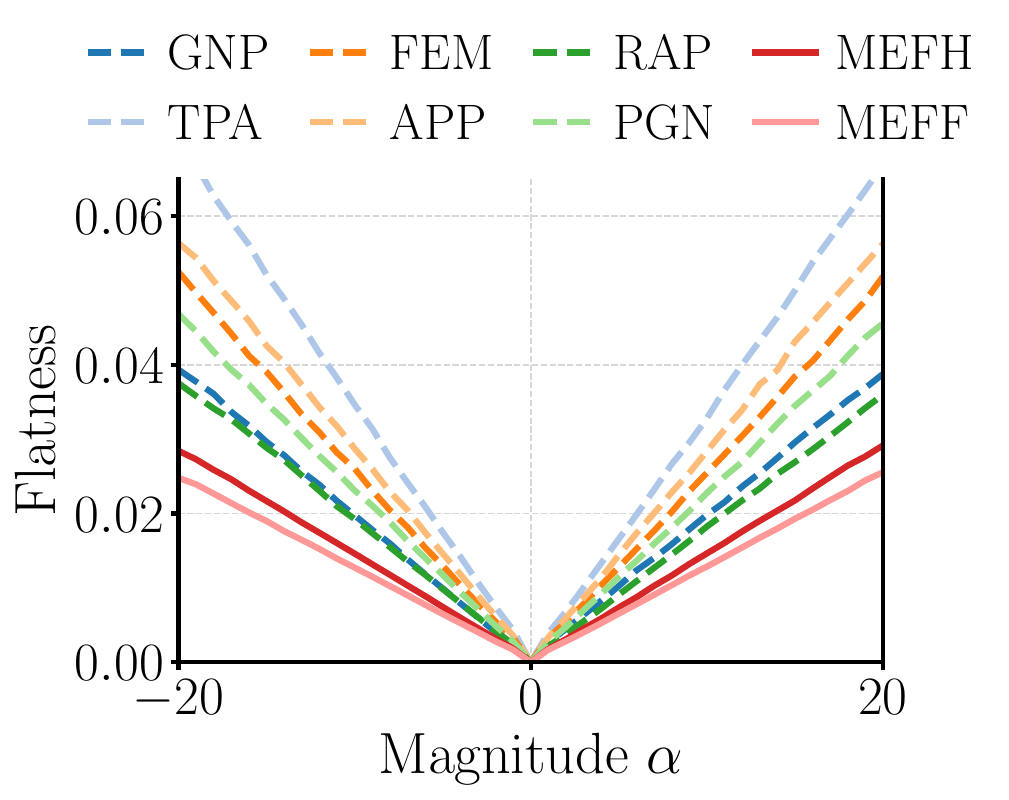}
    \end{minipage}
  }
  
  \caption{The flatness visualization of adversarial examples via $\ell_2$-constrained perturbation analysis. Evaluating MEF\textsubscript{H} and MEF\textsubscript{F} on Res-50~\cite{resnet}, Res-101~\cite{resnet}, Inc-v3~\cite{incv3} and Inc-v4~\cite{incv4} with random directional perturbations (magnitude range: $[-20,20]$).}
  \label{fig:all_flatness}
  
  \vspace{-5pt}
  
\end{figure*}
\subsection{Visualization of Landscape Flatness}
\label{sec:supp_flatness_vis}
\revminor{This section provides visual corroboration of the quantitative flatness analysis presented in the main text (Section~\ref{sec:exp_flatness_analysis}). We adopt the visualization framework from RAP~\cite{rap} to plot the loss landscape geometry. Specifically, we sample random directions from a Gaussian distribution and normalize them to an $\ell_2$-norm sphere to ensure uniform exploration. We then compute and plot the loss variations between the original and perturbed adversarial examples across perturbation magnitudes ranging from $-20$ to $20$. For robust visualization, each magnitude level employs 20 randomly sampled directions.}

\revminor{The results on Res-50~\cite{resnet} and Inc-v3~\cite{incv3} are illustrated in Figure~\ref{fig:all_flatness}. Consistent with the quantitative metrics in Table~\ref{tab:flatness_compare}, MEF\textsubscript{H} and MEF\textsubscript{F} exhibit the lowest loss variations and the smoothest loss surfaces compared to baselines. These visualizations intuitively verify that our Gradient Balancing strategy successfully prevents the optimization from falling into sharp, non-transferable local minima.}

\subsection{Additional Qualitative Results on Commercial Vision APIs}
\label{sec:supp_gcv}

\rev{
To provide intuitive visual evidence of the transferability of our method to real-world commercial systems, we conducted a qualitative evaluation using the official \href{https://cloud.google.com/vision/docs/drag-and-drop}{Google Cloud Vision API drag-and-drop interface}. We selected six pairs of clean images and their corresponding adversarial counterparts. All adversarial examples were generated on a ResNet-50 surrogate model using the proposed MEF attack under the standard perturbation budget $\epsilon=16/255$.

The visualization results are presented in Fig.~\ref{fig:gcv_visualization}. As shown in the screenshots, MEF successfully misleads the commercial black-box classifier into outputting incorrect labels with high confidence. For instance, in Case 1, a \textit{Dog} (confidence 96\%) is misclassified as a \textit{Calf} (confidence 60\%) after perturbation. Similarly, in Case 4, a \textit{Squirrel} is recognized as \textit{Linens} with 88\% confidence. These examples demonstrate severe semantic disruption, where the predicted labels shift to completely unrelated object categories (e.g., \textit{Zebra} $\to$ \textit{Marine Biology} in Case 6). This visual evidence confirms that the flat local minima identified by MEF on the surrogate model generalize effectively to disrupt the semantic understanding of production-level API services, posing a practical threat to real-world applications.
}
\begin{figure*}[htbp!]
    \centering
    \setlength{\tabcolsep}{0pt}
    
    \subfloat[Case 1: \textbf{Dog} $\to$ \textbf{Calf}]{
        \includegraphics[width=0.95\linewidth, height=7cm, keepaspectratio]{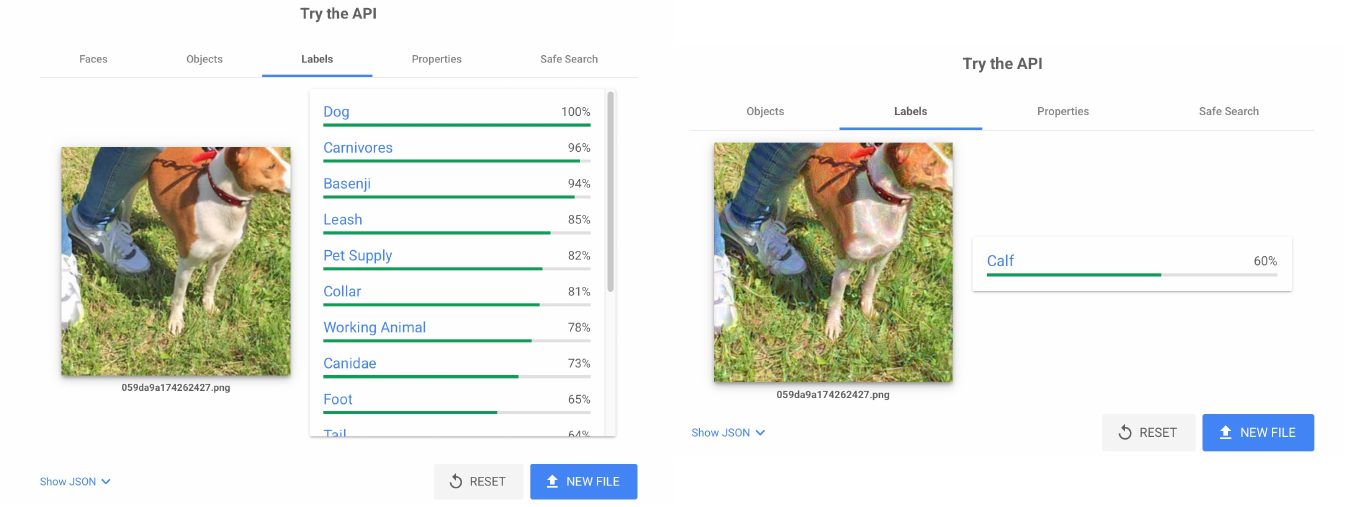}
    }
    
    \vspace{5pt}
    
    \subfloat[Case 2: \textbf{Dog} $\to$ \textbf{Temple}]{
        \includegraphics[width=0.95\linewidth, height=7cm, keepaspectratio]{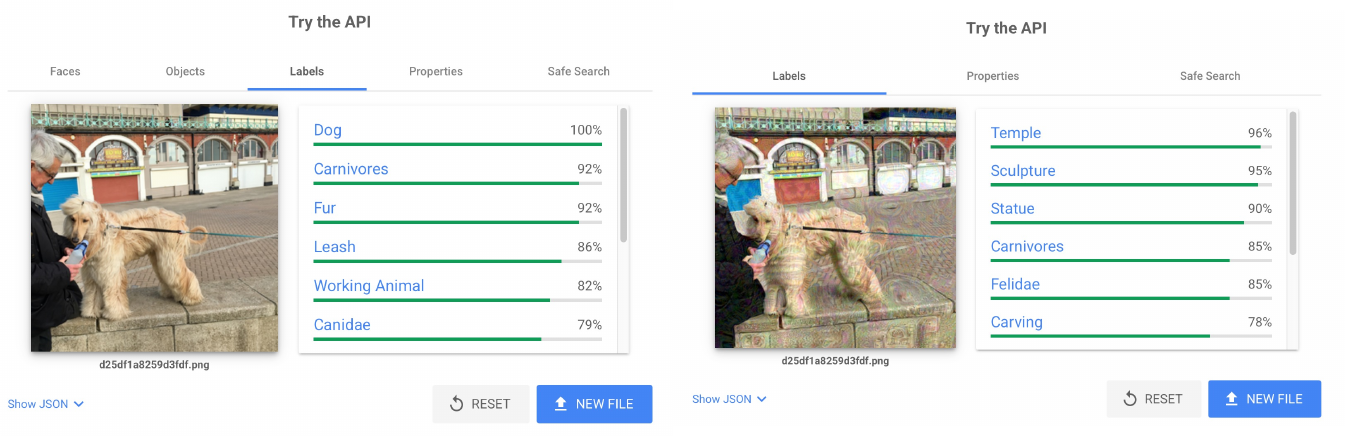}
    }
    
    \vspace{5pt}
    
    \subfloat[Case 3: \textbf{Dog} $\to$ \textbf{Sketch}]{
        \includegraphics[width=0.95\linewidth, height=7cm, keepaspectratio]{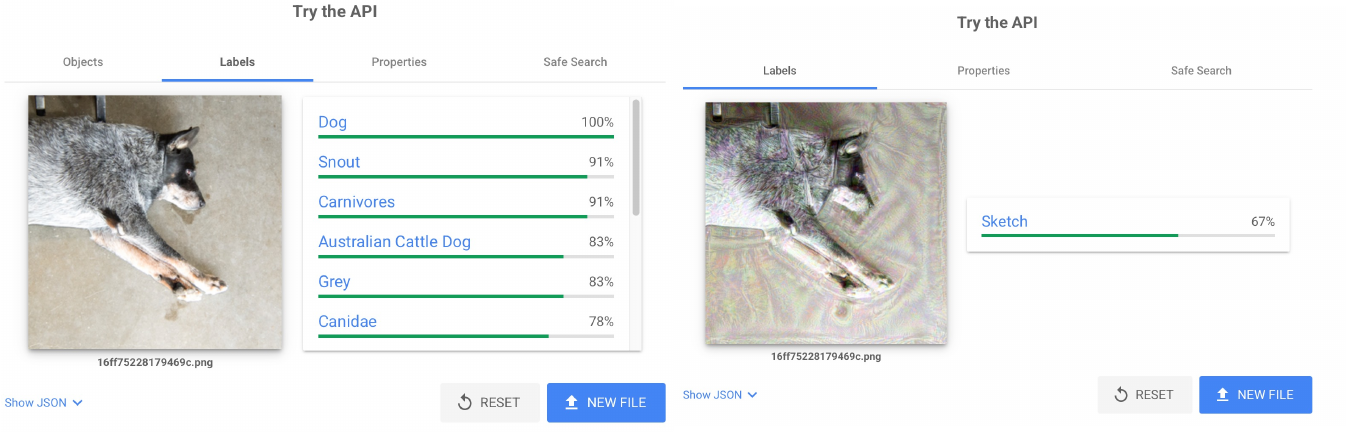}
    }

    \vspace{5pt}
    
    \caption{\textbf{Qualitative Visualization of Real-World Attacks on Google Cloud Vision API.} 
    We present six test cases generated by MEF on a ResNet-50 surrogate. 
    Each sub-figure displays a screenshot of the Google Vision API interface, showing the \textbf{Clean Image results (Left)} and the \textbf{Adversarial Image results (Right)}. 
    \textbf{(a) Cases 1-3:} Successful semantic subversion on natural images (Animals and Scenes). 
    (Figure continued on next page)}
    \label{fig:gcv_visualization}
\end{figure*}

\begin{figure*}[htbp!]
    \centering
    \ContinuedFloat
    \setlength{\tabcolsep}{0pt}
    
    \subfloat[Case 4: \textbf{Squirrel} $\to$ \textbf{Linens}]{
        \includegraphics[width=0.95\linewidth, height=7cm, keepaspectratio]{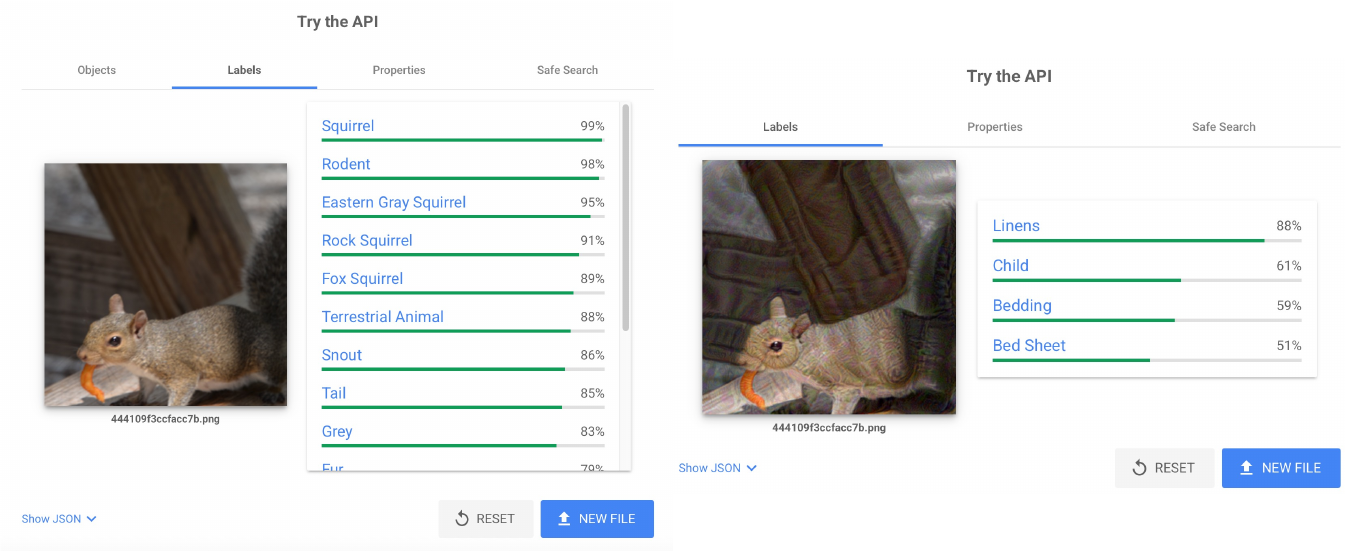}
    }
    
    \vspace{5pt}
    
    \subfloat[Case 5: \textbf{Bicycle Helmet} $\to$ \textbf{Net}]{
        \includegraphics[width=0.95\linewidth, height=7cm, keepaspectratio]{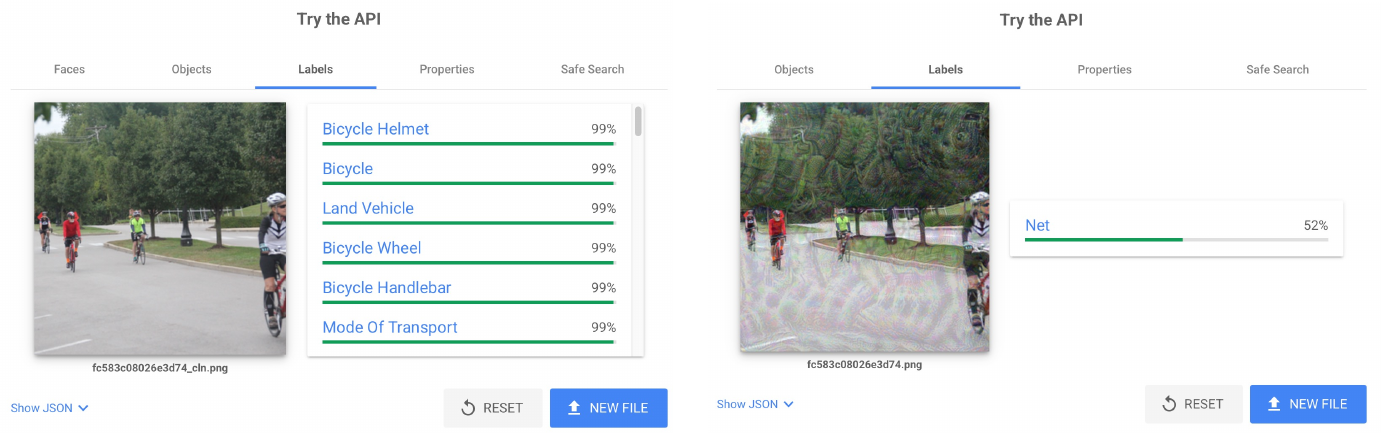}
    }
    
    \vspace{5pt}
    
    \subfloat[Case 6: \textbf{Zebra} $\to$ \textbf{Marine Biology}]{
        \includegraphics[width=0.95\linewidth, height=7cm, keepaspectratio]{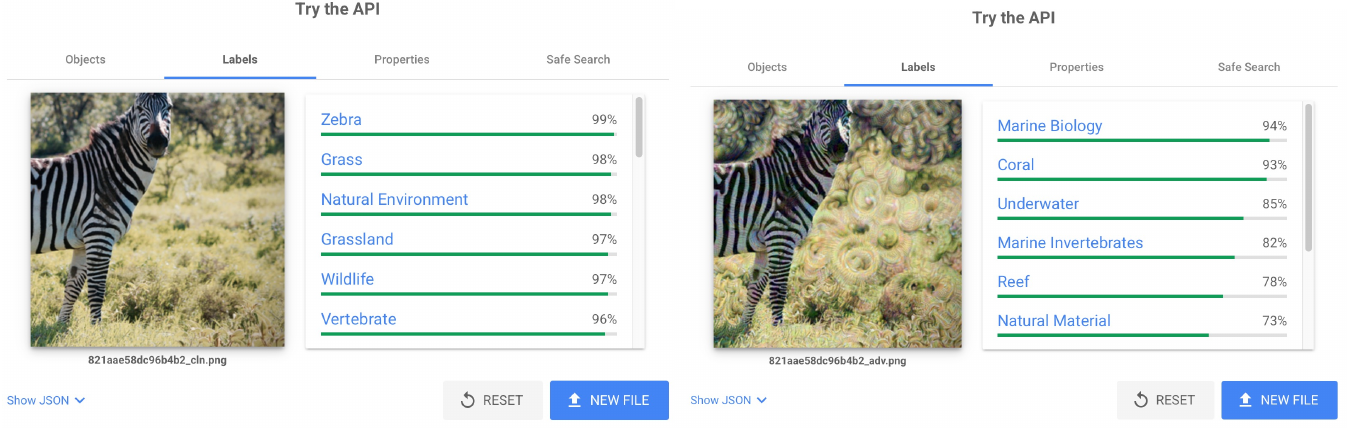}
    }

    \vspace{5pt}
    
    \caption[]{Qualitative Visualization of Real-World Attacks on Google Cloud Vision API (Continued). 
    \textbf{(b) Cases 4-6:} Additional comparisons demonstrating robustness against diverse object categories and textures. The screenshots validate that MEF successfully misleads the commercial black-box classifier into outputting high-confidence incorrect labels (e.g., misclassifying a Bagel as a Spoon in Case e), demonstrating practical threats in realistic API scenarios.}
\end{figure*}

\end{document}